\documentclass[11pt]{article}
\pdfoutput=1

\usepackage{mathrsfs}
\usepackage{amsmath,amssymb}
\usepackage{bm}
\usepackage{natbib}
\usepackage[usenames]{color}
\usepackage{amsthm}

\usepackage{multirow} 
\usepackage{enumitem}

\allowdisplaybreaks

\usepackage[colorlinks,
linkcolor=red,
anchorcolor=blue,
citecolor=blue
]{hyperref}

\usepackage{mylatexstyle}

\usepackage{setspace}
\usepackage[left=1in, right=1in, top=1in, bottom=1in]{geometry}

\newcommand{\la}{\langle}
\newcommand{\ra}{\rangle}

\def\poly{\mathrm{poly}}
\def\bbvec{\mathrm{\mathbf{b}}}
\author
{
    Yuan Cao\thanks{Department of Statistics and Actuarial Science and Department of Mathematics, The University of Hong Kong, Hong Kong; e-mail:  {\tt yuancao@hku.hk}} 
    ~~~and~~~
    Difan Zou\thanks{Department of Computer Science and Institute of Data Science, The University of Hong Kong, Hong Kong; e-mail:  {\tt dzou@cs.hku.hk}}
    ~~~and~~~
    Yuanzhi Li\thanks{Machine Learning Department, Carnegie Mellon University, Pittsburgh, PA, USA; e-mail: {\tt yuanzhil@andrew.cmu.edu}} 
    ~~~and~~~
    Quanquan Gu\thanks{Department of Computer Science, University of California, Los Angeles, CA, USA; e-mail: {\tt qgu@cs.ucla.edu}}
}

\date{}

\title{The Implicit Bias of Batch Normalization in Linear Models and Two-layer Linear Convolutional Neural Networks}

\begin{document}

\maketitle

\begin{abstract}%
  We study the implicit bias of batch normalization trained by gradient descent. We show that when learning a linear model with batch normalization for binary classification, gradient descent converges to a \textit{uniform margin classifier} on the training data with an $\exp(-\Omega(\log^2t))$ convergence rate. This distinguishes linear models with batch normalization from those without batch normalization in terms of both the type of implicit bias and the convergence rate. We further extend our result to a class of two-layer, single-filter linear convolutional neural networks, and show that batch normalization has an implicit bias towards a \textit{patch-wise uniform margin}. Based on two examples, we demonstrate that patch-wise uniform margin classifiers can outperform the maximum margin classifiers in certain learning problems. Our results contribute to a better theoretical understanding of batch normalization.
%   In comparison, it is well known that linear logistic regression as a l  with a convergence rate $\exp(-\Omega(\log^2t))$.  
\end{abstract}

% Our result also gives a $\exp(-\Omega(\log^2t))$ convergence rate towards the uniform margin, which is  significantly faster than the known convergence rates towards maximum margin in linear logistic regression.

\section{Introduction}
Batch normalization (BN) is a popular deep learning technique that normalizes network layers by re-centering/re-scaling within a batch of training data  \citep{ioffe2015batch}. It has been  empirically demonstrated to be helpful for training and often leads to better generalization performance. A series of works have attempted to understand and explain the empirical success of BN from different perspectives. \citet{santurkar2018does,bjorck2018understanding,arora2019theoretical} showed that BN enables a more feasible learning rate in the training process, thus leading to faster convergence and possibly finding flat minima. A similar phenomenon has also been discovered in \citet{luotowards}, which studied BN by viewing it as an implicit weight decay regularizer. \citet{cai2019quantitative,kohler2019exponential} conducted  studies on the benefits of BN in training linear models with gradient descent: they quantitatively demonstrated the accelerated convergence of gradient descent with BN compared to that without BN. Very recently, \citet{wu2023training} studied how SGD interacts with batch normalization and can exhibit undesirable training dynamics. However, these works mainly focus on the convergence analysis, and cannot fully reveal the role of BN in model training and explain its effect on the generalization performance.  

To this end, a more important research direction is to study the \textit{implicit bias} \citep{neyshabur2014search} for BN, or more precisely, identify the special properties of the solution obtained by gradient descent with BN. When the BN is absent, the implicit bias of gradient descent for the linear model (or deep linear models) is widely studied \citep{soudry2017implicit,ji2018gradient}: when the data is linearly separable, gradient descent on the cross-entropy loss will converge to the \textit{maximum margin solution}. However, when the BN is added, the implicit bias of gradient descent is far less understood. Even for the simplest linear classification model, there are nearly no such results for the gradient descent with BN except \citet{lyu2019gradient}, which established an implicit bias guarantee for gradient descent on general homogeneous models (covering homogeneous neural networks
with BN, if ruling out the zero point). However, they only proved that gradient descent will converge to a Karush–Kuhn–Tucker (KKT) point of the maximum margin problem, while it remains unclear whether the obtained solution is guaranteed to achieve maximum margin, or has other properties while satisfying the KKT conditions.

In this paper, we aim to systematically study the implicit bias of batch normalization in training linear models and a class of linear convolutional neural networks (CNNs) with gradient descent. Specifically, consider a training data set $\{(\xb_i , y_i)\}_{i=1}^n$ and assume that the training data inputs are centered without loss of generality. Then we consider the linear model and linear single-filter CNN model with batch normalization as follows
\begin{align*}
    f(\wb,\gamma,\xb) = \gamma\cdot \frac{\la \wb, \xb \ra }{ \sqrt{ n^{-1}\sum_{i=1}^n \la \wb, \xb_i\ra^2 } },~ g(\wb,\gamma,\xb) = \sum_{p=1}^P\gamma\cdot \frac{\la \wb, \xb^{(p)} \ra}{ \sqrt{ n^{-1}P^{-1}\sum_{i=1}^n\sum_{p=1}^P \la \wb, \xb_i^{(p)}\ra^2 } },
\end{align*}
where $\wb\in \RR^d$ is the parameter vector, $\gamma$ is the scale parameter in batch normalization, and $\xb^{(p)}$, $\xb_i^{(p)}$ denote the $p$-th patch in $\xb$, $\xb_i$ respectively. $\wb,\gamma$ are both trainable parameters in the model. To train $f(\wb,\gamma,\xb)$ and $g(\wb,\gamma,\xb)$, we use gradient descent starting from $\wb^{(0)}$, $\gamma^{(0)}$ to minimize the cross-entropy loss. The following informal theorem gives a simplified summary of our main results.
\begin{theorem}[Simplification of Theorems~\ref{thm:linearBN} and \ref{thm:CNNBN}]\label{thm:informal}
Suppose that $\gamma^{(0)} = 1$, and that the initialization scale $\| \wb^{(0)} \|_2$ and learning rate of gradient descent are sufficiently small. Then the following results hold:
\begin{enumerate}[leftmargin = *]
    \item \textbf{(Implicit bias of batch normalization in linear models)} When training $f(\wb,\gamma,\xb)$, if the equation system $\la \wb, y_i\cdot \xb_i \ra = 1$, $i\in [n]$ has at least one solution, then the training loss converges to zero, and the iterate $\wb^{(t)}$ of gradient descent satisfies that
    \begin{align*}
        \frac{1}{n^2}\sum_{i,i'=1}^n \big[\mathrm{margin}\big(\wb^{(t)}, (\xb_{i'},y_{i'})\big) - \mathrm{margin}\big(\wb^{(t)}, (\xb_i,y_i)\big) \big]^2 = \exp(-\Omega(\log^2t)),
    \end{align*}
    where $\mathrm{margin}\big(\wb, (\xb,y)\big):= y\cdot \frac{ \la \wb, \xb \ra }{ \sqrt{ n^{-1}\sum_{i=1}^n \la \wb, \xb_i\ra^2 } }$.
    \item \textbf{(Implicit bias of batch normalization in single-filter linear CNNs)} When training $g(\wb,\gamma,\xb)$, if the equation system $\la \wb, y_i\cdot \xb_i^{(p)} \ra = 1$, $i\in [n]$, $p\in [P]$ has at least one solution, then the training loss converges to zero, and the iterate $\wb^{(t)}$ of gradient descent satisfies that
    \begin{align*}
        \frac{1}{n^2}\sum_{i,i'=1}^n \big[\mathrm{margin}_{\mathrm{patch}}\big(\wb^{(t)}, (\xb_{i'}^{(p')},y_{i'})\big) - \mathrm{margin}_{\mathrm{patch}}\big(\wb^{(t)}, (\xb_i^{(p)},y_i)\big) \big]^2 = \exp(-\Omega(\log^2t)),
    \end{align*}
    where $\mathrm{margin}_{\mathrm{patch}}\big(\wb, (\xb^{(p)},y)\big):= y\cdot \frac{\la \wb, \xb^{(p)} \ra}{ \sqrt{ n^{-1}P^{-1}\sum_{i=1}^n\sum_{p=1}^P \la \wb, \xb_i^{(p)}\ra^2 } }$.
\end{enumerate}
\end{theorem}
Theorem~\ref{thm:informal} indicates that batch normalization in linear models and single-filter linear CNNs has an implicit bias towards a \textit{uniform margin}, i.e., gradient descent eventually converges to such a predictor that all training data points are on its margin. Notably, for CNNs, the margins of the predictor are uniform not only over different training data inputs, but also over different patches within each data input. 

% he predictor eventually achieves the same margin on all the training data points
%It shows that
% \begin{center}
% \emph{\parbox{0.95\textwidth}{\centering With batch normalization, linear models and simple two-layer convolutional neural networks trained by gradient descent tend to achieve \textbf{uniform margins} in binary classification.}}
% \end{center}

% Specifically, we consider

The major contributions of this paper are as follows:
\begin{itemize}[leftmargin = *]
    % \item To the best of our knowledge, this is the first result that demonstrates an implicit bias towards a uniform margin. Our result thus reveals a unique property of batch normalization 
    \item By revealing the implicit bias of batch normalization towards a uniform margin, our result distinguishes linear models and CNNs with batch normalization from those without batch normalization. The sharp convergence rate given in our result also allows us to make comparisons with existing results of the implicit bias for linear models and linear networks. In contrast to the $1/ \log(t)$ convergence rate of linear logistic regression towards the maximum margin solution \citep{soudry2017implicit}, the linear model with batch normalization converges towards a uniform margin solution with convergence rate $\exp(-\Omega(\log^2t))$, which is significantly faster. Therefore, our result demonstrates that batch normalization can significantly increase the directional convergence speed of linear model weight vectors and linear CNN filters.
    
    \item Our results can also serve as important examples in the literature of implicit bias of homogeneous models. Although \citet{lyu2019gradient} showed that gradient descent converges to a KKT point of the maximum margin problem, the KKT points are usually not unique. Recently, it has been shown by \citet{vardi2022margin} that for certain neural network architectures, there exist data sets such that the actual solution given by gradient descent is not the maximum margin solution. Our result also provides novel, concrete and practical examples -- linear models and linear CNNs with batch normalization -- whose implicit bias are not maximum margin, but uniform margin.
    
    % \item The sharp $\exp(-\Omega(\log^2t))$ convergence rate given in our result 
    % We show that when training a linear model with batch normalization in binary classification, gradient descent will converge to a \textit{uniform margin solution} whenever such solution exists and the learning rate and the norm of the weights at initialization are small enough. Moreover, we establish a $\Theta(1/t)$ convergence rate for the training loss function value, and a $\exp(-\Omega(\log^2(t)))$ convergence rate for the convergence to the uniform margin. 
    \item We note that the uniform margin in linear models can also be achieved by performing linear regression minimizing the square loss. However, the patch-wise uniform margin result for single-filter linear CNNs is unique and cannot be achieved by regression models without batch normalization. To further demonstrate the benefit of such an implicit bias for generalization, we also construct two example learning problems, for which the patch-wise uniform margin classifier is guaranteed to outperform the maximum margin classifier in terms of test accuracy. 
    \item From a technical perspective, our analysis gives many novel proof techniques that may be of independent interest. To derive tight bounds, we prove an inequality that is similar to the Chebyshev's sum inequality, but for pair-wise maximums and differences. We also establish an equivalence result connecting a quantity named margin discrepancy to the Euclidean metric and the metric induced by the data sample covariance matrix. At last, we develop an induction argument over five induction hypotheses, which play a key role in developing the sharp convergence rate towards the uniform margin. 
    % CNN different
    % We further extend our analysis to batch normalization on a two-layer linear convolutional neural network with a single filter. We demonstrate that gradient descent will converge to a \textit{patch-wise uniform margin solution}, whenever such solution exists and the learning rate and the norm of the weights at initialization are small enough. We also have the $\Theta(1/t)$ convergence rate for  the training loss function value and $\exp(-\Omega(\log^2(t)))$ convergence rate to the patch-wise uniform margin.
    % \item We also construct a simple learning problem, for which the patch-wise uniform margin classifier can achieve near-zero test error, while the maximum margin classifier only has around $50\%$ test accuracy. Note that the maximum margin classifier can be obtained by logistic regression without batch normalization while our result demonstrates that batch normalization on linear CNNs gives patch-wise uniform margin classifier. Therefore, this gives an explanation of the advantage of batch normalization.
    % \item Experiments
    % show that when training a linear model with batch normalization in binary classification, gradient descent will converge to a uniform margin solution whenever such solution exists and the initialization scale and learning rate are appropriately chosen. Moreover, we establish a $\Theta(1/t)$ convergence rate for  the training objective function value, and a $\exp(-\Omega(\log^2(t)))$ convergence rate for the convergence to the uniform margin. 
\end{itemize}

% Key techniques: 

% Despite extensive empirical explorations for all of these normalization methods, theoretical studies are mostly concentrated on weight normalization. In particular, \citet{hoffer2018norm} showed a connection between weight normalization methods, learning rate adjustment, and weight decay regularization.

\section{Batch Normalization in Linear Models}\label{section:linear_model}
% \dz{need a formal definition of uniform margin}
Suppose that $(\xb_1,y_1),(\xb_2,y_2),\ldots, (\xb_n,y_n)$ are $n$ arbitrary training data points, where $\xb_i\in\RR^{d}$ and $y_i\in \{\pm 1\}$ for $i\in [n]$. We consider using a linear model with batch normalization to fit these data points. During training, the prediction function is given as 
\begin{align*}
f(\wb,\gamma,\xb) = \gamma\cdot \frac{\la \wb, \xb \ra - \la \wb, \overline\xb \ra}{ \sqrt{ n^{-1}\sum_{i=1}^n (\la \wb, \xb_i\ra -  \la \wb, \overline\xb \ra)^2 } },  
\end{align*}
where $\overline{\xb} = n^{-1}\sum_{i=1}^n \xb_i$ is the mean of all the training data inputs, $\wb\in \RR^d$ is the linear model parameter vector, and $\gamma$ is the scale parameter in batch normalization. $\wb,\gamma$ are both trainable parameters in the model. Clearly, the above definition exactly gives a linear model with batch normalization during the training of full-batch gradient descent. Note that we can assume the data are centered (i.e., $\overline{\xb}= \mathbf{0}$) without loss of generality: if $\overline{\xb} \neq \mathbf{0}$, we can simply consider new data inputs $\tilde\xb_i = \xb_i - \overline{\xb}$. Therefore, assuming $\overline{\xb}= \mathbf{0}$, we have
\begin{align*}
   f(\wb,\gamma,\xb) = \gamma\cdot \frac{\la \wb, \xb \ra}{\| \wb \|_{\bSigma} },~~\text{where}~~\bSigma = \frac{1}{n} \sum_{i=1}^n \xb_{i}\xb_{i}^{\top}.
\end{align*}
% Suppose that $(\xb_1,y_1),(\xb_2,y_2),\ldots, (\xb_n,y_n)$ are $n$ arbitrary training data points. We consider using a linear model with batch normalization to fit these data points. Specifically, let 
% \begin{align*}
% f(\wb,\gamma,\xb)  &= \gamma\cdot \frac{\la \wb, \xb \ra}{\| \wb \|_{\bSigma} }~~\text{with}~~\bSigma = \frac{1}{n} \sum_{i=1}^n \xb_{i}\xb_{i}^{\top},
% \end{align*}
% where $\wb\in \RR^d$ is the linear model parameter vector, and $\gamma$ is the scale parameter in batch normalization. $\wb,\gamma$ are both trainable parameters in the model. 
% $f(\wb,\gamma,\xb)$ be the batch normalization function on a data input $\xb\in \RR^d$, 
Consider training $f(\wb,\gamma,\xb)$ by minimizing the cross-entropy loss
\begin{align*}
    L(\wb,\gamma) = \frac{1}{n}\sum_{i=1}^n \ell[y_i\cdot f(\wb,\gamma,\xb_i)]
\end{align*}
with gradient descent, where $\ell(z) = \log(1 + \exp(-z))$ is the logistic loss. Then starting from the initial $\wb^{(0)}$ and $\gamma^{(0)}$, gradient descent with learning rate $\eta$ takes the following update
% we can write the gradient descent update rules for the parameters $\wb$ and $\gamma$ as follows:
\begin{align}
    &\wb^{(t+1)} = \wb^{(t)} - \frac{\eta  \cdot \gamma^{(t)} }{n\cdot \| \wb^{(t)} \|_{\bSigma}}\sum_{i=1}^n  \ell'[y_i\cdot f(\wb^{(t)},\gamma^{(t)},\xb_i)] \cdot y_i \cdot \bigg( \Ib -  \frac{\bSigma \wb^{(t)} \wb^{(t)\top}}{\| \wb^{(t)} \|_{\bSigma}^2} \bigg)\xb_i, \label{eq:gdupdate_w}\\
    &\gamma^{(t+1)} = \gamma^{(t)} - \frac{\eta }{n}\sum_{i=1}^n  \ell'[y_i\cdot f(\wb^{(t)},\gamma^{(t)},\xb_i)] \cdot y_i \cdot \frac{\la \wb^{(t)}, \xb_i \ra}{\| \wb^{(t)} \|_{\bSigma}}. \label{eq:gdupdate_gamma}
\end{align}
Our goal is to show that $f(\wb^{(t)},\gamma^{(t)},\xb)$ trained by \eqref{eq:gdupdate_w} and \eqref{eq:gdupdate_gamma} eventually achieves the same margin on all training data points. Note that $f(\wb,\gamma,\xb)$ is $1$-homogeneous in $\gamma$ and $0$-homogeneous in $\wb$, and the prediction of $f(\wb,\gamma,\xb)$ on whether a data input $\xb$ belongs to class $+1$ or $-1$ does not depend on the magnitude of $\gamma$. Therefore we focus on the normalized margin $y_i\cdot \la \wb, \xb \ra/ \| \wb \|_{\bSigma}$. Moreover, we also introduce the following quantity which we call margin discrepancy:
\begin{align*}
    D(\wb) := \frac{1}{n^2}\sum_{i,i'=1}^n \bigg(y_{i'}\cdot \frac{\la \wb, \xb_{i'}\ra}{\|\wb \|_{\bSigma}} - y_{i}\cdot\frac{\la \wb, \xb_{i}\ra}{ \| \wb \|_{\bSigma}} \bigg)^2.
\end{align*}
The margin discrepancy $D(\wb)$ measures how uniform the margin achieved by $f(\wb,\gamma,\xb)$ is. When $D(\wb)$ is zero, $f(\wb,\gamma,\xb)$ achieves the exact same margin on all the training data points. Therefore, we call $f(\wb,\gamma,\xb)$ or the corresponding linear classifier $\wb$ the \textit{uniform margin classifier} if $D(\wb) = 0$.

\subsection{The Implicit Bias  of Batch Normalization in Linear Models}\label{section:linear_model_result}
In this subsection we present our main result on the implicit bias of batch normalization in linear models. We first state the following assumption. 
\begin{assumption}\label{assump:uniformly_separable}
The equation system $\la \wb, y_i\cdot \xb_i \ra = 1$, $i\in [n]$ has at least one solution. 
\end{assumption}
Assumption~\ref{assump:uniformly_separable} is a very mild assumption which commonly holds under over-parameterization. For example, when $d\geq n$, Assumption~\ref{assump:uniformly_separable} holds almost surely when $\xb_i$ are sampled from a non-degenerate distribution. Assumption~\ref{assump:uniformly_separable} can also hold in many low-dimensional learning problems as well. Note that Assumption~\ref{assump:uniformly_separable} is our \textit{only} assumption on the data: we do not make any distribution assumption on the data, and we do not make any assumption on $d$ and $n$ either. 

In order to present the most general implicit bias result for batch normalization that covers both the $d\geq n$ case and the $d < n$ case, we introduce some additional notations to unify the argument. We note that by the gradient descent update formula \eqref{eq:gdupdate_w}, $\wb^{(t)}$ is only updated in the space $ \cX := \mathrm{span}\{\xb_1,\ldots. \xb_n\}$, and the component of $\wb^{(t)}$ in $\cX^{\perp}$ is unchanged during training (of course, if $\cX = \RR^d$, then there is no such component). This motivates us to define
\begin{align*}
    \lambda_{\max} = \sup_{\ub \in \cX \backslash \{\mathbf{0}\}} \frac{\ub^\top \bSigma \ub}{\| \ub \|_2^2}, \qquad \lambda_{\min} = \inf_{\ub \in \cX \backslash \{\mathbf{0}\}} \frac{\ub^\top \bSigma \ub}{\| \ub \|_2^2}.
\end{align*}
Moreover, we let $\Pb_{\cX}$ be the projection matrix onto $\cX$ (if $\cX = \RR^d$, then $\Pb_{\cX} = \Ib$ is simply the identity matrix). We can briefly check the scaling of $\lambda_{\max}$ and $\lambda_{\min}$ with a specific example: suppose that $\xb_i$, $i\in [n]$ are independently drawn from $N(\mathbf{0}, \Ib)$. Then when $ d/n \rightarrow c$ for some constant $c \neq 1$, with high probability, $\lambda_{\max}$ and $\lambda_{\max}$ are both of constant order \citep{vershynin2010introduction}.

% (i) When $d \gg n$ and $\xb_i$ are independently drawn from $N(\mathbf{0}, \Ib)$,  $\lambda_{\max}$ and $\lambda_{\min}$ are both of order $\Theta(1/n)$. When  Then when $d \gg n$ or $n \gg d$, $\lambda_{\max}$ and $\lambda_{\min}$ are both of order $\tilde{\Theta}(1)$. When $n \asymp d$,  

Our main result for the implicit bias of batch normalization in linear models is given in the following theorem.

% Denote $\cX = \mathrm{span}\{\xb_1,\ldots. \xb_n\}$, and let $\Pb_{\cX}$ be the projection matrix onto $\cX$. Define
% \begin{align*}
%     \lambda_{\max} = \sup_{\ub \in \cX \backslash \{\mathbf{0}\}} \frac{\ub^\top \bSigma \ub}{\| \ub \|_2^2}, \qquad \lambda_{\min} = \inf_{\ub \in \cX \backslash \{\mathbf{0}\}} \frac{\ub^\top \bSigma \ub}{\| \ub \|_2^2}.
% \end{align*}

\begin{theorem}\label{thm:linearBN}
Suppose that Assumption~\ref{assump:uniformly_separable} holds. 
There exist $M = 1 / \poly( \lambda_{\min}^{-1}, \lambda_{\max}, \max_i \| \xb_i \|_2 )$, $\overline{\eta} =  1 / \poly( \lambda_{\min}^{-1}, \lambda_{\max}, \max_i \| \xb_i \|_2, \| \Pb_{\cX} \wb^{(0)} \|_2^{-1})$ and constants $C_1,C_2,C_3,C_4 > 0$, such that when $  \gamma^{(0)} = 1$, $\| \Pb_{\cX} \wb^{(0)} \|_2 \leq M$, and $\eta \leq \overline{\eta}$, there exists  
$t_0 \leq \eta^{-1}\cdot \poly( \lambda_{\min}^{-1}, \lambda_{\max}, \max_i \| \xb_i \|_2 ) $
and the following inequalities hold for all $t\geq t_0$:
\begin{align*}
    &\frac{C_1}{\eta\cdot (t - t_0 + 1)} \leq L(\wb^{(t)},\gamma^{(t)}) \leq \frac{C_2}{\eta\cdot (t - t_0 + 1)},\\
    &D(\wb^{(t)})\leq \frac{C_3\lambda_{\max}}{\lambda_{\min} } \cdot \exp\Bigg[ -\frac{ C_4\lambda_{\min}  }{\lambda_{\max}^{3/2} \cdot \| \Pb_{\cX}\wb^{(0)} \|_2^2} \cdot \log^2( (8/9) \eta \cdot (t - t_0)+ 1)\Bigg].
\end{align*}
% $L(\wb^{(t)},\gamma^{(t)}) = O((t-t_0)^{-1})$
%  Moreover, 
% \begin{align*}
%     % &\frac{1}{n^2}\sum_{i,i'=1}^n \bigg(y_{i'}\cdot \frac{\la \wb^{(t)}, \xb_{i'}\ra}{\|\wb^{(t)} \|_{\bSigma}} - y_{i}\cdot\frac{\la \wb^{(t)}, \xb_{i}\ra}{ \| \wb^{(t)} \|_{\bSigma}} \bigg)^2 \\
%     % % &\quad = O\Bigg\{ \frac{\lambda_{\max} \cdot  \|\Pb_{\cX}\wb^{(0)} \|_2^2}{\lambda_{\min} \cdot \max_i \|\xb_i \|_2^2\cdot \| \wb^* \|_2^2}\cdot  \exp\Bigg[ -\frac{ \lambda_{\min}  }{512\lambda_{\max}^{3/2} \cdot \| \Pb_{\cX}\wb^{(0)} \|_2^2} \cdot \log^2( (8/9) \eta\cdot (t - t_0)+ 1)\Bigg] \Bigg\}
%     % &\quad\quad 
%     D(\wb^{(t)})\leq \frac{\lambda_{\max}}{\lambda_{\min} } \cdot  \|\Pb_{\cX}\wb^{(0)} \|_2^2\cdot  \exp\Bigg[ -\frac{ \lambda_{\min}  }{512\lambda_{\max}^{3/2} \cdot \| \Pb_{\cX}\wb^{(0)} \|_2^2} \cdot \log^2( (8/9) \eta\cdot (t - t_0)+ 1)\Bigg]
% \end{align*}
% for all $t\geq t_0$.
\end{theorem}
By the first result in Theorem~\ref{thm:linearBN}, we see that the loss function converges to zero, which guarantees perfect fitting on the training dataset. Moreover, the upper and lower bounds together demonstrate that the $\Theta(1/t)$ convergence rate is sharp. The second result in Theorem~\ref{thm:linearBN} further shows that the margin discrepancy $D(\wb)$ converges to zero at a rate of $O(\exp(-\log^2t))$, demonstrating that batch normalization in linear models has an implicit bias towards the uniform margin. Note that Theorem~\ref{thm:linearBN} holds under mild assumptions: we only require that (i) a uniform margin solution exists (which is obviously a necessary assumption), and (ii) the initialization scale $\| \wb^{(0)} \|_2$ and the learning rate $\eta$ are small enough (which are common in existing implicit bias results \citep{gunasekar2017implicit,li2018algorithmic,arora2019implicit}). Therefore, Theorem~\ref{thm:linearBN} gives a comprehensive characterization on the implicit bias of batch normalization in linear models. 

\noindent\textbf{Comparison with the implicit bias for linear models without BN.}
When the BN is absent, it has been widely known that gradient descent, performed on linear models with cross-entropy loss,   converges to the maximum margin solution with a $O\big(1/\log(t)\big)$ rate \citep{soudry2017implicit}.
In comparison, our result demonstrates that when batch normalization is added to the model, the implicit bias is changed from maximum margin to uniform margin and the (margin) convergence rate can be significantly improved to $\exp(-\Omega(\log^2t))$, which is way faster than $O(1/\log(t))$.
% The implicit bias of logistic regression for linear models and linear networks has been relatively well studied by a line of recent works \citep{soudry2017implicit,gunasekar2018characterizing,gunasekar2018implicit,ji2019implicit}. 
% They demonstrated that gradient descent is guaranteed to achieve maximum margin on the training data, and showed a convergence rate of $1/\log(t)$ towards maximum margin \CCC{for linear model}.
% Moreover, it is also clear that batch normalization can significantly increase the convergence speed to $O(\exp(-\log^2t))$, which is faster than a $O(1/t^k)$ rate for all constant $k$.

\noindent\textbf{Comparison with the implicit bias for homogeneous models.} Note that when BN is added, the model function $f(\wb,\gamma,\xb)$ is  $1$-homogeneous for any $\wb\neq \boldsymbol{0}$, i.e., for any constant $c$, we have $f(c\wb,c\gamma,\xb) = c\cdot f(\wb,\gamma,\xb)$. Therefore, by the implicit bias result for general homogeneous models in \citet{lyu2019gradient}, we can conclude that the uniform margin solution is a KKT point of the maximum margin problem. On the other hand, it is clear that the uniform margin solution may not be able to achieve the maximum margin on the training data. This implies that for general homogeneous models (or homogeneous models with BN, which is still a class of homogeneous models), it is  possible that gradient descent will not converge to the maximum margin solution. 

\section{Batch Normalization in Two-layer Linear CNNs} \label{section:CNN_model}
In Section~\ref{section:linear_model}, we have shown that batch normalization in linear models has an implicit bias towards the uniform margin. This actually reveals the fact that linear predictors with batch normalization in logistic regression may behave more similarly to the linear regression predictor trained by minimizing the square loss. To further distinguish  batch normalization from other methods, in this section we extend our results in Section~\ref{section:linear_model} to a class of  linear CNNs with a single convolution filter. 

Suppose that $(\xb_1,y_1),\ldots, (\xb_n,y_n)$ are $n$ training data points, and each $\xb_i$ consists of $P$ patches $\xb_i = [\xb_i^{(1)},\xb_i^{(2)},\ldots, \xb_i^{(P)}]$, where $\xb_i^{(p)}\in \RR^d$, $p\in [P]$. We train a linear CNN with a single filter $\wb$ to fit these data points. During training, the CNN model with batch normalization is given as 
\begin{align*}
g(\wb,\gamma,\xb)  &= \sum_{p=1}^P \gamma\cdot \frac{\la \wb, \xb^{(p)} \ra}{\| \wb \|_{\bSigma} } ~~\text{with}~~\bSigma = \frac{1}{nP} \sum_{i=1}^n \sum_{p=1}^P \xb_{i}^{(p)}\xb_{i}^{(p)\top},
\end{align*}
where $\gamma$ is the scale parameter in batch normalization. Here $\wb,\gamma$ are both trainable parameters. Similar to the linear model case, the above definition can be obtained by centering the training data points. Moreover, note that here different patches of a data point $\xb^{(1)},\ldots,\xb^{(P)}$ can have arbitrary overlaps, and our theory applies even when some of the patches are identical. 
Note also that in our definition, batch normalization is applied not only over the batch of data, but also over patches, which matches the definition of batch normalization in CNN in \citet{ioffe2015batch}. 

We again consider training $g(\wb,\gamma,\xb)$ by minimizing the cross-entropy loss 
% \begin{align*}
%     L(\wb,\gamma) = \frac{1}{n}\sum_{i=1}^n \ell[y_i\cdot g(\wb,\gamma,\xb_i)]
% \end{align*}
with gradient descent starting from initialization $\wb^{(0)}$, $\gamma^{(0)}$. 
% Starting from $\wb^{(0)}$, $\gamma^{(0)}$, gradient descent gives
% \begin{align*}
%     &\wb^{(t+1)} = \wb^{(t)} - \frac{\eta }{n\cdot \| \wb^{(t)} \|_{\bSigma}}\sum_{i=1}^n \sum_{p=1}^P \ell'[y_i\cdot f(\wb^{(t)},\gamma^{(t)},\xb_i)] \cdot y_i \cdot \gamma^{(t)} \cdot \bigg( \Ib -  \frac{\bSigma \wb^{(t)} \wb^{(t)\top}}{\| \wb^{(t)} \|_{\bSigma}^2} \bigg)\xb_i^{(p)}, \\
%     &\gamma^{(t+1)} = \gamma^{(t)} - \frac{\eta }{n}\sum_{i=1}^n  \sum_{p=1}^P \ell'[y_i\cdot f(\wb^{(t)},\gamma^{(t)},\xb_i)] \cdot y_i \cdot \frac{\la \wb^{(t)}, \xb_i^{(p)} \ra}{\| \wb^{(t)} \|_{\bSigma}}, 
% \end{align*}
% where $\eta$ is the learning rate. 
As the counterpart of the margin discrepancy $D(\wb)$ defined in Section~\ref{section:linear_model}, here we define the patch-wise margin discrepancy as
\begin{align*}
    D_{\mathrm{patch}}(\wb) = \frac{1}{n^2P^2}\sum_{i,i'=1}^n \sum_{p,p'=1}^P   \bigg(y_{i'}\cdot \frac{\la \wb , \xb_{i'}^{(p')}\ra}{\| \wb \|_{\bSigma}} - y_{i}\cdot \frac{\la \wb , \xb_{i}^{(p)}\ra}{\| \wb \|_{\bSigma}}\bigg)^2.
\end{align*}
We call $g(\wb,\gamma,\xb)$ or the corresponding linear classifier defined by $\wb$ the \textit{patch-wise uniform margin classifier} if $D_{\mathrm{patch}}(\wb) = 0$.

\subsection{The Implicit Bias of Batch Normalization in Two-Layer Linear CNNs}
Similar to Assumption \ref{assump:uniformly_separable}, we make the following assumption to guarantee the existence of the patch-wise uniform margin solution.
\begin{assumption}\label{assump:uniformly_separable_CNN}
The equation system $\la \wb, y_i\cdot \xb_i^{(p)} \ra = 1$, $i\in [n]$, $p\in[P]$ has at least one solution. 
%There exists $\wb^*$ such that $\la \wb^*, y_i\cdot \xb_i^{(p)} \ra = 1$ for all $i\in [n]$ and $p\in[P]$.
\end{assumption}
We also extend the notations in Section~\ref{section:linear_model_result} to the multi-patch setting as follows. We define $\cX=\mathrm{span}\{\xb_i^{(p)}, i\in[n], p\in[P]\}$, and 
\begin{align*}
    \lambda_{\max} = \sup_{\ub \in \cX \backslash \{\mathbf{0}\}} \frac{\ub^\top \bSigma \ub}{\| \ub \|_2^2}, \qquad \lambda_{\min} = \inf_{\ub \in \cX \backslash \{\mathbf{0}\}} \frac{\ub^\top \bSigma \ub}{\| \ub \|_2^2}.
\end{align*}
Moreover, we let $\Pb_{\cX}$ be the projection matrix onto $\cX$.

The following theorem states the implicit bias of batch normalization in two-layer linear CNNs.
\begin{theorem}\label{thm:CNNBN}
Suppose that Assumption~\ref{assump:uniformly_separable_CNN} holds. 
There exist $M = 1 / \poly( \lambda_{\min}^{-1}, \lambda_{\max}, \max_i \| \xb_i \|_2, P )$, $\overline{\eta} =  1 / \poly( \lambda_{\min}^{-1}, \lambda_{\max}, \max_i \| \xb_i \|_2, P, \| \Pb_{\cX} \wb^{(0)} \|_2^{-1})$ and constants $C_1,C_2,C_3,C_4 > 0$, such that when $  \gamma^{(0)} = 1$, $\| \Pb_{\cX} \wb^{(0)} \|_2 \leq M$, and $\eta \leq \overline{\eta}$, there exists  
$t_0 \leq \eta^{-1}\cdot \poly( \lambda_{\min}^{-1}, \lambda_{\max}, \max_i \| \xb_i \|_2 , P) $
and the following inequalities hold for all $t\geq t_0$:
\begin{align*}
    &\frac{C_1}{\eta  P\cdot (t - t_0 + 1)} \leq L(\wb^{(t)},\gamma^{(t)}) \leq \frac{C_2}{\eta  P\cdot (t - t_0 + 1)},\\
    &D_{\mathrm{patch}}(\wb^{(t)})\leq \frac{C_3\lambda_{\max}}{\lambda_{\min} } \cdot \exp\Bigg[ -\frac{ C_4\lambda_{\min}  }{\lambda_{\max}^{3/2} \cdot \| \Pb_{\cX}\wb^{(0)} \|_2^2} \cdot \log^2( (8/9) \eta P\cdot (t - t_0)+ 1)\Bigg].
\end{align*}
\end{theorem}
Theorem~\ref{thm:CNNBN} is the counterpart of Theorem~\ref{thm:linearBN} for two-layer single-filter CNNs. It further reveals that fact that batch normalization encourages CNNs to achieve the same margin on all data patches. It is clear that the convergence rate is $\exp(-\Omega(\log^2t))$, which is fast compared with the convergence towards the maximum classifier when training linear CNNs without batch normalization.

\subsection{Examples Where Uniform Margin Outperforms Maximum Margin}\label{subsection:data_examples}
Here we give two examples of learning problems, and show that when training two-layer, single-filter linear CNNs, the patch-wise uniform classifier given by batch normalization outperforms the maximum margin classifier given by the same neural network model without batch normalization.

% We first recall that a two-layer, single filter linear CNN with batch normalization is given by 
% \begin{align*}
% f(\wb,\gamma,\xb)  &= \sum_{p=1}^P \gamma\cdot \frac{\la \wb, \xb^{(p)} \ra}{\| \wb \|_{\bSigma} }. %~~\text{with}~~\bSigma = \frac{1}{nP} \sum_{i=1}^n \sum_{p=1}^P \xb_{i}^{(p)}\xb_{i}^{(p)\top}.
% \end{align*}
We note that when making predictions on a test data input $\xb_{\mathrm{test}}$, the standard deviation calculated in batch normalization (i.e., the denominator $\| \wb \|_{\bSigma}$) is still based on the training data set and is therefore unchanged. Then as long as $\gamma > 0$, we have
\begin{align*}
    \mathrm{sign}[g(\wb,\gamma,\xb_{\mathrm{test}})] = \mathrm{sign}\Bigg[ \sum_{p=1}^P \gamma\cdot \frac{\la \wb, \xb^{(p)} \ra}{\| \wb \|_{\bSigma} } \Bigg] = \mathrm{sign}( \la \wb ,\overline\xb_{\mathrm{test}}\ra ),
\end{align*}
 where we define $\overline\xb_{\mathrm{test}} := \sum_{p=1}^P \xb_{\mathrm{test}}^{(p)}$. We compare the performance of the patch-wise uniform margin classifier $\wb_{\mathrm{uniform}}$ 
 with the maximum margin classifier, which is defined as
\begin{align}\label{eq:max_margin_def}
    \wb_{\max} = \argmin  \| \wb \|_2^2, ~~\text{subject to } y_i \cdot \la\wb , \overline\xb_i\ra \geq 1, i\in [n].
\end{align}
By \citet{soudry2017implicit}, $\wb_{\max}$ can be obtained by training a two-layer, single-filter CNN without batch normalization. We hence study the difference between $\wb_{\mathrm{uniform}}$ and $\wb_{\max}$ in two examples. Below we present the first learning problem example and the learning guarantees.

\begin{example}\label{def:data_example1}
Let $\ub\in \RR^d$ be a fixed vector. Then each data point $(\xb,y)$ with 
$$\xb_i = [\xb_i^{(1)},\xb_i^{(2)},\ldots, \xb_i^{(P)}]\in\RR^{d\times P}$$ 
and $y\in\{-1,1\}$ is generated from the distribution $\cD_1$ as follows: 
\begin{itemize}[leftmargin = *]
    \item The label $y$ is generated as a Rademacher random variable.
    \item For $p\in [P]$, the input patch $\xb^{(p)}$ is given as $y\cdot \ub + \bxi^{(p)}$, where $\bxi^{(p)}$, $p\in[n]$ are independent noises generated from $N(\mathbf{0}, \sigma^2 (\Ib - \ub \ub^\top / \|\ub \|_2^2) )$.  
\end{itemize}
\end{example}

% We have the following guarantees for the patch-wise uniform classifier and the maximum margin classifier when learning from the data distribution $\cD_1$.

\begin{theorem}\label{thm:data_example1}
Let $S = \{(\xb_i,y_i)\}_{i=1}^n$ be the training data set consisting of $n$ independent data points drawn from the distribution $\cD$ in Example~\ref{def:data_example1}. Suppose that $d = 2n$, $P \geq 4$ and $\sigma \geq  20 \| \ub \|_2 \cdot  P^{1/2} d^{1/2} $. 
% Then with probability $1$, the maximum margin classifier and the patch-wise uniform margin classifier $\wb_{\mathrm{uniform}}$ exist. Moreover, with probability at least $1 - \exp(-\Omega(d))$ with respect to the randomness in the training data, the following results hold:
Then with probability $1$, the maximum margin classifier $\wb_{\max}$ on $S$ exists and is unique, and the patch-wise uniform margin classifier $\wb_{\mathrm{uniform}}$ on $S$ exist, and is unique up to a scaling factor. Moreover, with probability at least $1 - \exp(-\Omega(d))$ with respect to the randomness in the training data, the following results hold:
\begin{itemize}[leftmargin = *]
    \item $\PP_{(\xb_{\mathrm{test}},y_{\mathrm{test}})\sim \cD_1}( y_{\mathrm{test}}\cdot \la \wb_{\mathrm{uniform}}, \overline\xb_{\mathrm{test}} \ra < 0 ) = 0$.
    \item $\PP_{(\xb_{\mathrm{test}},y_{\mathrm{test}})\sim \cD_1}( y_{\mathrm{test}}\cdot \la \wb_{\max}, \overline\xb_{\mathrm{test}} \ra < 0 ) = \Theta(1)$.
\end{itemize}
\end{theorem}

Below we give the second learning problem example and the corresponding learning guarantees for the patch-wise uniform margin classifier and the maximum margin classifier.

\begin{example}\label{def:data_distribution_feat_noise}
Let $\xb=[\xb^{(1)},\xb^{(2)}]$ be the data with two patches, where $\xb^{(p)}\in\RR^d$ is a $d$-dimensional vector.  Let $\ub$ and $\vb$ be two fixed vectors and $\rho\in[0,0.5)$. Then given the Rademacher label $y\in\{-1,1\}$, the data input $\xb$ is generated from the distribution $\cD_2$ as follows:
\begin{itemize}[leftmargin = *]
\item With probability $1-\rho$, one data patch $\xb^{(p)}$ is the strong signal $y\ub$ and the other patch is the random Gaussian noise $\bxi\sim N\big(0,\sigma^2(\Ib-\ub\ub^\top/\|\ub\|_2^2-\vb\vb^\top/\|\vb\|_2^2\big)\big)$.
\item With probability $\rho$, one data patch $\xb^{(p)}$ is the weak signal $y\vb$ and the other patch is the combination of random noise $\bxi\sim N(0,\sigma^2\Ib)$ and feature noise $\alpha\cdot\zeta \ub$, where $\zeta$ is randomly drawn from $\{-1,1\}$ equally and $\alpha\in(0,1)$.
\end{itemize}
The signals are set to be orthogonal to each other, i.e., $\la\ub,\vb\ra=0$. Moreover, we set $d=n^2\log(n)$, $\sigma_0=d^{-1/2}$, $\rho=n^{-3/4}$, $\alpha=n^{-1/2}$, $\|\ub\|_2=1$, and $\|\vb\|_2=\alpha^2$.
\end{example}
\begin{theorem}\label{thm:data_example2}
Suppose that the data is generated according to  Example \ref{def:data_distribution_feat_noise}, then let $ \wb_{\mathrm{uniform}}$  and $\wb_{\max}$ be the  uniform margin and maximum margin solution in the subspace $\cX$ respectively. Then with probability at least $1 - \exp(-\Omega(n^{1/4}))$ with respect to the randomness in the training data, the following holds:
\begin{itemize}[leftmargin = *]
    \item $\PP_{(\xb_{\mathrm{test}},y_{\mathrm{test}})\sim \cD_2}( y\cdot \la \wb_{\mathrm{uniform}}, \overline\xb_{\mathrm{test}} \ra < 0 ) \le \frac{1}{n^{10}}$.
    \item $\PP_{(\xb_{\mathrm{test}},y_{\mathrm{test}})\sim \cD_2}( y\cdot \la \wb_{\max}, \overline\xb_{\mathrm{test}} \ra < 0 ) \ge \frac{1}{4n^{3/4}}$.
\end{itemize}
\end{theorem}

We note that data models similar to Examples~\ref{def:data_example1} and \ref{def:data_distribution_feat_noise} have been considered in a series of works \citep{allen2020towards,zou2021understanding,cao2022benign}. According to Theorems~\ref{thm:data_example1} and \ref{thm:data_example2}, the patch-wise uniform margin classifier achieves better test accuracy than the maximum margin classifier in both examples. The intuition behind these examples is that by ensuring a patch-wise uniform margin, the classifier amplifies the effect of weak, stable features over strong, unstable features and noises. We remark that it is not difficult to contract more examples where patch-wise uniform margin classifier performs better. However, we can also similarly construct other examples where the maximum margin classifier gives better predictions. The goal of our discussion here is just to demonstrate that there exist such learning problems where patch-wise uniform margin classifiers perform well.

% the vectors $y_{i'}\cdot \bxi_{i'}^{(p')} - y_{i}\cdot\bxi_{i}^{(p)}$ for $i,i'\in [n]$ and $p,p'\in [P]$ span the whole 

% We prove the two results for maximum margin and patch-wise uniform margin classifiers separately as follows. 

% \noindent\textbf{Proof for the patch-wise uniform margin classifier $\wb_{\mathrm{uniform}}$.} By the definition of patch-wise uniform margin classifier, we have
% \begin{align*}
%     y_{i'}\cdot \la \wb , \xb_{i'}^{(p')}\ra - y_{i}\cdot \la \wb , \xb_{i}^{(p)}\ra
% \end{align*}
% for all $i,i'\in [n]$ and $p,p'\in [P]$. By the data model, $\xb_i^{(p)} = y_i\cdot \ub + \bxi_i^{(p)}$ for $i\in [n]$, $p\in [P]$, where $\bxi_i^{(p)} \sim N(\mathbf{0}, \sigma^2 \Ib)$. Therefore, we have
% \begin{align*}
%     \la \wb , y_{i'}\cdot \bxi_{i'}^{(p')} - y_{i}\cdot\bxi_{i}^{(p)}\ra = 0 %-  \la \wb , y_{i}\cdot\bxi_{i}^{(p)}\ra = 0
% \end{align*}
% for all $i,i'\in [n]$ and $p,p'\in [P]$.

% \noindent\textbf{Proof for the maximum margin classifier $\wb_{\max}$.}

\section{Overview of the Proof Technique}\label{section:technique_overview}
In this section, we explain our key proof techniques in the study of implicit bias of batch normalization, and discuss the key technical contributions in this paper. For clarity, we will mainly focus on the setting of batch normalization in linear models as is defined in Section~\ref{section:linear_model}. %The proofs for linear CNNs are given in Appendix~XXX.

We first introduce some simplification of notations. As we discussed in Subsection~\ref{section:linear_model_result}, the training of $\wb^{(t)}$ always happen in the subspace $\cX = \mathrm{span}\{\xb_1,\ldots,\xb_n\}$. And when $\cX \subsetneq \RR^d$, we need the projection matrix $\Pb_{\cX}$ in our result. In fact, under the setting where $\cX \subsetneq \RR^d$, we need to apply such a projection whenever $\wb$ occurs, and the notations  can thus be quite complicated. To simplify notations, throughout our proof we use the slight abuse of notation
\begin{align*}
    \| \ab \|_2 := \| \Pb_{\cX} \ab \|_2,\qquad \la \ab , \bbvec \ra = \la \Pb_{\cX}\ab , \Pb_{\cX}\bbvec \ra
\end{align*}
for all $\ab, \bbvec \in \RR^d$. Then by the definition of $\lambda_{\max},\lambda_{\min}$, we have $\lambda_{\min}\cdot \| \ab \|_2^2  \leq \| \ab \|_{\bSigma}^2\leq \lambda_{\max} \cdot \| \ab \|_2^2$ 
% \begin{align*}
%     \lambda_{\min}\cdot \| \ab \|_2^2  \leq \| \ab \|_{\bSigma}^2\leq \lambda_{\max} \cdot \| \ab \|_2^2
% \end{align*}
for all $\ab \in \RR^d$. 
Moreover, let $\wb^*\in \cX$ be the unique vector satisfying
\begin{align}\label{eq:definition_w*}
    \wb^*\in \cX, \quad \la \wb^*, y_i\cdot \xb_i \ra = 1,\quad i \in [n],
\end{align}
% solution of the equation system $\la \wb, y_i\cdot \xb_i \ra = 1$, $i \in [n]$ in $\cX$, and 
and denote $\zb_i = y_i\cdot \xb_i$, $i \in [n]$. Then it is clear that $\la \wb^*, \xb_i \ra = 1$ for all $i\in [n]$. In our analysis, we frequently encounter the derivatives of the cross entropy loss on each data point. Therefore we also denote $ \ell'_i = \ell'[y_i\cdot f(\wb,\gamma,\xb_i)]$, $\ell_i'^{(t)} = \ell'[y_i\cdot f(\wb^{(t)},\gamma^{(t)},\xb_i)]$, $i\in[n]$. %introduce the notations 
% \begin{align*}
%     \ell'_i = \ell'[y_i\cdot f(\wb,\gamma,\xb_i)],\quad \ell_i'^{(t)} = \ell'[y_i\cdot f(\wb^{(t)},\gamma^{(t)},\xb_i)], \quad i\in[n].
% \end{align*}

% $ \ell'_i = \ell'[y_i\cdot f(\wb,\gamma,\xb_i)]$, $ \ell_i'^{(t)} = \ell'[y_i\cdot f(\wb^{(t)},\gamma^{(t)},\xb_i)]$, $i\in[n]$.

% We now discuss the key challenges and key proof techniques in the following subsections. 

\subsection{Positive Correlation Between $\wb^*$ and the Gradient Update}

In order to study the implicit bias of batch normalization, the first challenge is to identify a proper target to which the predictor converges. In our analysis, this proper target, i.e. the uniform margin solution, is revealed by a key identity which is presented in the following lemma.

% uniform margin solution as the target of the analysis, and mathematically build connection between the margin discrepancy and the optimization procedure. 

% that uniform margin is a reasonable hypothesis. 

% , and mathematically relate
% \begin{align}\label{eq:def_w*_linear_model}
%     \wb^* := \argmin_{\wb\in \cX} \| \wb \|_2^2, ~~\text{subject to }  \la \wb, y_i\cdot \xb_i \ra = 1, ~i \in [n],
% \end{align}

% Based on Assumption~\ref{assump:uniformly_separable}, we can define $\wb^*$ as the minimum norm solution of the system:
% \begin{align}\label{eq:def_w*_linear_model}
%     \wb^* := \argmin_{\wb} \| \wb \|_2^2, ~~\text{subject to }  \la \wb, y_i\cdot \xb_i \ra = 1, ~i \in [n].
% \end{align}
% $\wb^* = (\Zb\Zb^\top)^\dagger \Zb^\top \mathbf{1}$

% There exists $\wb^*$ such that $\la \wb^*, y_i\cdot \xb_i \ra = 1$ for all $i\in [n]$.

\begin{lemma}\label{lemma:gradient_inner_linear_model}
Under Assumption~\ref{assump:uniformly_separable}, for any $\wb\in \RR^d$, it holds that 
\begin{align*}
    &\la -\nabla_{\wb} L(\wb,\gamma), \wb^* \ra \\
    &\qquad\quad = \frac{\gamma}{2 n^2 \| \wb \|_{\bSigma}^3}\sum_{i,i'=1}^n |\ell'_i|\cdot |\ell'_{i'}| \cdot (\la \wb , \zb_{i'}\ra - \la \wb , \zb_{i}\ra)\cdot ( |\ell'_{i'}|^{-1} \cdot \la \wb , \zb_{i'}\ra - |\ell'_{i}|^{-1} \cdot \la \wb , \zb_{i}\ra ).
\end{align*}
% where $ \ell'_i = \ell'[y_i\cdot f(\wb,\gamma,\xb_i)]$ and $\zb_i = y_i\cdot  \xb_i$, $i\in[n]$.
\end{lemma}
Recall that $\wb^*$ defined in \eqref{eq:definition_w*} is a uniform margin solution. Lemma~\ref{lemma:gradient_inner_linear_model} thus gives an exact calculation on the component of the training loss gradient pointing towards a uniform margin classifier. More importantly, we see that the factors $|\ell'_{i}|^{-1}$, $|\ell'_{i'}|^{-1}$ are essentially also functions of $ \la \wb , \zb_{i}\ra$ and $\la \wb , \zb_{i'}\ra$ respectively. If the margins of the predictor on a pair of data points $(\xb_i,y_i)$ and $(\xb_{i'},y_{i'})$ are not equal, i.e.,
$ \la \wb , \zb_{i}\ra \neq \la \wb , \zb_{i'}\ra$, then we can see from Lemma~\ref{lemma:gradient_inner_linear_model} that a gradient descent step on the current predictor will push the predictor towards the direction of $\wb^*$ by a positive length. To more accurately characterize this property, we give the following lemma.

\begin{lemma}\label{lemma:innerproduct_lowerbound}
For all $t \geq 0$, it holds that
\begin{align*}
%  &
%  \la \wb^{(t+1)}, \wb^* \ra \geq \la \wb^{(t)}, \wb^* \ra  + \frac{ \gamma^{(t)} \eta }{16 n^2\| \wb^{(t)} \|_{\bSigma}^3}\cdot \exp(-\gamma^{(t)} ) \cdot  \sum_{i,i'=1}^n  (\la \wb^{(t)} , \zb_{i'}\ra - \la \wb^{(t)} , \zb_{i}\ra)^2,
%     \\
    &\la \wb^{(t+1)}, \wb^* \ra \geq \la \wb^{(t)}, \wb^* \ra + \frac{ \gamma^{(t)} \eta}{2 \| \wb^{(t)} \|_{\bSigma}^2} \cdot \max\bigg\{ \frac{\exp(-\gamma^{(t)} )}{8} , \min_i |\ell_i'^{(t)}| \bigg\} \cdot D(\wb^{(t)}),
\end{align*}
% where $ \ell_i'^{(t)} = \ell'[y_i\cdot f(\wb^{(t)},\gamma^{(t)},\xb_i)]$, $i\in[n]$.
% \begin{align*}
%  &
%  \la \wb^{(t+1)}, \wb^* \ra \geq \la \wb^{(t)}, \wb^* \ra  + \frac{ \gamma^{(t)} \eta }{16 n^2\| \wb^{(t)} \|_{\bSigma}^3}\cdot \exp(-\gamma^{(t)} ) \cdot  \sum_{i,i'=1}^n  (\la \wb^{(t)} , \zb_{i'}\ra - \la \wb^{(t)} , \zb_{i}\ra)^2,
%     \\
%     &\la \wb^{(t+1)}, \wb^* \ra \geq \la \wb^{(t)}, \wb^* \ra + \frac{ \gamma^{(t)} \eta}{2 \| \wb^{(t)} \|_{\bSigma}^3}  \cdot D(\wb^{(t)}).
% \end{align*}
\end{lemma}
Lemma~\ref{lemma:innerproduct_lowerbound} is established based on Lemma~\ref{lemma:gradient_inner_linear_model}. It shows that as long as the margin discrepancy is not zero, the inner product $\la \wb^{(t)}, \wb^* \ra$ will increase during training. It is easy to see that the result in Lemma~\ref{lemma:innerproduct_lowerbound}  essentially gives two inequalities:
\begin{align}
    &\la \wb^{(t+1)}, \wb^* \ra \geq \la \wb^{(t)}, \wb^* \ra + \frac{ \gamma^{(t)} \eta}{2 \| \wb^{(t)} \|_{\bSigma}^2} \cdot \frac{\exp(-\gamma^{(t)} )}{8}  \cdot D(\wb^{(t)}),\label{eq:innerproduct_lowerbound1}\\
    &\la \wb^{(t+1)}, \wb^* \ra \geq \la \wb^{(t)}, \wb^* \ra + \frac{ \gamma^{(t)} \eta}{2 \| \wb^{(t)} \|_{\bSigma}^2} \cdot \min_i |\ell_i'^{(t)}|  \cdot D(\wb^{(t)}).\label{eq:innerproduct_lowerbound2}
\end{align}
In fact, inequality \eqref{eq:innerproduct_lowerbound2} above with the factor $ \min_i |\ell_i'^{(t)}|$ is relatively easy to derive -- we essentially lower bound each $ |\ell_i'^{(t)}|$ by the minimum over all of them. In comparison, inequality \eqref{eq:innerproduct_lowerbound1} with the factor $\exp(-\gamma^{(t)} )$ is highly nontrivial: in the early stage of training where the predictor may have very different margins on different data, we can see that $y_i\cdot f(\wb^{(t)},\gamma^{(t)},\xb_i) = \gamma^{(t)}\cdot \la\wb^{(t)}, y_i\cdot \xb_i \ra \cdot (n^{-1} \sum_j \la\wb^{(t)}, \xb_j \ra^2 )^{-1/2}$ can be as large as $\gamma^{(t)} \cdot \sqrt{n}$ when only one inner product among $\{ \la \wb^{(t)}, \xb_i \ra \}_{i=1}^n$ is large and the other are all close to zero. Hence, $\min_i |\ell_i'^{(t)}|$ can be smaller than $\exp(-\gamma^{(t)} \cdot \sqrt{n})$ in the worst case. Therefore, \eqref{eq:innerproduct_lowerbound1} is tighter than \eqref{eq:innerproduct_lowerbound2} when the margins of the predictor on different data are not close. This tighter result is proved based on a technical inequality (see Lemma~\ref{lemma:auxiliary_inequality}), and deriving this result is one of the key technical contributions of this paper.  
% This observation further leads to the following lower bound 
\subsection{Equivalent Metrics of Margin Discrepancy and Norm Bounds}
% \subsection{Margin Discrepancy and Its Relation to Two Different Metrics}
% Lemma~\ref{lemma:innerproduct_lowerbound}
% We note that Lemma~\ref{lemma:innerproduct_lowerbound} alone is not sufficient to demonstrate the convergence to a uniform margin: 
% \begin{itemize}[leftmargin = *]
%     \item The result in Lemma~\ref{lemma:innerproduct_lowerbound} involves the inner products $\la \wb^{(t+1)}, \wb^* \ra$, $\la \wb^{(t)}, \wb^* \ra $ as well as the margin discrepancy $D(\wb^{(t)})$. The inner products and the margin discrepancy are essentially different metrics on how uniform the margins are. To proceed, we need to unify these different metrics. 
%     \item If the gradient consistently has a even larger component along a different direction, then the predictor, after normalization, will in fact be pushed far away from the uniform margin solution. Therefore it is necessary to control the growth of $\wb^{(t)}$ along the other directions.
% \end{itemize}
% In this subsection we present the key lemmas handling the two points above. We have the following lemma addressing the first point above.
We note that the result in Lemma~\ref{lemma:innerproduct_lowerbound} involves the inner products $\la \wb^{(t+1)}, \wb^* \ra$, $\la \wb^{(t)}, \wb^* \ra $ as well as the margin discrepancy $D(\wb^{(t)})$. The inner products and the margin discrepancy are essentially different metrics on how uniform the margins are. To proceed, we need to unify these different metrics. We have the following lemma addressing this issue.
\begin{lemma}\label{lemma:key_identity}
For any $\wb \in \RR^d$, it holds that
\begin{align*}
    &\| \wb \|_{\bSigma}^2 \cdot D(\wb) = \frac{1}{n^2}\sum_{i,i'=1}^n   (\la \wb , \zb_{i'}\ra - \la \wb , \zb_{i}\ra)^2 = 2\|  \wb - \la \wb^*, \wb\ra_{\bSigma} \cdot \wb^* \|_{\bSigma}^2,\\
    &\lambda_{\min}\cdot \bigg\|  \wb -  \frac{\la \wb^*, \wb\ra}{\| \wb^* \|_2^2} \cdot \wb^* \bigg\|_2^2\leq \|  \wb - \la \wb^*, \wb\ra_{\bSigma} \cdot \wb^* \|_{\bSigma}^2 \leq \lambda_{\max}\cdot \bigg\|  \wb -  \frac{\la \wb^*, \wb \ra}{\| \wb^* \|_2^2} \cdot \wb^* \bigg\|_2^2.
\end{align*}
\end{lemma}
By Lemma~\ref{lemma:key_identity}, it is clear that the $\| \wb \|_{\bSigma} \cdot \sqrt{D(\wb)}$, $ \|  \wb - \la \wb^*, \wb\ra_{\bSigma} \cdot \wb^* \|_{\bSigma} $ and $\|  \wb -  \| \wb^* \|_2^{-2}\cdot \la \wb^*, \wb \ra \cdot \wb^* \|_2$ are equivalent metrics on the distance  between $\wb$ and $\mathrm{span}\{ \wb^*\}$. Lemma~\ref{lemma:key_identity} is fundamental throughout our proof as it unifies (i) the Euclidean geometry induced by linear model and gradient descent and (ii) the geometry defined by $\bSigma$ induced by batch normalization. Eventually, Lemma~\ref{lemma:key_identity} converts both metrics to the margin discrepancy, which is essential in our proof. 

Even with the metric equivalence results,  Lemma~\ref{lemma:innerproduct_lowerbound} alone is still not sufficient to demonstrate the convergence to a uniform margin: If the gradient consistently has a even larger component along a different direction, then the predictor, after normalization, will be pushed farther away from the uniform margin solution. Therefore, it is necessary to control the growth of $\wb^{(t)}$ along the other directions. To do so, we give upper bounds on $\| \wb^{(t)} \|_2$. Intuitively, the change of $\wb$ in the radial direction will not change the objective function value as $f(\wb,\gamma,\xb)$ is $0$-homogeneous in $\wb$. Therefore, the training loss gradient is orthogonal to $\wb$, and when the learning rate is  small, $\| \wb^{(t)} \|_2$ will hardly change during training. The following lemma is established following this intuition.

% in the case of gradient flow, the training dynamic should be a rotation of $\wb$ and 
% $\| \wb \|_2$
% The following lemma handles this issue by give upper bounds of $ \| \wb^{(t)} \|_2$.
% By the definition of margin discrepancy $D(\wb)$, the term $n^{-2}\sum_{i,i'=1}^n   (\la \wb , \zb_{i'}\ra - \la \wb , \zb_{i}\ra)^2$ is equal to $\| \wb \|_{\bSigma} \cdot $
% if the gradient consistently has a even larger component along a different direction, then the predictor will be effectively pushed far away from the uniform margin solution. Therefore we also present the following result on the changes of $\| \wb^{(t)} \|_2$ during training. 
% we also present the following result on the changes of $\| \wb^{(t)} \|_2$ during training.

\begin{lemma}\label{lemma:norm_upperbound}
For all $t \geq 0$, it holds that
\begin{align*}
    \| \wb^{(t)}\|_2^2 \leq \|\wb^{(t+1)}\|_2^2 \leq \|\wb^{(t)}\|_2^2 + 4 \eta^2 \cdot  \frac{ \gamma^{(t)2} \cdot \max_i \| \xb_i \|_2^3 }{ \lambda_{\min}^2 \cdot \| \wb^{(t)} \|_2^2}.
\end{align*}
% $ $. 
Moreover, if $ \|  \wb^{(t)} -  \la \wb^{(t)}, \wb^* \ra \cdot \|\wb^*\|_2^{-2} \cdot \wb^* \|_2 \leq \| \wb^{(0)} \|_2 / 2 $, then
\begin{align*}
    \|\wb^{(t+1)}\|_2^2
    \leq \| \wb^{(t)}\|_2^2 + &\eta^2  G\cdot H^{(t)} \cdot  \bigg\|  \wb^{(t)} -  \frac{\la \wb^*, \wb^{(t)}\ra}{\| \wb^* \|_2^2} \cdot \wb^* \bigg\|_2^2,
\end{align*}
where 
$G = 64 \lambda_{\min}^{-3}\cdot \max_i \|\xb_{i}\|_2^6 \cdot \| \wb^{(0)} \|_2^{-4}$, and $H^{(t)} = \max\{ |\ell_1'^{(t)}|^2, \ldots, |\ell_n'^{(t)}|^2, \exp(-2\gamma^{(t)}) \} \cdot \max\{\gamma^{(t)2}, \gamma^{(t)4}\}$. 
% $G = 32 \lambda_{\min}^{-3/2}\cdot \max_i \|\xb_{i}\|_2^3 \cdot \| \wb^{(0)} \|_2^{-2} $.
\end{lemma}
Lemma~\ref{lemma:norm_upperbound} demonstrates that $\|\wb^{(t)}\|_2$ is monotonically increasing during training, but the speed it increases is much slower than the speed $\la \wb^{(t)}, \wb^* \ra$ increases (when $\eta$ is small), as is shown in Lemma~\ref{lemma:innerproduct_lowerbound}. We note that Lemma~\ref{lemma:norm_upperbound} particularly gives a tighter inequality when $\wb^{(t)}$ is close enough to $\mathrm{span}\{ \wb^* \}$. This inequality shows that in the later stage of training when the margins tend uniform and the loss derivatives $|\ell_i'^{(t)}|$, $i\in [n]$ on the training data tend small, $\|\wb^{(t)}\|_2$ increases even slower. The importance of this result can be seen considering the case where all $|\ell_i'^{(t)}|$'s are equal (up to constant factors): in this case, combining the bounds in Lemmas~\ref{lemma:innerproduct_lowerbound}, \ref{lemma:key_identity} and \ref{lemma:norm_upperbound} can give an inequality of the form
\begin{align}\label{eq:convergence_illustration}
    \bigg\|  \wb^{(t+1)} -  \frac{\la \wb^*, \wb^{(t+1)}\ra}{\| \wb^* \|_2^2} \cdot \wb^* \bigg\|_2^2 \leq (1 - A^{(t)}) \cdot \bigg\|  \wb^{(t)} -  \frac{\la \wb^*, \wb^{(t)}\ra}{\| \wb^* \|_2^2} \cdot \wb^* \bigg\|_2^2
\end{align}
for some $ A^{(t)} > 0$ that depends on $|\ell_i'^{(t)}|$, $i\in [n]$. \eqref{eq:convergence_illustration} is clearly the key to show the monotonicity and convergence of the $\|  \wb -  \| \wb^* \|_2^{-2}\cdot \la \wb^*, \wb \ra \cdot \wb^* \|_2$, which eventually leads to a last-iterate bound of the margin discrepancy according to Lemma~\ref{lemma:key_identity}. However, the rigorous version of the inequality above needs to be proved within an induction, which we explain in the next subsection.

\subsection{Final Convergence Guarantee With Sharp Convergence Rate}
Lemmas~\ref{lemma:innerproduct_lowerbound}, \ref{lemma:key_identity} and \ref{lemma:norm_upperbound} give the key intermediate results in our proof. However, it is still technically challenging to show the convergence and give a sharp convergence rate. We remind the readers that during training, the loss function value at each data point $(\xb_i,y_i)$ converges to zero, and so is the absolute value of the loss derivative $|\ell_i'^{(t)}|$. Based on the bounds in Lemma~\ref{lemma:innerproduct_lowerbound} and the informal result in \eqref{eq:convergence_illustration}, we see that if $|\ell_i'^{(t)}|$, $i\in [n]$ vanish too fast, then the margin discrepancy may not have sufficient time to converge. Therefore, in order to show the convergence of margin discrepancy, we need to accurately characterize the orders of $\gamma^{(t)}$ and $|\ell_i'^{(t)}|$. 

Intuitively, characterizing the orders of $\gamma^{(t)}$ and $|\ell_i'^{(t)}|$ is easier when the margins are relatively uniform, because in this case $|\ell_i'^{(t)}|$, $i\in [n]$ are all roughly of the same order. Inspired by this, we implement a two-stage analysis, where the first stage provides a warm start for the second stage with relatively uniform margins. The following lemma presents the guarantees in the first stage.

\begin{lemma}\label{lemma:firststage}
%[Simplification of Lemma~\ref{lemma:firststage}]\label{lemma:firststage_simple}
Let $\epsilon =  (16\max_i \| \zb_i \|_2 )^{-1}\cdot \| \wb^* \|_2^{-1}\cdot \min\big\{1/3,  \| \wb^{(0)} \|_2^{-1} \cdot \lambda_{\min}^{1/2} / (40 \lambda_{\max}^{3/4}) \big\}$. 
% $\epsilon = \poly( \lambda_{\min}, \lambda_{\max}^{-1}, (\max_i \| \xb_i \|_2)^{-1} )$. Suppose that $\| \wb^{(0)} \|_2 \leq \poly( \lambda_{\min}, \lambda_{\max}^{-1}, (\max_i \| \xb_i \|_2)^{-1} )$, $\eta \leq $
% \begin{align*}
%     \epsilon =  (16\max_i \| \zb_i \|_2 )^{-1}\cdot \| \wb^* \|_2^{-1}\cdot \min\big\{1/3,  \| \wb^{(0)} \|_2^{-1} \cdot \lambda_{\min}^{1/2} / (40 \lambda_{\max}^{3/4}) \big\}.
% \end{align*}
There exist constants $c_1,c_2,c_3> 0$ such that if
\begin{align*}
    &\| \wb^{(0)} \|_2 \leq c_1\cdot \min\big\{1, (\max_i \| \zb_i \|_2)^{-1/2} \cdot  \| \wb^* \|_2^{-1}  \big\} \cdot \lambda_{\max}^{-3/4} \cdot \lambda_{\min}^{1/2},\\
    &\eta \leq c_2\cdot \min\big\{ 1, \epsilon\cdot \| \wb^{(0)} \|_2^2 \cdot \lambda_{\min} \cdot (\max_i \| \xb_i \|_2)^{-3/2}, \epsilon^4\cdot \lambda_{\min}^3 \cdot \lambda_{\max}^{-3/2}\cdot  \| \wb^* \|_2^{-1} \cdot (\max_i \| \xb_i \|_2)^{-3} \big\},
\end{align*}
then there exists $t_0 \leq c_3 \eta^{-1} \epsilon^{-2} \cdot \lambda_{\max}^{3/2} \cdot \lambda_{\min}^{-1} \cdot \| \wb^{(0)} \|_2^2 \cdot \| \wb^* \|_2$ such that 
\begin{enumerate}[leftmargin = *]
    \item $1/2 \leq \gamma^{(t)} \leq 3/2$ for all $0 \leq t \leq t_0$.
    \item $\|\wb^{(0)} \|_2 \leq | \wb^{(t)} \|_2 \leq (1 + \epsilon / 2 )  \cdot \|\wb^{(0)} \|_2$ for all $0 \leq t \leq t_0$.
    \item $\big\| \wb^{(t_0)} - \la \wb^*, \wb^{(t_0 )} \ra \cdot \| \wb^* \|_2^{-2}  \cdot \wb^* \big\|_2 \leq \epsilon \cdot \|\wb^{(0)} \|_2$.
\end{enumerate}
\end{lemma}
In Lemma~\ref{lemma:firststage}, we set up a target rate $\epsilon$ that only depends on the training data and the initialization $\wb^{(0)}$, so that $\wb^{(t_0)}$ satisfying the three conclusions Lemma~\ref{lemma:firststage} can serve as a good enough warm start for our analysis on the second stage starting from iteration $t_0$. 

The study of the convergence of margin discrepancy for $t \geq t_0$ is the most technical part of our proof. The results are summarized in the following lemma.

\begin{lemma}\label{lemma:linear_asymp}
Let $G$ be defined as in Lemma~\ref{lemma:norm_upperbound}, and $\epsilon$ be defined as in Lemma~\ref{lemma:firststage}. 
% \begin{align*}
%     \epsilon =  (16\max_i \| \zb_i \|_2 )^{-1}\cdot \| \wb^* \|_2^{-1}\cdot \min\big\{1/3,  \| \wb^{(0)} \|_2^{-1} \cdot \lambda_{\min}^{1/2} / (40 \lambda_{\max}^{3/4}) \big\}.
% \end{align*}
Suppose that there exists $t_0\in \NN_+$ such that $ 1/2 \leq \gamma^{(t_0)} \leq 3/2 $, and
\begin{align*}
     &\|\wb^{(0)} \|_2 \leq | \wb^{(t_0)} \|_2 \leq (1 + \epsilon / 2 )  \cdot \|\wb^{(0)} \|_2,~\big\| \wb^{(t_0)} - \la \wb^*, \wb^{(t_0 )} \ra \cdot \| \wb^* \|_2^{-2}  \cdot \wb^* \big\|_2 \leq \epsilon \cdot \|\wb^{(0)} \|_2.
\end{align*}
% \begin{align*}
%      &\|\wb^{(0)} \|_2 \leq | \wb^{(t_0)} \|_2 \leq (1 + \epsilon / 2 )  \cdot \|\wb^{(0)} \|_2,\\
%      &\big\| \wb^{(t_0)} - \la \wb^*, \wb^{(t_0 )} \ra \cdot \| \wb^* \|_2^{-2}  \cdot \wb^* \big\|_2 \leq \epsilon \cdot \|\wb^{(0)} \|_2.
% \end{align*}
% \begin{align*}
%      &\|\wb^{(0)} \|_2 \leq | \wb^{(t_0)} \|_2 \leq [1 + (1/2)\cdot \min\{1,  \| \wb^* \|_2^{-1}\}\cdot (24 \max_i \| \zb_i \|_2)^{-1} ]  \cdot \|\wb^{(0)} \|_2,\\
%      &\big\| \wb^{(t_0)} - \la \wb^*, \wb^{(t_0 )} \ra \cdot \| \wb^* \|_2^{-2}  \cdot \wb^* \big\|_2 \leq \min\{1,  \| \wb^* \|_2^{-1}\}\cdot (24 \max_i \| \zb_i \|_2)^{-1} \cdot \|\wb^{(0)} \|_2.
% \end{align*}
% $\big\| \wb^{(t_0)} - \la \wb^*, \wb^{(t_0 )} \ra \cdot \| \wb^* \|_2^{-2}  \cdot \wb^* \big\|_2 \leq \epsilon \cdot \|\wb^{(0)} \|_2$ 
% $ |\la \wb^{(t_0)} , \zb_{i}\ra - \alpha| \leq \epsilon$ 
% $\eta \leq \min\big\{ 1/8, \lambda_{\min}^{5/2}\exp(-7/4 ) \lambda_{\min}^{5/2} / ( 2048 \lambda_{\max}^{3/2} \cdot \max_i \|\xb_{i}\|_2^3  )  \big\}$
Then there exist constants $c_1,c_2,c_3,c_4$, such that as long as $\eta \leq c_1 G^{-1}\min\{ \epsilon^2, \lambda_{\min}\cdot \lambda_{\max}^{3/2} \cdot \| \wb^{(0)} \|_2^{-2} \}$, the following results hold for all $t \geq t_0$:
\begin{enumerate}[leftmargin = *]
    % \item\label{induction_1} $ \big\| \wb^{(\tau + t_0)} - \la \wb^*, \wb^{(\tau + t_0 )} \ra \cdot \| \wb^* \|_2^{-1}  \cdot \wb^* \big\|_2 \leq 2\epsilon$.
    \item $\big\| \wb^{(t_0)} - \la \wb^*, \wb^{(t_0 )} \ra \cdot \| \wb^* \|_2^{-2}  \cdot \wb^* \big\|_2 \geq \cdots \geq \big\| \wb^{(t)} - \la \wb^*, \wb^{(t )} \ra \cdot \| \wb^* \|_2^{-2}  \cdot \wb^* \big\|_2$. %is a decreasing sequence.
    % \begin{align*}
    %     \alpha^{-1}\cdot \log((\eta\alpha / \beta) \tau + A)\leq \gamma^{t_0 + \tau } \leq \alpha^{-1} \cdot \log(\eta\alpha\beta (\tau+1) + A),
    % \end{align*}
    % where $A = \exp(\alpha \gamma^{t_0})$.
    \item $ \|\wb^{(0)}\|_2\leq \| \wb^{(t)}\|_2 \leq (1 + \epsilon )\cdot \|\wb^{(0)}\|_2$.
    \item %$\gamma^{(t)}$ has the following upper and lower bounds: 
    $\log[ (\eta / 8) \cdot (t - t_0) +\exp(\gamma^{(t_0)}) ] \leq \gamma^{(t)} \leq  \log[ 8\eta\cdot (t - t_0) +2\exp(\gamma^{(t_0)}) ]$.
    % \begin{align*}
    %     &\log[ (\eta / 8) \cdot (t - t_0) +\exp(\gamma^{(t_0)}) ] \leq \gamma^{(t)} \leq  \log[ 8\eta\cdot (t - t_0) +2\exp(\gamma^{(t_0)}) ]. 
    %     % \qquad \gamma^{(t)} \geq \log[ (\eta / 8) \cdot (t - t_0) +\exp(\gamma^{(t_0)}) ].
    % \end{align*}
    % \begin{align*}
    %     &\gamma^{(t)} \leq  2\eta\cdot \exp(3\beta / \alpha) \cdot \exp(-\gamma^{(t_0)}) + \log[ 2\eta\cdot \exp(3\beta / \alpha) \cdot (t - t_0) +\exp(\gamma^{(t_0)}) ],\\
    %     &\gamma^{(t)} \geq \log[ (\eta / 4)\cdot \exp(-\beta / \alpha) \cdot (t - t_0) +\exp(\gamma^{(t_0)}) ].
    % \end{align*}
    % $ \log[ (\eta / 4)\cdot \exp(\beta / \alpha) \cdot (t - t_0) +\exp(\gamma^{(t_0)}) ]\leq \gamma^{(t)} \leq  2\eta\cdot \exp(3\beta / \alpha) \cdot \exp(-\gamma^{(t_0)}) + \log[ 2\eta\cdot \exp(3\beta / \alpha) \cdot (t - t_0) +\exp(\gamma^{(t_0)}) ]$. %,  where $A = \exp( \gamma^{t_0})$.
    \item $\Big\|  \wb^{(t)} -  \frac{\la \wb^*, \wb^{(t)}\ra}{\| \wb^* \|_2^2} \cdot \wb^* \Big\|_2^2 \leq \epsilon \cdot \|\wb^{(0)} \|_2 \cdot \exp\Big[ -\frac{ c_2\lambda_{\min} \cdot \log^2( (8/9) \eta\cdot (t - t_0)+ 1) }{\lambda_{\max}^{3/2} \cdot \| \wb^{(0)} \|_2^2} \Big]$.
        % &\big\| \wb^{(t)} - \la \wb^*, \wb^{(t )} \ra \cdot \| \wb^* \|_2^{-2}  \cdot \wb^* \big\|_2\\
        % &\qquad \qquad  \leq \epsilon\cdot \|\wb^{(0)} \|_2 \cdot \exp\Bigg[ -\frac{ \lambda_{\min}  }{1024\lambda_{\max}^{3/2} \cdot \| \wb^{(0)} \|_2^2} \cdot \log^2( (8/9) \eta\cdot (t - t_0)+ 1)\Bigg]
    % \item It holds that 
    % \begin{align*}
    %     % &\big\| \wb^{(t)} - \la \wb^*, \wb^{(t )} \ra \cdot \| \wb^* \|_2^{-2}  \cdot \wb^* \big\|_2\\
    %     % &\qquad \qquad  \leq \epsilon\cdot \|\wb^{(0)} \|_2 \cdot \exp\Bigg[ -\frac{ \lambda_{\min}  }{1024\lambda_{\max}^{3/2} \cdot \| \wb^{(0)} \|_2^2} \cdot \log^2( (8/9) \eta\cdot (t - t_0)+ 1)\Bigg].
    %     \bigg\|  \wb^{(t)} -  \frac{\la \wb^*, \wb^{(t)}\ra}{\| \wb^* \|_2^2} \cdot \wb^* \bigg\|_2^2 \leq \epsilon \cdot \|\wb^{(0)} \|_2 \cdot \exp\Bigg[ -\frac{ c_2\lambda_{\min} \cdot \log^2( (8/9) \eta\cdot (t - t_0)+ 1) }{\lambda_{\max}^{3/2} \cdot \| \wb^{(0)} \|_2^2} \Bigg].
    %     % &\frac{1}{n^2}\sum_{i,i'=1}^n (\la \wb^{(t)} , \zb_{i'}\ra - \la \wb^{(t)} , \zb_{i}\ra)^2 \\
    %     % &\qquad \qquad \leq 2\lambda_{\max}\cdot \epsilon^2\cdot \|\wb^{(0)} \|_2^2 \cdot \exp\Bigg[ -\frac{ \lambda_{\min}  }{512\lambda_{\max}^{3/2} \cdot \| \wb^{(0)} \|_2^2} \cdot \log^2( (8/9) \eta\cdot (t - t_0)+ 1)\Bigg].
    % \end{align*}
    \item $\max_{i} |\la \wb^{(t)}/\|  \wb^{(t)}\|_2 , \zb_{i}\ra - \| \wb^* \|_2^{-1} | \cdot \gamma^{(t)} \leq \| \wb^* \|_2^{-1} / 4$.
    \item It holds that
    \begin{align*}
        &\frac{1}{40}\cdot \frac{1}{\eta\cdot (t - t_0 ) + 1} \leq \ell( y_i\cdot f(\wb^{(t)},\gamma^{(t)},\xb_i ) )\leq  \frac{12}{\eta\cdot (t - t_0) + 1},\\
        &D(\wb) \leq \frac{c_3\lambda_{\max}}{ \lambda_{\min} } \cdot \exp\Bigg[ -\frac{ c_4 \lambda_{\min}  }{\lambda_{\max}^{3/2} \cdot \| \wb^{(0)} \|_2^2} \cdot \log^2( (8/9) \eta\cdot (t - t_0)+ 1)\Bigg].
    \end{align*}
\end{enumerate}
\end{lemma}
A few remarks about Lemma~\ref{lemma:linear_asymp} are in order: %The
%roles of these results are as follows:
%\begin{itemize}[leftmargin = *]
    %\item 
    The first and the second results guarantee that the properties of the warm start $\wb^{(t_0)}$ given in Lemma~\ref{lemma:firststage} are preserved during training. 
   The third result on $\gamma^{(t)}$ controls the values of $|\ell_i'|$, $i\in[n]$ given uniform enough margins, and also implies the convergence rate of the training loss function.
    The fourth result gives the convergence rate of $\big\| \wb^{(t)} - \la \wb^*, \wb^{(t )} \ra \cdot \| \wb^* \|_2^{-2}  \cdot \wb^* \big\|_2$, which is an equivalent metric of the margin discrepancy in $\ell_2$ distance. %The fast rate in this result, together with the increase rate of $\gamma^{(t)}$, implies the fifth result.
    The fifth result essentially follows by the convergence rates of $\gamma^{(t)}$ and $\big\| \wb^{(t)} - \la \wb^*, \wb^{(t )} \ra \cdot \| \wb^* \|_2^{-2}  \cdot \wb^* \big\|_2$, and it further implies that $|\ell_i'|$, $i\in[n]$ only differ by constant factors.
    Finally, the sixth result reformulates the previous results and gives the conclusions of Theorem~\ref{thm:linearBN}.

In our proof of Lemma~\ref{lemma:linear_asymp}, we first show the first five results based on an induction, where each of these results rely on the margin uniformity in the previous iterations. The last result is then proved based on the first five results. We also regard this proof as a key technical contribution of our work. Combining Lemmas~\ref{lemma:firststage} and \ref{lemma:linear_asymp} leads to  Theorem~\ref{thm:linearBN}, and the proof is thus complete.

\section{Conclusion and Future Work}
In this work we theoretically demonstrate that with batch normalization, gradient descent converges to a uniform margin solution when training linear models, and converges to a patch-wise uniform margin solution when training two-layer, single-filter linear CNNs. These results give a precise characterization of the implicit bias of batch normalization. 

An important future work direction is to study the implicit bias of batch normalization in more complicated neural network models with multiple filters, multiple layers, and nonlinear activation functions. Our analysis  may also find other applications in studying the implicit bias of other normalization techniques such as layer normalization.

% \newpage

% Denote $\cX_{n} = \mathrm{span}\{\xb_1,\ldots. \xb_n\}$, and let $\Pb_{\cX_{n}}$ be the projection matrix onto $\cX_{n}$.

% \newpage

\appendix

% \crefalias{section}{appendix} % uncomment if you are using cleveref

\section{Additional Related Work}

\paragraph{Implicit Bias.} 
In recent years, there have emerged a large number of works studying the implicit bias of various optimization algorithms for different models.  We will only review a part of them that is mostly related to this paper.

Theoretical analysis of implicit bias has originated from linear models. For linear regression, \citet{gunasekar2018implicit} showed that starting from the origin point, gradient descent converges to the minimum $\ell_2$ norm solution. For linear classification problems, gradient descent is shown to converge to the maximum margin solution on separable data \citep{soudry2017implicit,nacson2018convergence,ji2019implicit}. Similar results have also been obtained for other optimization algorithms such as mirror descent \citep{gunasekar2018characterizing} and stochastic gradient descent \citep{nacson2019stochastic}.  
The implicit bias has also been widely studied beyond linear models such as matrix factorization \citep{gunasekar2017implicit,li2021towards,arora2019implicit}, linear networks \citep{li2018algorithmic,ji2018gradient,gunasekar2018implicit,pesme2021implicit}, and more complicated nonlinear networks \citep{chizat2020implicit}, showing that (stochastic) gradient descent may converge to different types of solutions. However, none of them can be adapted to our setting.

\paragraph{Theory for Normalization Methods.} Many normalization methods, including batch normalization, weight normalization \citep{salimans2016weight}, layer normalization \citep{ba2016layer}, and group normalization \citep{wu2018group}, have been recently developed to improve generalization performance. From a theoretical perspective, a series of works have established the close connection between the normalization methods and the adaptive effective learning rate \citep{hoffer2018norm,arora2019theoretical,morwani2020inductive}. Based on the auto-learning rate adjustment effect of weight normalization, \citet{wu2018wngrad} developed an adaptive learning rate scheduler and proved a near-optimal convergence rate for SGD. Moreover, \citet{wu2020implicit}  studied the implicit regularization effect of weight normalization for linear regression, showing that gradient descent can always converge close to the minimum $\ell_2$ norm solution, even from the initialization that is far from zero. \citet{dukler2020optimization} further investigated the convergence of gradient descent for training a two-layer ReLU network with weight normalization. For multi-layer networks, \citet{yang2019mean} developed a mean-field theory for batch normalization, showing that the gradient explosion brought by large network depth cannot be addressed by BN, if there is no residual connection.
These works either concerned different normalization methods or investigated the behavior of BN in different aspects, which can be seen as orthogonal to our results.

\section{Experiments}
Here we present some preliminary simulation and real data experiment results to demonstrate that batch normalization encourages a uniform margin. The results are given in Figure~\ref{fig1} and Figure~\ref{fig2} respectively.

In Figure~\ref{fig1}, we train linear models with/without batch normalization, and plot the values of $y_i\cdot \la \wb^{(t)}, \xb_i \ra$, $i=1,\ldots,n$ in Figure~\ref{fig1}. From Figure~\ref{fig1}, we can conclude that:
\begin{enumerate}[nosep,leftmargin = *]
    \item The obtained linear model with batch normalization indeed achieves a uniform margin.
    \item Batch normalization plays a key role in determining the implicit regularization effect. Without batch normalization, the obtained linear model does not achieve a uniform margin.
\end{enumerate}
Clearly, our simulation results match our theory well.

In Figure~\ref{fig2}, we present experiment results for VGG-16 with/without batch normalization trained by stochastic gradient descent to classify cat and dog images in the CIFAR-10 data set. We focus on the last hidden layer of VGG-16, and for each neuron on this layer, we estimate the margin uniformity over (i) all activated data from class cat; (ii) all activated data from class dog; (iii) all activated data from both classes. The final result is then calculated by taking an average over all neurons. Note that in these experiments, the neural network model, the data and the training algorithm all mismatch the exact setting in Theorem~\ref{thm:CNNBN}. Nevertheless, the experiment results still corroborate our theory to a certain extent.
% calculate $\sum_{i,i' \in \cI} \big(y_{i'}\cdot \frac{\la \wb_1^{(t)} , \xb_{i'}\ra }{\| \wb_1^{(t)} \|_{\bSigma}} - y_{i}\cdot \frac{ \la \wb_1^{(t)} , \xb_{i} \ra }{\| \wb_1^{(t)} \|_{\bSigma}} \big)^2$, where $\bSigma$ is calculated based on the mini-batch of data, and index set
% $$\cI = \{i: \text{this neuron is activated when the network takes the }i\text{-th data in mini-batch as input}\}.$$
% We also compare the result with VGG-16 without batch normalization, where we 
% The results are given in Figure~\ref{fig2}.

\begin{figure}[h!]
% \vskip -0.15in
     \centering
     \subfigure[Linear model with batch normalization]{\includegraphics[width=0.49\textwidth]{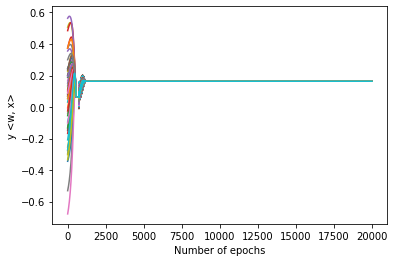}}
      \subfigure[Linear model without batch normalization]{\includegraphics[width=0.49\textwidth]{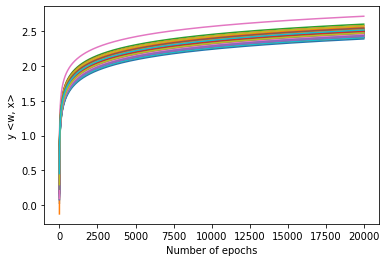}}
      \caption{The values of $y_i\cdot \la \wb^{(t)} , \xb_i \ra$, $i=1,\ldots,n$ during the training of linear models with and without batch normalization. Each curve in the figures corresponds to a specific $i\in\{1,\ldots,n\}$ and illustrates the dynamics of a specific inner product $y_i\cdot \la \wb^{(t)} , \xb_i \ra$. (a) gives the result for the linear model with batch normalization; (b) gives the result for the linear model without batch normalization. For both settings, we set the sample size $n = 50$ and dimension $d = 1000$. For $i=1,\ldots, 50$, we generate the data inputs $\xb_i$ independently as standard Gaussian random vectors, and set $y_i$ randomly as $+1$ or $-1$ with equal probability. }%In gradient descent, we set the learning rate $\eta = 0.001$, and initialize the linear model weights as Gaussian random variables with zero mean and standard deviation $0.01$. In batch normalization, we set the initial scaling factor $\gamma^{(0)}=0.001$.}
    % \caption{The values of $y_i\cdot \la \wb^{(t)} , \xb_i \ra$, $i=1,\ldots,n$ during the  training of linear models with/without batch normalization. (a) gives the result for a linear model with batch normalization; (b) gives the result for a linear model without batch normalization.}
    \label{fig1}
\end{figure}

\begin{figure}[h!]
% \vskip -0.15in
     \centering
     \subfigure[Dog]{\includegraphics[width=0.32\textwidth]{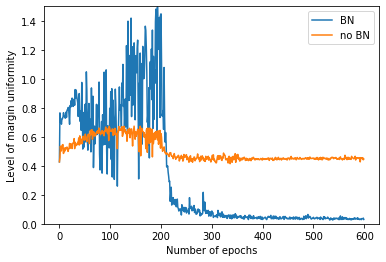}}
      \subfigure[Cat]{\includegraphics[width=0.32\textwidth]{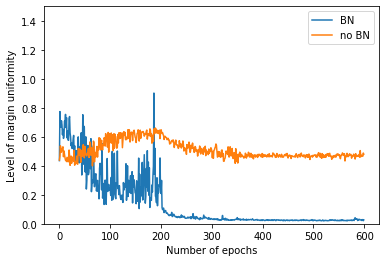}}
      \subfigure[Both]{\includegraphics[width=0.32\textwidth]{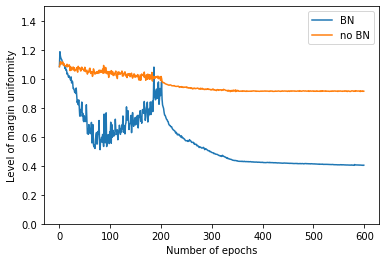}}
      \caption{Margin uniformity of the last hidden layer in VGG-16 with/without batch normalization. The model is trained to classify cat and dog images in the CIFAR-10 data set. For each neuron on last hidden layer, treating it as a function $\mathrm{neuron(\xb)}$ of the data input $\xb$, we use $\frac{ \frac{1}{n^2} \sum_{i,i' \in \cI} [y_{i'}\cdot \mathrm{neuron(\xb_{i'})} -  y_{i}\cdot \mathrm{neuron(\xb_i)} ]^2}{ \frac{1}{n} \sum_{i \in \cI} [ \mathrm{neuron(\xb_i)} ]^2 }$ to measure the margin uniformity, where we set $\cI$ as the index set of (a) all activated data in the class cat, (b) all activated data  in the class dog, and (c) all activated data in both classes. The plots are for the average of this quantity over all hidden neurons. Clearly, a smaller value implies a higher level of margin uniformity.}
    \label{fig2}
\end{figure}

\section{Proofs for Batch Normalization in Linear Models}
In this section we present the proofs of the lemmas in Section~\ref{section:technique_overview}. Combining these proofs with the discussion given in Section~\ref{section:technique_overview} would give the complete proof of Theorem~\ref{thm:linearBN}.

\subsection{Proof of Lemma~\ref{lemma:gradient_inner_linear_model}}

\begin{proof}[Proof of Lemma~\ref{lemma:gradient_inner_linear_model}] By definition, we have 
\begin{align*}
    \nabla_{\wb} L(\wb,\gamma) = \frac{1}{n}\sum_{i=1}^n \ell'[y_i\cdot f(\wb,\gamma,\xb_i)] \cdot y_i\cdot \nabla_{\wb} f(\wb,\gamma,\xb_i).
\end{align*}
Then by the definition of the linear predictor with batch normalization, we have the following calculation using chain rule: 
\begin{align*}
    \nabla_{\wb} L(\wb,\gamma) &=  \| \wb \|_{\bSigma}^{-1}\cdot \frac{1}{n}\sum_{i=1}^n  \ell'[y_i\cdot f(\wb,\gamma,\xb_i)] \cdot y_i \cdot \gamma \cdot \Big( \Ib - \| \wb \|_{\bSigma}^{-2}\cdot \bSigma \wb \wb^\top \Big)\xb_i \\
    & = \| \wb \|_{\bSigma}^{-3}\cdot \frac{\gamma}{n}\sum_{i=1}^n \ell'_i \cdot y_i \cdot \Big( \| \wb \|_{\bSigma}^{2} \cdot \Ib -  \bSigma \wb \wb^\top \Big)\xb_i\nonumber \\
     & = \| \wb \|_{\bSigma}^{-3}\cdot \frac{\gamma}{n}\sum_{i=1}^n \ell'_i \cdot \Bigg( \frac{1}{n} \sum_{i'=1}^n \la \wb, \zb_{i'}\ra^2\cdot \Ib - \frac{1}{n}\sum_{i'=1}^n \la \wb , \zb_{i'}\ra \cdot \zb_{i'} \wb^\top \Bigg)\zb_i\\
     & = \| \wb \|_{\bSigma}^{-3}\cdot \frac{\gamma}{n^2}\sum_{i=1}^n \sum_{i'=1}^n  \ell'_i \cdot \la \wb, \zb_{i'}\ra^2\cdot \zb_i - \| \wb \|_{\bSigma}^{-3}\cdot \frac{\gamma}{n^2}\sum_{i=1}^n \sum_{i'=1}^n \ell'_i \cdot \la \wb , \zb_{i'}\ra \cdot  \la\wb,\zb_i\ra \cdot \zb_{i'},
\end{align*}
where we remind readers that $\zb_i = y_i\cdot \xb_i$, $i\in[n]$. 
By Assumption~\ref{assump:uniformly_separable}, taking inner product with $\wb^*$ on both sides above then gives
\begin{align*}
    \la \nabla_{\wb} L(\wb,\gamma), \wb^* \ra = \| \wb \|_{\bSigma}^{-3}\cdot \frac{\gamma}{n^2}\sum_{i=1}^n \sum_{i'=1}^n  \ell'_i \cdot \la \wb, \zb_{i'}\ra^2 - \| \wb \|_{\bSigma}^{-3}\cdot \frac{\gamma}{n^2}\sum_{i=1}^n \sum_{i'=1}^n \ell'_i \cdot \la \wb , \zb_{i'}\ra \cdot  \la\wb,\zb_i\ra. 
\end{align*}
Further denote $u_i = \la \wb, \zb_{i}\ra$ for $i\in[n]$. Then we have 
\begin{align}\label{eq:proof_gradient_inner_eq1}
    \la \nabla_{\wb} L(\wb,\gamma), \wb^* \ra = \| \wb \|_{\bSigma}^{-3}\cdot \frac{\gamma}{n^2}\sum_{i=1}^n \sum_{i'=1}^n  \ell'_i \cdot u_{i'}^2 - \| \wb \|_{\bSigma}^{-3}\cdot \frac{\gamma}{n^2}\sum_{i=1}^n \sum_{i'=1}^n \ell'_i \cdot u_{i'}u_i.
\end{align}
Switching the index notations $i,i'$ in the above equation also gives
\begin{align}\label{eq:proof_gradient_inner_eq2}
    \la \nabla_{\wb} L(\wb,\gamma), \wb^* \ra = \| \wb \|_{\bSigma}^{-3}\cdot \frac{\gamma}{n^2}\sum_{i=1}^n \sum_{i'=1}^n  \ell'_{i'} \cdot u_{i}^2 - \| \wb \|_{\bSigma}^{-3}\cdot \frac{\gamma}{n^2}\sum_{i=1}^n \sum_{i'=1}^n \ell'_{i'} \cdot u_{i'}u_i.
\end{align}
We can add \eqref{eq:proof_gradient_inner_eq1} and \eqref{eq:proof_gradient_inner_eq2} together to obtain
\begin{align*}
    2\cdot\la \nabla_{\wb} L(\wb,\gamma), \wb^* \ra &= \| \wb \|_{\bSigma}^{-3}\cdot \frac{\gamma}{n^2}\sum_{i=1}^n \sum_{i'=1}^n ( \ell'_i \cdot u_{i'}^2 - \ell'_i \cdot u_{i'}u_i + \ell'_{i'} \cdot u_{i}^2 - \ell'_{i'} \cdot u_{i'}u_i )\\
     &= \| \wb \|_{\bSigma}^{-3}\cdot \frac{\gamma}{n^2}\sum_{i=1}^n \sum_{i'=1}^n (u_{i'} - u_i)( \ell'_i \cdot u_{i'} - \ell'_{i'} \cdot u_{i} ).
\end{align*}
Note that by definition we have $\ell'_i< 0$. Therefore, 
\begin{align*}
    - \la \nabla_{\wb} L(\wb,\gamma), \wb^* \ra = \| \wb \|_{\bSigma}^{-3}\cdot \frac{\gamma}{2 n^2}\sum_{i=1}^n \sum_{i'=1}^n |\ell'_i|\cdot |\ell'_{i'}| \cdot (u_{i'} - u_i)( |\ell'_{i'}|^{-1} \cdot u_{i'} - |\ell'_{i}|^{-1} \cdot u_{i} ).
\end{align*}
This completes the proof.
\end{proof}

\subsection{Proof of Lemma~\ref{lemma:innerproduct_lowerbound}}

\begin{proof}[Proof of Lemma~\ref{lemma:innerproduct_lowerbound}]
By Lemma~\ref{lemma:gradient_inner_linear_model} and the gradient descent update rule, we have
\begin{align}
    &\la \wb^{(t+1)}, \wb^* \ra -\la \wb^{(t)}, \wb^* \ra = \eta\cdot \la -\nabla_{\wb} L(\wb^{(t)},\gamma^{(t)}), \wb^* \ra \nonumber\\
    & = \frac{ \gamma^{(t)} \eta }{2 n^2 \| \wb^{(t)} \|_{\bSigma}^3}\sum_{i,i'=1}^n |\ell_i'^{(t)}|\cdot |\ell_{i'}'^{(t)}| \cdot (\la \wb^{(t)} , \zb_{i'}\ra - \la \wb^{(t)} , \zb_{i}\ra)\cdot ( |\ell_{i'}'^{(t)}|^{-1} \cdot \la \wb^{(t)}, \zb_{i'}\ra - |\ell_{i}'^{(t)}|^{-1} \cdot \la \wb^{(t)} , \zb_{i}\ra )\label{eq:originalequation}
\end{align}
for all $t\geq 0$. Note that $\ell(z) = \log[1 + \exp(-z)]$, $\ell'(z) = -1/[1 + \exp(z)]$ and
$$|\ell_i'^{(t)}|^{-1} = -\{ \ell'[y_i\cdot f(\wb^{(t)},\gamma^{(t)},\xb_i)] \}^{-1} = 1 + \exp( \gamma^{(t)}\cdot \la \wb^{(t)} , \zb_{i}\ra /  \| \wb^{(t)} \|_{\bSigma}).
$$
Denote $F_i = \gamma^{(t)}\cdot \la \wb^{(t)} , \zb_{i}\ra /  \| \wb^{(t)} \|_{\bSigma}$ for $i\in [n]$. Then we have
\begin{align}
    \frac{( |\ell_{i'}'^{(t)}|^{-1} \cdot \la \wb^{(t)}, \zb_{i'}\ra - |\ell_{i}'^{(t)}|^{-1} \cdot \la \wb^{(t)} , \zb_{i}\ra )}{ \la \wb^{(t)}, \zb_{i'}\ra - \la \wb^{(t)} , \zb_{i}\ra}& =\frac{( |\ell_{i'}'^{(t)}|^{-1} \cdot \la \wb^{(t)}, \zb_{i'}\ra - |\ell_{i}'^{(t)}|^{-1} \cdot \la \wb^{(t)} , \zb_{i}\ra ) \cdot \gamma^{(t)}/  \| \wb^{(t)} \|_{\bSigma} }{ (\la \wb^{(t)}, \zb_{i'}\ra - \la \wb^{(t)} , \zb_{i}\ra) \cdot \gamma^{(t)}/  \| \wb^{(t)} \|_{\bSigma}} \nonumber\\
    &= \frac{F_{i'}\cdot [ 1 + \exp(F_{i'}) ] -  F_{i}\cdot [ 1 + \exp(F_{i}) ]}{ F_{i'} - F_{i} }  \nonumber\\
    & = 1 + \frac{F_{i'}\cdot\exp(F_{i'}) - F_{i} \cdot\exp(F_{i})}{F_{i'} - F_{i}} \nonumber\\
    &\geq 1 + \exp(\max\{ F_{i'}, F_{i} \}) \nonumber\\
    & = \max\{ |\ell_{i}'^{(t)}|^{-1}, |\ell_{i'}'^{(t)}|^{-1} \},\label{eq:meanvalue}
\end{align}
where the inequality follows by the fact that $[a\exp(a) - b\exp(b)]/(a-b) \geq  \exp(\max\{a,b\}) $ for all $a\neq b$. 
% Define $g(z) = |\ell'(z)|^{-1} \cdot z = z + z\cdot \exp(z)$. Then we have
% \begin{align}\label{eq:g_derivativecalc}
%     g'(z) = 1 + \exp(z) + z\cdot \exp(z) \geq [ 1 + \exp(z)] / 2 = |\ell'(z)|^{-1} / 2. 
% \end{align}
% Therefore by the mean value theorem, for any $i,i'\in[n]$, there exists $\xi_{ii'}$ between $\gamma^{(t)}\cdot \la \wb^{(t)} , \zb_{i'}\ra /  \| \wb^{(t)} \|_{\bSigma}$ and $ \gamma^{(t)}\cdot \la \wb^{(t)} , \zb_{i}\ra /  \| \wb^{(t)} \|_{\bSigma} $, such that
% \begin{align}
%     \frac{ \gamma^{(t)} }{\| \wb^{(t)} \|_{\bSigma}}\cdot ( |\ell'_{i'}|^{-1} \cdot \la \wb^{(t)}, \zb_{i'}\ra - |\ell'_{i}|^{-1} \cdot \la \wb^{(t)} , \zb_{i}\ra ) &= g\bigg( \frac{\gamma^{(t)}\cdot \la \wb^{(t)} , \zb_{i'}\ra }{  \| \wb^{(t)} \|_{\bSigma}}\bigg) - g\bigg( \frac{\gamma^{(t)}\cdot \la \wb^{(t)} , \zb_{i}\ra }{  \| \wb^{(t)} \|_{\bSigma}}\bigg) \nonumber\\
%     & = g'(\xi_{ii'})\cdot \bigg( \frac{\gamma^{(t)}\cdot \la \wb^{(t)} , \zb_{i'}\ra }{  \| \wb^{(t)} \|_{\bSigma}} - \frac{\gamma^{(t)}\cdot \la \wb^{(t)} , \zb_{i}\ra }{  \| \wb^{(t)} \|_{\bSigma}}\bigg) \label{eq:meanvalue}
% \end{align}
Plugging \eqref{eq:meanvalue} into \eqref{eq:originalequation} gives
\begin{align}
    &\la \wb^{(t+1)}, \wb^* \ra -\la \wb^{(t)}, \wb^* \ra \nonumber\\ 
    &\qquad\qquad \geq \frac{ \gamma^{(t)} \eta }{2 n^2 \| \wb^{(t)} \|_{\bSigma}^3}\sum_{i,i'=1}^n |\ell_i'^{(t)}|\cdot |\ell_{i'}'^{(t)}| \cdot (\la \wb^{(t)} , \zb_{i'}\ra - \la \wb^{(t)} , \zb_{i}\ra)^2\cdot  \max\{ |\ell_{i}'^{(t)}|^{-1}, |\ell_{i'}'^{(t)}|^{-1} \} \nonumber\\
    &\qquad\qquad  =\frac{ \gamma^{(t)} \eta }{2 n^2 \| \wb^{(t)} \|_{\bSigma}^3}\sum_{i,i'=1}^n \max\{ |\ell_i'^{(t)}|, |\ell_{i'}'^{(t)}|\} \cdot (\la \wb^{(t)} , \zb_{i'}\ra - \la \wb^{(t)} , \zb_{i}\ra)^2.\label{eq:lowerbound_natual}
\end{align}
Now by Lemma~\ref{lemma:auxiliary_inequality}, we have
\begin{align}
    &\la \wb^{(t+1)}, \wb^* \ra -\la \wb^{(t)}, \wb^* \ra\nonumber\\ 
    &\qquad\qquad  \geq \frac{ \gamma^{(t)} \eta }{8 n^2 \| \wb^{(t)} \|_{\bSigma}^3}\cdot \Bigg(\frac{1}{n} \sum_{i=1}^n |\ell_i'^{(t)}| \Bigg) \cdot \sum_{i,i'=1}^n  (\la \wb^{(t)} , \zb_{i'}\ra - \la \wb^{(t)} , \zb_{i}\ra)^2.\label{eq:sharper_lower_boudnd_proof_1}%\nonumber\\ 
    % &\qquad\qquad = \frac{ \gamma^{(t)} \eta }{8 \| \wb^{(t)} \|_{\bSigma}^3}\cdot \Bigg(\frac{1}{n} \sum_{i=1}^n |\ell_i'^{(t)}| \Bigg) \cdot  \|  \wb - \la \wb^*, \wb\ra_{\bSigma} \cdot \wb^* \|_{\bSigma}^2.\label{eq:sharper_lower_boudnd_proof_1}
\end{align}
Moreover, by definition we have
\begin{align*}
    \frac{1}{n} \sum_{i=1}^n |\ell_i'^{(t)}| & = \frac{1}{n} \sum_{i=1}^n \Bigg[1 + \exp\Bigg( \frac{\gamma^{(t)} \cdot \la \wb^{(t)}, \zb_j \ra }{\sqrt{\frac{1}{n} \sum_{j=1}^n \la \wb^{(t)}, \zb_j \ra^2 }} \Bigg) \Bigg]^{-1}\\
     & \geq  \frac{1}{n} \sum_{i=1}^n \Bigg[1 + \exp\Bigg( \frac{\gamma^{(t)} \cdot |\la \wb^{(t)}, \zb_j \ra| }{\sqrt{\frac{1}{n} \sum_{j=1}^n \la \wb^{(t)}, \zb_j \ra^2 }} \Bigg) \Bigg]^{-1},
\end{align*}
where the inequality follows by the fact that $[1 + \exp(z)]^{-1}$ is a decreasing function. Further note that $[1 + \exp(z)]^{-1}$ is convex over $z \geq 0$. Therefore by Jensen's inequality, we have
\begin{align}
    \frac{1}{n} \sum_{i=1}^n |\ell_i'^{(t)}|
     & \geq  \frac{1}{n} \sum_{i=1}^n \Bigg[1 + \exp\Bigg( \frac{\gamma^{(t)} \cdot |\la \wb^{(t)}, \zb_j \ra| }{\sqrt{\frac{1}{n} \sum_{j=1}^n \la \wb^{(t)}, \zb_j \ra^2 }} \Bigg) \Bigg]^{-1} \nonumber\\
     & \geq \Bigg[1 + \exp\Bigg( \frac{1}{n} \sum_{i=1}^n \frac{\gamma^{(t)} \cdot |\la \wb^{(t)}, \zb_j \ra| }{\sqrt{\frac{1}{n} \sum_{j=1}^n \la \wb^{(t)}, \zb_j \ra^2 }} \Bigg) \Bigg]^{-1} \nonumber\\
     & \geq [1 + \exp(\gamma^{(t)} )]^{-1} \nonumber\\
     & \geq \exp(-\gamma^{(t)} ) / 2.\label{eq:apply_jensen}
\end{align}
Plugging \eqref{eq:apply_jensen} into \eqref{eq:sharper_lower_boudnd_proof_1} then gives
\begin{align}
    \la \wb^{(t+1)}, \wb^* \ra - \la \wb^{(t)}, \wb^* \ra  &\geq \frac{ \gamma^{(t)} \eta }{16 n^2 \| \wb^{(t)} \|_{\bSigma}^3}\cdot \exp(-\gamma^{(t)} ) \cdot  \sum_{i,i'=1}^n  (\la \wb^{(t)} , \zb_{i'}\ra - \la \wb^{(t)} , \zb_{i}\ra)^2\nonumber\\
    &= \frac{ \gamma^{(t)} \eta }{16 \| \wb^{(t)} \|_{\bSigma}^3}\cdot \exp(-\gamma^{(t)} ) \cdot D(\wb^{(t)})\label{eq:sharper_lower_boudnd_first}
\end{align}
Moreover, we can simply utilize \eqref{eq:lowerbound_natual} %and the fact that $|\ell_{i_0}'^{(t)}| \geq \min_i |\ell_{i}'^{(t)}|$ 
to obtain 
\begin{align}
    \la \wb^{(t+1)}, \wb^* \ra -\la \wb^{(t)}, \wb^* \ra 
    &\geq \frac{ \gamma^{(t)} \eta}{2n^2 \| \wb^{(t)} \|_{\bSigma}^3}  \cdot \min_i |\ell_i'^{(t)}| \cdot  \sum_{i,i'=1}^n  (\la \wb^{(t)} , \zb_{i'}\ra - \la \wb^{(t)} , \zb_{i}\ra)^2\\
    &= \frac{ \gamma^{(t)} \eta}{2 \| \wb^{(t)} \|_{\bSigma}^3}  \cdot \min_i |\ell_i'^{(t)}| \cdot  D( \wb^{(t)})\label{eq:sharper_lower_boudnd_second}
\end{align}
Combining \eqref{eq:sharper_lower_boudnd_first} and \eqref{eq:sharper_lower_boudnd_second} finishes the proof. 
\end{proof}

\subsection{Proof of Lemma~\ref{lemma:key_identity}}

\begin{proof}[Proof of Lemma~\ref{lemma:key_identity}]
We note that the following identity holds:
\begin{align*}
    &\frac{1}{n^2 }\sum_{i,i'=1}^n (\la \wb , \zb_{i'}\ra - \la \wb , \zb_{i}\ra)^2\\
    &\qquad= \frac{1}{n^2 } \sum_{i,i'=1}^n (\la \wb , \zb_{i'}\ra^2 - 2\cdot \la \wb , \zb_{i'}\ra\cdot  \la \wb , \zb_{i}\ra + \la \wb , \zb_{i}\ra^2) \\
    &\qquad= \frac{2}{n } \sum_{i}^n \la \wb , \zb_{i}\ra^2 -2\cdot \Bigg(\frac{1}{n } \sum_{i}^n \la \wb , \zb_{i}\ra \Bigg)^2 \\
    &\qquad= \frac{2}{n } \sum_{i}^n \la \wb , \zb_{i}\ra^2 - 4\cdot \Bigg(\frac{1}{n } \sum_{i}^n \la \wb , \zb_{i}\ra \Bigg)^2 + 2\cdot \Bigg(\frac{1}{n } \sum_{i}^n \la \wb , \zb_{i}\ra \Bigg)^2 \\
    &\qquad= \frac{2}{n } \sum_{i}^n \la \wb , \zb_{i}\ra^2 - \frac{4}{n } \sum_{i}^n\la \wb , \zb_{i}\ra\cdot \Bigg(\frac{1}{n } \sum_{i}^n \la \wb , \zb_{i}\ra \Bigg) + \frac{2}{n } \sum_{i}^n\Bigg(\frac{1}{n } \sum_{i}^n \la \wb , \zb_{i}\ra \Bigg)^2 \\
    &\qquad=  \frac{2}{n }\sum_{i}^n \bigg(\la \wb , \zb_{i}\ra - \frac{1}{n }\sum_{i'}^n \la \wb , \zb_{i'}\ra \bigg)^2\\
    &\qquad = \frac{2}{n } \bigg\| \wb^{\top} \Zb \bigg( \Ib - \frac{1}{n }\mathbf{1}\mathbf{1}^\top \bigg) \bigg\|_2^2 \\
    &\qquad = \frac{2}{n } \wb^{\top} \Zb \bigg( \Ib - \frac{1}{n }\mathbf{1}\mathbf{1}^\top \bigg) \Zb^\top \wb.
\end{align*}
It is easy to see that the null space of $\Ib - \mathbf{1}\mathbf{1}^\top / n$ is $\mathrm{span}\{\mathbf{1}\}$, and the non-zero eigenvalues of  $\Ib - \mathbf{1}\mathbf{1}^\top / n$ are all $1$'s. Moreover, we note that the projection of the vector $\Zb^\top \wb$ onto the space $\mathrm{span}\{\mathbf{1}\}^\perp$ is
\begin{align*}
\Zb^\top \wb -  \mathbf{1} \mathbf{1}^\top  \Zb^\top \wb / n 
=  \Zb^\top \wb -  \Zb\wb^* \wb^{*\top} \Zb^\top  \Zb^\top \wb / n
= \Zb^\top \wb -  \Zb\wb^* \cdot \la \wb^{*\top}, \wb\ra_{\bSigma},
\end{align*}
where we utilize the property that $\Zb\wb^* = \mathbf{1}$ to obtain the first equality. Therefore, we have
\begin{align*}
    \frac{1}{n^2 }\sum_{i,i'=1}^n (\la \wb , \zb_{i'}\ra - \la \wb , \zb_{i}\ra)^2 
    & = \frac{2}{n } ( \wb -  \la \wb^*, \wb\ra_{\bSigma} \cdot \wb^* )^\top \Zb \Zb^\top ( \wb -  \la \wb^*, \wb\ra_{\bSigma} \cdot \wb^* )\nonumber\\
    & = 2\cdot \|  \wb -  \la \wb^*, \wb\ra_{\bSigma} \cdot \wb^* \|_{\bSigma}^2.%\label{eq:innerproduct_lowerbound_eq2}
\end{align*}
This finishes the proof of the first result. For the second result, we first note that by definition, $\| \wb^* \|_{\bSigma} = 1$, and therefore for any $\wb\in \RR^d$, $ \la \wb^*, \wb\ra_{\bSigma} \cdot \wb^*$ is the projection of $\wb$ onto $\mathrm{span}\{ \wb^* \}$ under the inner product $\la\cdot , \cdot\ra_{\bSigma}$. Therefore we have
\begin{align*}
     \|  \wb -  \la \wb^*, \wb\ra_{\bSigma} \cdot \wb^* \|_{\bSigma}^2 \leq \|  \wb -  c \cdot \wb^* \|_{\bSigma}^2
\end{align*}
for all $c\in \RR$, and hence
\begin{align*}
     \|  \wb -  \la \wb^*, \wb\ra_{\bSigma} \cdot \wb^* \|_{\bSigma}^2 \leq \bigg\|  \wb -  \frac{\la \wb^*, \wb\ra}{ \| \wb^* \|_2^2 } \cdot \wb^* \bigg\|_{\bSigma}^2 \leq \lambda_{\max}\cdot \bigg\|  \wb -  \frac{\la \wb^*, \wb\ra}{ \| \wb^* \|_2^2 } \cdot \wb^* \bigg\|_2^2.
\end{align*}
Similarly, we note that $\| \wb^* \|_2^{-2 } \cdot \la \wb^*, \wb\ra  \cdot \wb^* $ is the projection of $\wb$ onto $\mathrm{span}\{ \wb^* \}$ under the Euclidean inner product $\la\cdot , \cdot\ra$. Therefore we have
\begin{align*}
    \lambda_{\min}\cdot \bigg\|  \wb -  \frac{\la \wb^*, \wb\ra}{ \| \wb^* \|_2^2 } \cdot \wb^* \bigg\|_2^2 \leq  \lambda_{\min}\cdot \|  \wb -  \la \wb^*, \wb\ra_{\bSigma} \cdot \wb^* \|_2^2 \leq \|  \wb -  \la \wb^*, \wb\ra_{\bSigma} \cdot \wb^* \|_{\bSigma}^2.
\end{align*}
This completes the proof.
\end{proof}

\subsection{Proof of Lemma~\ref{lemma:norm_upperbound}}
We first present the following technical lemma. 
\begin{lemma}\label{lemma:auxiliary_inequality}
Let $\{a_i\}_{i=1,\dots,n}$ and $\{b_i\}_{i=1,\dots,n}$ be two sequences that satisfy
\begin{align*}
a_1\le a_2\le\dots\le a_n;\quad b_1\ge b_2\ge\dots\ge b_n \ge 0.
\end{align*}
Then it holds that
\begin{align*}
\sum_{i,i'=1}^n \max\{b_i,b_{i'}\}\cdot (a_i - a_{i'})^2 \ge \frac{\sum_{i=1}^n b_i}{4n}\cdot\sum_{i,i'=1}^n  (a_i - a_{i'})^2.
\end{align*}
 \end{lemma}
Based on Lemma~\ref{lemma:auxiliary_inequality}, the proof of Lemma~\ref{lemma:norm_upperbound} is as follows. 
\begin{proof}[Proof of Lemma~\ref{lemma:norm_upperbound}]
By the gradient descent update rule, we have
\begin{align}\label{eq:gradientcalculation_1}
    \wb^{(t+1)} &= \wb^{(t)} - \eta\cdot \nabla_{\wb} L(\wb^{(t)},\gamma^{(t)}).
    % \\
    % &= \wb^{(t)} -
    % \frac{\eta}{n \cdot \| \wb \|_{\bSigma}}\sum_{i=1}^n  \ell'[y_i\cdot f(\wb^{(t)},\gamma^{(t)},\xb_i)] \cdot y_i \cdot \gamma^{(t)} \cdot \Big( \Ib - \| \wb^{(t)} \|_{\bSigma}^{-2}\cdot \bSigma \wb^{(t)} \wb^{(t)\top} \Big)\xb_i.
\end{align}
Note that we have the calculation 
\begin{align*}
    \nabla_{\wb} L(\wb^{(t)},\gamma^{(t)}) = \frac{1}{n \cdot \| \wb \|_{\bSigma}}\sum_{i=1}^n  \ell'[y_i\cdot f(\wb^{(t)},\gamma^{(t)},\xb_i)] \cdot y_i \cdot \gamma^{(t)} \cdot \bigg( \Ib - \frac{\bSigma \wb^{(t)} \wb^{(t)\top}}{\| \wb^{(t)} \|_{\bSigma}^2} \bigg)\xb_i.
\end{align*}
It is easy to see that $\wb^{(t)}$ is orthogonal to $\nabla_{\wb} L(\wb^{(t)},\gamma^{(t)})$. Therefore, taking $\| \cdot \|_2^2$ on both sides of \eqref{eq:gradientcalculation_1} gives
\begin{align}\label{eq:gradientcalculation_2}
    \|\wb^{(t+1)}\|_2^2 = \|\wb^{(t)}\|_2^2 + \eta^2 \cdot \| \nabla_{\wb} L(\wb^{(t)},\gamma^{(t)}) \|_2^2. 
\end{align}
Therefore we directly conclude that $\| \wb^{(t)} \|_2^2 \leq \| \wb^{(t+1)} \|_2^2$. Besides, plugging in the calculation of $\nabla_{\wb} L(\wb^{(t)},\gamma^{(t)})$ also gives
\begin{align*}
    \|\wb^{(t+1)}\|_2^2 &= \|\wb^{(t)}\|_2^2 + \eta^2 \cdot \Bigg\| \frac{1}{n \cdot \| \wb^{(t)} \|_{\bSigma}}\sum_{i=1}^n \ell'[y_i\cdot f(\wb^{(t)},\gamma^{(t)},\xb_i)] \cdot y_i \cdot \gamma^{(t)} \cdot \bigg( \Ib - \frac{\bSigma \wb^{(t)} \wb^{(t)\top}}{\| \wb^{(t)} \|_{\bSigma}^2} \bigg)\xb_i \Bigg\|_2^2\\
    &\leq \|\wb^{(t)}\|_2^2 + \eta^2 \cdot  \frac{ \gamma^{(t)2} }{n \cdot \| \wb^{(t)} \|_{\bSigma}^2}\sum_{i=1}^n \Bigg\|  \bigg( \Ib - \frac{\bSigma \wb^{(t)} \wb^{(t)\top}}{\| \wb^{(t)} \|_{\bSigma}^2} \bigg)\xb_i \Bigg\|_2^2\\
    &\leq \|\wb^{(t)}\|_2^2 + \eta^2 \cdot  \frac{ \gamma^{(t)2} \cdot \max_i \| \xb_i \|_2^2 }{n \cdot \| \wb^{(t)} \|_{\bSigma}^2} \cdot \sum_{i=1}^n \bigg( 1 +  \bigg\| \frac{\bSigma \wb^{(t)} \wb^{(t)\top}}{\| \wb^{(t)} \|_{\bSigma}^2}  \bigg\|_2 \bigg)^2\\
    &= \|\wb^{(t)}\|_2^2 + \eta^2 \cdot  \frac{ \gamma^{(t)2} \cdot \max_i \| \xb_i \|_2^2 }{\| \wb^{(t)} \|_{\bSigma}^2} \cdot \bigg( 1 +   \frac{\| \bSigma \wb^{(t)}\|_2\cdot \| \wb^{(t)}\|_2}{\| \wb^{(t)} \|_{\bSigma}^2} \bigg)^2\\
    &\leq \|\wb^{(t)}\|_2^2 + 4 \eta^2 \cdot  \frac{ \gamma^{(t)} \cdot \max_i \| \xb_i \|_2^2 }{\| \wb^{(t)} \|_{\bSigma}^2} \cdot \bigg( \frac{\| \bSigma \wb^{(t)}\|_2\cdot \| \wb^{(t)} \|_2}{\| \wb^{(t)} \|_{\bSigma}^2} \bigg)^2\\
    &=  \|\wb^{(t)}\|_2^2 + 4 \eta^2 \cdot  \frac{ \gamma^{(t)2} \cdot \max_i \| \xb_i \|_2^2 }{\| \wb^{(t)} \|_{\bSigma}^6} \cdot \| \bSigma \wb^{(t)}\|_2^2\cdot \| \wb^{(t)} \|_2^2,
\end{align*}
where the first inequality follows by Jensen's inequality, and the last inequality follows by the fact that $\| \wb \|_{\bSigma}^2 = \la \wb, \bSigma \wb \ra \leq \| \bSigma \wb\|_2\cdot \| \wb\|_2$. Further plugging in the definition of $\bSigma$ gives
\begin{align*}
    \|\wb^{(t+1)}\|_2^2 &\leq \|\wb^{(t)}\|_2^2 + 4 \eta^2 \cdot  \frac{ \gamma^{(t)2} \cdot \max_i \| \xb_i \|_2^2 }{\| \wb^{(t)} \|_{\bSigma}^6} \cdot \| \wb^{(t)} \|_2^2 \cdot \Bigg\| \frac{1}{n}\sum_{i=1}^n \xb_i \xb_i^\top \wb^{(t)} \Bigg\|_2^2\\
    &\leq \|\wb^{(t)}\|_2^2 + 4 \eta^2 \cdot  \frac{ \gamma^{(t)2} \cdot \max_i \| \xb_i \|_2^3 }{\| \wb^{(t)} \|_{\bSigma}^6} \cdot \| \wb^{(t)} \|_2^2 \cdot \Bigg( \frac{1}{n}\sum_{i=1}^n | \la \xb_i, \wb^{(t)} \ra| \Bigg)^2\\
    &\leq \|\wb^{(t)}\|_2^2 + 4 \eta^2 \cdot  \frac{ \gamma^{(t)2} \cdot \max_i \| \xb_i \|_2^3 }{\| \wb^{(t)} \|_{\bSigma}^6} \cdot \| \wb^{(t)} \|_2^2 \cdot  \frac{1}{n}\sum_{i=1}^n \la \xb_i, \wb^{(t)} \ra^2 \\
    &= \|\wb^{(t)}\|_2^2 + 4 \eta^2 \cdot  \frac{ \gamma^{(t)2} \cdot \max_i \| \xb_i \|_2^3 }{\| \wb^{(t)} \|_{\bSigma}^4} \cdot \| \wb^{(t)} \|_2^2 \\
    &\leq \|\wb^{(t)}\|_2^2 + 4 \eta^2 \cdot  \frac{ \gamma^{(t)2} \cdot \max_i \| \xb_i \|_2^3 }{ \lambda_{\min}^2 \cdot \| \wb^{(t)} \|_2^4} \cdot \| \wb^{(t)} \|_2^2 \\
    &= \|\wb^{(t)}\|_2^2 + 4 \eta^2 \cdot  \frac{ \gamma^{(t)2} \cdot \max_i \| \xb_i \|_2^3 }{ \lambda_{\min}^2 \cdot \| \wb^{(t)} \|_2^2}.
\end{align*}
This finishes the proof of the first result. 

To prove the second result in the lemma, we first denote  $\hat\wb^*_t = \la \wb^{(t)}, \wb^* \ra \cdot \|\wb^*\|_2^{-2} \cdot \wb^*$, and define
\begin{align*}
    \cB^{(t)}:= \{\wb \in \RR^d : \| \wb - \wb^{(t)} \|_2 \leq \| \wb^{(t)} \|_2/2 \}.
\end{align*}
Then the condition $ \|  \wb^{(t)} -  \la \wb^{(t)}, \wb^* \ra \cdot \|\wb^*\|_2^{-1} \cdot \wb^* \|_2 \leq \| \wb^{(0)} \|_2 / 2 $ and the result in the first part that  $\| \wb^{(t)} \|_2^2 \leq \| \wb^{(t+1)} \|_2^2$ for all $t\geq 0$ imply that 
\begin{align*}
     \|  \wb^{(t)} - \hat\wb^*_t \|_2 \leq \| \wb^{(0)} \|_2 / 2 \leq \| \wb^{(t)} \|_2/2,
\end{align*}
and therefore 
\begin{align}\label{eq: wstar_nearwt}
   \hat\wb^*_t \in \cB^{(t)}.
\end{align}
It is also clear that under this condition we have $ \la \wb^{(t)}, \wb^* \ra > 0$. 

We proceed to derive an upper bound of $ \| \nabla_{\wb} L(\wb^{(t)},\gamma^{(t)}) \|_2$. By the definition of $\wb^*,\hat\wb^*_t$, the positive homogeneity of $f$ in $\wb$ and the fact that  $ \la \wb^{(t)}, \wb^* \ra > 0$, it is easy to see that
\begin{align*}
    y_i \cdot f(\hat\wb^{*}_t, \gamma^{(t)} , \xb_i) = y_i \cdot f(\la \wb^{(t)} , \wb^* \ra \cdot \| \wb^* \|_2^{-2} \cdot \wb^*,\gamma^{(t)},\xb_i) = y_i \cdot f( \wb^*,\gamma^{(t)},\xb_i) = \gamma^{(t)}
\end{align*}
for all $i\in [n]$. Therefore, denoting $\ell_{*}'^{(t)} = \ell'( \gamma^{(t)} )$, then we have $\ell_{*}'^{(t)} = \ell'[ y_i \cdot f(\la \wb^{(t)} , \wb^* \ra \cdot \| \wb^* \|_2^{-2} \cdot \wb^*,\gamma^{(t)},\xb_i) ] = \ell'[ y_i \cdot f(\hat\wb^*_t,\gamma^{(t)},\xb_i) ]$ for all $i\in [n]$. Moreover, we have
\begin{align}
    \| \nabla_{\wb} L(\wb^{(t)},\gamma^{(t)})\|_2 &\leq \Bigg\| \frac{ \gamma^{(t)}}{n \cdot \| \wb^{(t)} \|_{\bSigma}}\sum_{i=1}^n  \ell_{*}'^{(t)} \cdot y_i  \cdot \bigg( \Ib - \frac{\bSigma \wb^{(t)} \wb^{(t)\top}}{\| \wb^{(t)} \|_{\bSigma}^2} \bigg)\xb_i \Bigg\|_2 \nonumber \\
    &\quad + \Bigg\| \frac{ \gamma^{(t)} }{n \cdot \| \wb^{(t)} \|_{\bSigma}}\sum_{i=1}^n (\ell_{i}'^{(t)} - \ell_{*}'^{(t)}) \cdot y_i \cdot \bigg( \Ib - \frac{\bSigma \wb^{(t)} \wb^{(t)\top}}{\| \wb^{(t)} \|_{\bSigma}^2} \bigg)\xb_i \Bigg\|_2 \nonumber \\
    &\leq \Bigg\| \frac{ \gamma^{(t)}}{n \cdot \| \wb^{(t)} \|_{\bSigma}}\sum_{i=1}^n \ell_{*}'^{(t)} \cdot y_i\cdot \bigg( \Ib - \frac{\bSigma \wb^{(t)} \wb^{(t)\top}}{\| \wb^{(t)} \|_{\bSigma}^2} \bigg)\xb_i \Bigg\|_2  \nonumber \\
    &\quad +  \frac{ \gamma^{(t)}}{n \cdot \| \wb^{(t)} \|_{\bSigma}}\sum_{i=1}^n |\ell_{i}'^{(t)} - \ell_{*}'^{(t)}| \cdot \Bigg\|\bigg( \Ib - \frac{\bSigma \wb^{(t)} \wb^{(t)\top}}{\| \wb^{(t)} \|_{\bSigma}^2} \bigg)\xb_i \Bigg\|_2 \nonumber \\
    &\leq \Bigg\| \frac{ \gamma^{(t)}}{n \cdot \| \wb^{(t)} \|_{\bSigma}}\sum_{i=1}^n  \ell_{*}'^{(t)} \cdot y_i \cdot \bigg( \Ib - \frac{\bSigma \wb^{(t)} \wb^{(t)\top}}{\| \wb^{(t)} \|_{\bSigma}^2} \bigg)\xb_i \Bigg\|_2  \nonumber \\
    &\quad +  \frac{ \gamma^{(t)}}{n \cdot \| \wb^{(t)} \|_{\bSigma}}\sum_{i=1}^n |\ell_{i}'^{(t)} - \ell_{*}'^{(t)}| \cdot \|\xb_i \|_2 \nonumber \\
    &= \underbrace{\frac{ \gamma^{(t)}\cdot |\ell_{*}'^{(t)}| }{n \cdot \| \wb^{(t)} \|_{\bSigma}} \cdot \Bigg\| \sum_{i=1}^n  \bigg( \Ib - \frac{\bSigma \wb^{(t)} \wb^{(t)\top}}{\| \wb^{(t)} \|_{\bSigma}^2} \bigg)\zb_i \Bigg\|_2}_{I_1}  \nonumber \\
    &\quad + \underbrace{\frac{ \gamma^{(t)}}{n \cdot \| \wb^{(t)} \|_{\bSigma}}\sum_{i=1}^n |\ell_{i}'^{(t)} - \ell_{*}'^{(t)}| \cdot \|\xb_i \|_2}_{I_2}.\label{eq:gradientcalculation_3}
\end{align}
In the following, we bound $I_1$ and $I_2$ separately. For $I_1$, we have
\begin{align}
    \sum_{i=1}^n  \bigg( \Ib - \frac{\bSigma \wb^{(t)} \wb^{(t)\top}}{\| \wb^{(t)} \|_{\bSigma}^2} \bigg)\zb_i 
    & = \| \wb^{(t)} \|_{\bSigma}^{-2}\cdot \sum_{i=1}^n \Big( \| \wb^{(t)} \|_{\bSigma}^{2} \cdot \Ib -  \bSigma \wb^{(t)} \wb^{(t)\top} \Big)\zb_i\nonumber \\
     & = \| \wb^{(t)} \|_{\bSigma}^{-2}\cdot \sum_{i=1}^n \Bigg( \frac{1}{n} \sum_{i'=1}^n \la \wb^{(t)}, \zb_{i'}\ra^2\cdot \Ib - \frac{1}{n}\sum_{i'=1}^n \la \wb^{(t)} , \zb_{i'}\ra \cdot \zb_{i'} \wb^{(t)\top} \Bigg)\zb_i \nonumber \\
     & = \| \wb^{(t)} \|_{\bSigma}^{-2}\cdot \frac{1}{n}\sum_{i,i'=1}^n \la \wb^{(t)}, \zb_{i'}\ra^2\cdot \zb_i - \| \wb^{(t)} \|_{\bSigma}^{-2}\cdot \frac{1}{n}\sum_{i,i'=1}^n \la \wb^{(t)} , \zb_{i'}\ra \cdot  \la\wb^{(t)},\zb_i\ra \cdot \zb_{i'} \nonumber \\
     & = \| \wb^{(t)} \|_{\bSigma}^{-2}\cdot \frac{1}{n}\sum_{i,i'=1}^n \la \wb^{(t)}, \zb_{i}\ra^2\cdot \zb_{i'} - \| \wb^{(t)} \|_{\bSigma}^{-2}\cdot \frac{1}{n}\sum_{i,i'=1}^n \la \wb^{(t)} , \zb_{i'}\ra \cdot  \la\wb^{(t)} ,\zb_i\ra \cdot \zb_{i'} \nonumber \\
     & = \| \wb^{(t)} \|_{\bSigma}^{-2}\cdot \frac{1}{n}\sum_{i,i'=1}^n ( \la \wb^{(t)}, \zb_{i}\ra^2 - \la \wb^{(t)} , \zb_{i'}\ra \cdot  \la\wb^{(t)},\zb_i\ra ) \cdot \zb_{i'}. \label{eq:gradientcalculation_4}
\end{align}
Moreover, by the definition of $\hat\wb^*_t $, it is clear that 
\begin{align*}
    \la \hat\wb^*_t, \zb_{i}\ra^2 - \la \hat\wb^*_t , \zb_{i'}\ra \cdot  \la\hat\wb^*_t,\zb_i\ra  = 0.
\end{align*}
Therefore we have
\begin{align*}
     &\la \wb^{(t)}, \zb_{i}\ra^2 - \la \wb^{(t)} , \zb_{i'}\ra \cdot  \la\wb^{(t)},\zb_i\ra \\
     &\qquad  =  \la \wb^{(t)}, \zb_{i}\ra^2 - \la \wb^{(t)} , \zb_{i'}\ra \cdot  \la\wb^{(t)},\zb_i\ra -  \la \hat\wb^*_t, \zb_{i}\ra^2 + \la \hat\wb^*_t , \zb_{i'}\ra \cdot  \la\hat\wb^*_t,\zb_i\ra \\
     &\qquad = (\la \wb^{(t)}, \zb_{i}\ra^2 -  \la \hat\wb^*_t, \zb_{i}\ra^2 ) - ( \la \wb^{(t)} , \zb_{i'}\ra \cdot  \la\wb^{(t)},\zb_i\ra -  \la \hat\wb^*_t , \zb_{i'}\ra \cdot  \la\hat\wb^*_t,\zb_i\ra  )\\
     &\qquad = (\la \wb^{(t)} + \hat\wb^*_t, \zb_{i}\ra )\cdot (\la \wb^{(t)} - \hat\wb^*_t, \zb_{i}\ra ) - \la \wb^{(t)} , \zb_{i'}\ra \cdot \la \wb^{(t)} - \wb^* , \zb_{i}\ra + \la \wb^{(t)} - \wb^* , \zb_{i'}\ra \cdot \la\hat\wb^*_t,\zb_i\ra.
    %  \cdot
    %  ( \la \wb^{(t)} , \zb_{i'}\ra \cdot  \la\wb^{(t)},\zb_i\ra -  \la \hat\wb^* , \zb_{i'}\ra \cdot  \la\hat\wb^*,\zb_i\ra  )
\end{align*}
Taking absolute value on both sides and applying triangle inequality gives
\begin{align}
    &|\la \wb^{(t)}, \zb_{i}\ra^2 - \la \wb^{(t)} , \zb_{i'}\ra \cdot  \la\wb^{(t)},\zb_i\ra| \nonumber \\
    &\quad \leq |\la \wb^{(t)} + \hat\wb^*_t, \zb_{i}\ra | \cdot |\la \wb^{(t)} - \hat\wb^*_t, \zb_{i}\ra | + |\la \wb^{(t)} , \zb_{i'}\ra| \cdot |\la \wb^{(t)} - \wb^* , \zb_{i}\ra| + |\la \wb^{(t)} - \wb^* , \zb_{i'}\ra| \cdot |\la\hat\wb^*_t,\zb_i\ra| \nonumber \\
    &\quad \leq (\| \wb^{(t)}\|_2 + \|\hat\wb^*_t\|_2)\cdot \|\zb_{i}\|_2 \cdot \| \wb^{(t)} - \hat\wb^*_t\|_2 \cdot\| \zb_{i}\|_2 + \| \wb^{(t)}\|_2 \cdot \|\zb_{i'}\|_2 \cdot \| \wb^{(t)} - \wb^*\|_2 \cdot \| \zb_{i}\|_2 \nonumber \\
    &\quad\quad+ \| \wb^{(t)} - \wb^* \|_2 \cdot \| \zb_{i'}\|_2 \cdot \|\hat\wb^*_t\|_2\cdot \|\zb_i\|_2 \nonumber \\
    &\quad \leq 2 (\| \wb^{(t)}\|_2 + \|\hat\wb^*_t\|_2)\cdot \max_i \|\zb_{i}\|_2^2 \cdot \| \wb^{(t)} - \hat\wb^*_t\|_2 \nonumber \\
    &\quad \leq 4 \| \wb^{(t)}\|_2 \cdot \max_i \|\zb_{i}\|_2^2 \cdot \| \wb^{(t)} - \hat\wb^*_t\|_2 \nonumber \\
    &\quad = 4 \| \wb^{(t)}\|_2 \cdot \max_i \|\xb_{i}\|_2^2 \cdot \| \wb^{(t)} - \hat\wb^*_t\|_2.\label{eq:gradientcalculation_5}
\end{align}
Plugging \eqref{eq:gradientcalculation_5} into \eqref{eq:gradientcalculation_4} gives
\begin{align}
I_1 &= \frac{ \gamma^{(t)}\cdot |\ell_{*}'^{(t)}| }{n \cdot \| \wb^{(t)} \|_{\bSigma}} \cdot \Bigg\| \sum_{i=1}^n  \bigg( \Ib - \frac{\bSigma \wb^{(t)} \wb^{(t)\top}}{\| \wb^{(t)} \|_{\bSigma}^2} \bigg)\zb_i \Bigg\|_2\nonumber\\
&\leq \frac{ 4\gamma^{(t)}\cdot |\ell_{*}'^{(t)}| }{\| \wb^{(t)} \|_{\bSigma}} \cdot\frac{ \| \wb^{(t)}\|_2}{ \| \wb^{(t)} \|_{\bSigma}^2 } \cdot \max_i \|\xb_{i}\|_2^3 \cdot \| \wb^{(t)} - \hat\wb^*_t\|_2\nonumber \\
&\leq \frac{ 4\gamma^{(t)}\cdot \exp(-\gamma^{(t)}) }{\| \wb^{(t)} \|_{\bSigma}} \cdot\frac{ \| \wb^{(t)}\|_2}{ \| \wb^{(t)} \|_{\bSigma}^2 } \cdot \max_i \|\xb_{i}\|_2^3 \cdot \| \wb^{(t)} - \hat\wb^*_t\|_2\nonumber \\
&\leq  \frac{ 4\gamma^{(t)}\cdot \exp(-\gamma^{(t)}) }{\lambda_{\min}^{3/2}\cdot \| \wb^{(t)} \|_2^2} \cdot \max_i \|\xb_{i}\|_2^3 \cdot \| \wb^{(t)} - \hat\wb^*_t\|_2\nonumber \\
&\leq  \frac{ 4\gamma^{(t)}\cdot \exp(-\gamma^{(t)}) }{\lambda_{\min}^{3/2}\cdot \| \wb^{(0)} \|_2^2} \cdot \max_i \|\xb_{i}\|_2^3 \cdot \| \wb^{(t)} - \hat\wb^*_t\|_2\label{eq:gradientcalculation_6}
% &\leq \frac{ 2\gamma^{(t)}\cdot \exp(-\gamma^{(t)})}{\| \wb^{(t)} \|_{\bSigma}} \cdot\frac{ (1 + \|\wb^*\|_2)\cdot \| \wb^{(t)}\|_2}{ \| \wb^{(t)} \|_{\bSigma}^2 } \cdot \max_i \|\xb_{i}\|_2^3 \cdot \| \wb^{(t)} - \hat\wb^*\|_2, 
\end{align}
where the second inequality follows by the fact that $-\exp(-\gamma^{(t)}) \leq  \ell_{*}'^{(t)} = \ell'(\gamma^{(t)} ) < 0$. 
% \Bigg\| \sum_{i=1}^n  \bigg( \Ib - \frac{\bSigma \wb^{(t)} \wb^{(t)\top}}{\| \wb^{(t)} \|_{\bSigma}^2} \bigg)\zb_i \Bigg\|_2 \leq
Regarding the bound for $I_2$, by the mean value theorem, there exists $z$ between $ y_i\cdot f(\wb^{(t)}, \gamma^{(t)}, \xb_i)$ and $ y_i\cdot f(\hat\wb^*_t, \gamma^{(t)}, \xb_i) = \gamma^{(t)}$ such that 
\begin{align}
    |\ell_{i}'^{(t)} - \ell_{*}'^{(t)}| &= |\ell''( z) \cdot y_i \cdot [f(\wb^{(t)}, \gamma^{(t)}, \xb_i) - f(\hat\wb^*_t, \gamma^{(t)}, \xb_i)]| \nonumber\\
    &\leq |\ell''( z)| \cdot |f(\wb^{(t)}, \gamma^{(t)}, \xb_i) - f(\hat\wb^*_t, \gamma^{(t)}, \xb_i)| \nonumber \\
    &\leq \max\{ |\ell_1'^{(t)}|, \ldots, |\ell_n'^{(t)}|, \exp(-\gamma^{(t)}) \} \cdot |f(\wb^{(t)}, \gamma^{(t)}, \xb_i) - f(\hat\wb^*_t, \gamma^{(t)}, \xb_i)|.\label{eq:l'_meanvalue}
\end{align}
where the first inequality follows by the property of cross-entropy loss that $0\leq \ell''( z) \leq - \ell'( z)$, and the second inequality follows from the fact that $z$ is between $ y_i\cdot f(\wb^{(t)}, \gamma^{(t)}, \xb_i)$ and $ y_i\cdot f(\hat\wb^*_t, \gamma^{(t)}, \xb_i) = \gamma^{(t)}$. 
Moreover, by \eqref{eq: wstar_nearwt} we have $\hat\wb^*_t \in \cB^{(t)}$ and $\| \wb \|_2 \geq \| \wb^{(t)} \|_2 / 2 $ for all $\wb \in \cB^{(t)}= \{\wb \in \RR^d : \| \wb - \wb^{(t)} \|_2 \leq \| \wb^{(t)} \|_2/2 \}$. For all $\wb \in \cB^{(t)}$, we have
\begin{align*}
    \| \nabla_{\wb} f(\wb, \gamma^{(t)}, \xb_i)\|_2 &=  \bigg\| \frac{ \gamma^{(t)}}{ \| \wb\|_{\bSigma}} \cdot \bigg( \Ib - \frac{\bSigma \wb \wb^{\top}}{\| \wb \|_{\bSigma}^2} \bigg)\xb_i \bigg\|_2 \\
    % & \leq \frac{ \gamma^{(t)}}{ \| \wb \|_{\bSigma}} \cdot  \Bigg\| \bigg( \Ib - \frac{\bSigma \wb \wb^{\top}}{\| \wb \|_{\bSigma}^2} \bigg) \Bigg\|_2 \cdot \| \xb_i \|_2\\
    &\leq \frac{ \gamma^{(t)}}{ \| \wb \|_{\bSigma}} \cdot \| \xb_i \|_2 \cdot \bigg( 1 +  \bigg\| \frac{\bSigma \wb \wb^{\top}}{\| \wb \|_{\bSigma}^2}  \bigg\|_2 \bigg) \\
    &= \frac{ \gamma^{(t)}}{ \| \wb \|_{\bSigma}} \cdot \| \xb_i \|_2 \cdot \bigg( 1 +   \frac{\| \bSigma \wb\|_2\cdot \| \wb\|_2}{\| \wb \|_{\bSigma}^2} \bigg) \\
    &\leq  \frac{2 \gamma^{(t)}}{ \| \wb \|_{\bSigma}} \cdot \| \xb_i \|_2 \cdot \frac{\| \bSigma \wb\|_2\cdot \| \wb\|_2}{\| \wb \|_{\bSigma}^2} ,
\end{align*}
where the last inequality follows by the fact that $\| \wb \|_{\bSigma}^2 = \la \wb, \bSigma \wb \ra \leq \| \bSigma \wb\|_2\cdot \| \wb\|_2$. Further plugging in the definition of $\bSigma$ gives
\begin{align*}
\| \nabla_{\wb} f(\wb, \gamma^{(t)}, \xb_i)\|_2
    &\leq \frac{2 \gamma^{(t)}}{ \| \wb \|_{\bSigma}^3} \cdot \| \xb_i \|_2 \cdot \| \wb\|_2 \cdot \Bigg\| \frac{1}{n}\sum_{j=1}^n \xb_j \cdot \la \wb, \xb_j \ra \Bigg\|_2  \\
    &\leq \frac{2 \gamma^{(t)}}{ \| \wb \|_{\bSigma}^3} \cdot \max_i\| \xb_i \|_2^2 \cdot \| \wb\|_2 \cdot \frac{1}{n}\sum_{i=1}^n |\la \wb, \xb_i \ra|  \\
    &\leq \frac{2 \gamma^{(t)}}{ \| \wb \|_{\bSigma}^3} \cdot \max_i\| \xb_i \|_2^2 \cdot \| \wb\|_2 \cdot \sqrt{\frac{1}{n}\sum_{i=1}^n |\la \wb, \xb_i \ra|^2 } \\
    &= \frac{2 \gamma^{(t)}}{ \| \wb \|_{\bSigma}^2} \cdot \max_i\| \xb_i \|_2^2 \cdot \| \wb\|_2 \\
    &\leq \frac{2 \gamma^{(t)}}{ \lambda_{\min}\cdot \| \wb \|_2^2} \cdot \max_i\| \xb_i \|_2^2 \cdot \| \wb\|_2 \\
    &\leq \frac{4 \gamma^{(t)}}{ \lambda_{\min}\cdot \| \wb^{(t)} \|_2} \cdot \max_i\| \xb_i \|_2^2,
    % \\
    % &\leq \frac{ \gamma^{(t)}}{ \lambda_{\min}^{1/2}\cdot \| \wb \|_2 } \cdot \| \xb_i \|_2\\
    % &\leq \frac{ 2\gamma^{(t)}}{ \lambda_{\min}^{1/2}\cdot \| \wb^{(t)} \|_2 } \cdot \| \xb_i \|_2,
\end{align*}
where the third inequality follows by Jensen's inequality, and the last inequality follows by $\|\wb\|_2 \geq \| \wb^{(t)} \|_2 / 2$ for all $\wb \in \cB^{(t)}$. 
% \begin{align*}
%     \| \nabla_{\wb} f(\wb, \gamma^{(t)}, \xb_i)\|_2 &=  \bigg\| \frac{ \gamma^{(t)}}{ \| \wb\|_{\bSigma}} \cdot \bigg( \Ib - \frac{\bSigma \wb \wb^{\top}}{\| \wb \|_{\bSigma}^2} \bigg)\xb_i \bigg\|_2 \\
%     % & \leq \frac{ \gamma^{(t)}}{ \| \wb \|_{\bSigma}} \cdot  \Bigg\| \bigg( \Ib - \frac{\bSigma \wb \wb^{\top}}{\| \wb \|_{\bSigma}^2} \bigg) \Bigg\|_2 \cdot \| \xb_i \|_2\\
%     &\leq \frac{ \gamma^{(t)}}{ \| \wb \|_{\bSigma}} \cdot \| \xb_i \|_2\\
%     &\leq \frac{ \gamma^{(t)}}{ \lambda_{\min}^{1/2}\cdot \| \wb \|_2 } \cdot \| \xb_i \|_2\\
%     &\leq \frac{ 2\gamma^{(t)}}{ \lambda_{\min}^{1/2}\cdot \| \wb^{(t)} \|_2 } \cdot \| \xb_i \|_2,
% \end{align*}
% where the first inequality follows by the fact that the matrix $\| \wb \|_{\bSigma}^{-2}\cdot\bSigma \wb \wb^{\top}$ has eigenvalues $1,0,0,\ldots, 0$ ($\bSigma \wb / \|\bSigma \wb \|_2$ is the right eigenvector corresponding to the eigenvalue $1$).
Therefore, $g(\wb):= f(\wb, \gamma^{(t)}, \xb_i)$ is $  (4 \lambda_{\min}^{-1} \cdot \| \wb^{(t)} \|_{2}^{-1}\cdot \gamma^{(t)} \cdot\max_i\| \xb_i \|_2^2)$-Lipschitz continuous over $\cB^{(t)}$, and
\begin{align}\label{eq:f_Lipschitz}
    |f(\wb^{(t)}, \gamma^{(t)}, \xb_i) - f(\hat\wb^*_t, \gamma^{(t)}, \xb_i)| \leq  \frac{ 4 \gamma^{(t)}}{ \lambda_{\min}\cdot \| \wb^{(t)} \|_2} \cdot \max_i\| \xb_i \|_2^2 \cdot \| \wb^{(t)} - \hat\wb^*_t \|_2.
\end{align}
Plugging \eqref{eq:f_Lipschitz} into \eqref{eq:l'_meanvalue} gives 
% we consider the functions $g_i^{(t)}(\wb) = \ell'[y_i\cdot f(\wb, \gamma^{(t)}, \xb_i)]$ for $i\in[n]$ and $t\geq 0$. Then by definition, we have
% \begin{align*}
%     \| \nabla g_i^{(t)}(\wb)\|_2 &= \Bigg\|  \ell''[y_i\cdot f(\wb, \gamma^{(t)}, \xb_i)] \cdot y_i\cdot \frac{ \gamma^{(t)}}{ \| \wb^{(t)} \|_{\bSigma}} \cdot \bigg( \Ib - \frac{\bSigma \wb^{(t)} \wb^{(t)\top}}{\| \wb^{(t)} \|_{\bSigma}^2} \bigg)\xb_i \Bigg\|_2\\
%     & \leq \frac{ \gamma^{(t)}}{ \| \wb^{(t)} \|_{\bSigma}} \cdot  \Bigg\| \bigg( \Ib - \frac{\bSigma \wb^{(t)} \wb^{(t)\top}}{\| \wb^{(t)} \|_{\bSigma}^2} \bigg) \Bigg\|_2 \cdot \| \xb_i \|_2\\
%     &\leq \frac{ \gamma^{(t)}}{ \| \wb^{(t)} \|_{\bSigma}} \cdot \| \xb_i \|_2,
% \end{align*}
% where the first inequality above follows by $\ell''(\cdot) \leq 1$, and the second inequality follows by the fact that the matrix $\| \wb^{(t)} \|_{\bSigma}^{-2}\cdot\bSigma \wb^{(t)} \wb^{(t)\top}$ has eigenvalues $1,0,0,\ldots, 0$. 
% Therefore $g_i^{(t)}(\wb)$ is $  (\| \wb^{(t)} \|_{\bSigma}^{-1}\cdot \gamma^{(t)} \cdot \| \xb_i \|_2)$-Lipschitz continuous, and
% \begin{align*}
%     \frac{ \gamma^{(t)}}{n \cdot \| \wb^{(t)} \|_{\bSigma}}\sum_{i=1}^n |\ell_{i}'^{(t)} - \ell_{*}'^{(t)}| \cdot \|\xb_i \|_2
% \end{align*}
\begin{align*}
     |\ell_{i}'^{(t)} - \ell_{*}'^{(t)}| \leq \max\{ |\ell_1'^{(t)}|, \ldots, |\ell_n'^{(t)}|, \exp(-\gamma^{(t)}) \} \cdot \frac{ 4 \gamma^{(t)}}{ \lambda_{\min}\cdot \| \wb^{(t)} \|_2} \cdot \max_i\| \xb_i \|_2^2 \cdot \| \wb^{(t)} - \hat\wb^*_t \|_2
\end{align*}
for all $i\in[n]$ and $t \geq 0$. Then by the definition of $I_2$, we have
\begin{align}
    I_2 &= \frac{ \gamma^{(t)}}{n \cdot \| \wb^{(t)} \|_{\bSigma}}\sum_{i=1}^n |\ell_{i}'^{(t)} - \ell_{*}'^{(t)}| \cdot \|\xb_i \|_2 \nonumber\\
    &\leq \frac{ \gamma^{(t)}}{ \| \wb^{(t)} \|_{\bSigma}} \cdot \max_i \| \xb_i\|_2^3  \cdot \max\{ |\ell_1'^{(t)}|, \ldots, |\ell_n'^{(t)}|, \exp(-\gamma^{(t)}) \} \cdot  \frac{ 4 \gamma^{(t)}}{ \lambda_{\min} \cdot \| \wb^{(t)} \|_2} \cdot \| \wb^{(t)} - \hat\wb^*_t \|_2\nonumber\\
    &\leq  \frac{ 4 \gamma^{(t)2}}{ \lambda_{\min}^{3/2}\cdot \| \wb^{(t)} \|_2^2} \cdot \max_i \| \xb_i\|_2^3  \cdot \max\{ |\ell_1'^{(t)}|, \ldots, |\ell_n'^{(t)}|, \exp(-\gamma^{(t)}) \} \cdot \| \wb^{(t)} - \hat\wb^*_t \|_2\nonumber\\
    &\leq  \frac{ 4 \gamma^{(t)2}}{ \lambda_{\min}^{3/2}\cdot \| \wb^{(0)} \|_2^2} \cdot \max_i \| \xb_i\|_2^3  \cdot \max\{ |\ell_1'^{(t)}|, \ldots, |\ell_n'^{(t)}|, \exp(-\gamma^{(t)}) \} \cdot \| \wb^{(t)} - \hat\wb^*_t \|_2 \label{eq:gradientcalculation_7} 
\end{align}
Plugging \eqref{eq:gradientcalculation_3}, \eqref{eq:gradientcalculation_6} and \eqref{eq:gradientcalculation_7} into \eqref{eq:gradientcalculation_2} gives
\begin{align*}
    \|\wb^{(t+1)}\|_2^2 &\leq \|\wb^{(t)}\|_2^2 + 2 \eta^2 \cdot \Bigg[ \frac{ 4\gamma^{(t)}\cdot \exp(-\gamma^{(t)}) }{\lambda_{\min}^{3/2}\cdot \| \wb^{(0)} \|_2^2} \cdot \max_i \|\xb_{i}\|_2^3 \cdot \| \wb^{(t)} - \hat\wb^*_t\|_2 \Bigg]^2\\
    &\quad +  2 \eta^2 \cdot\Bigg[   \frac{ 4 \gamma^{(t)2}}{ \lambda_{\min}^{3/2}\cdot \| \wb^{(0)} \|_2^2} \cdot \max_i \| \xb_i\|_2^3  \cdot \max\{ |\ell_1'^{(t)}|, \ldots, |\ell_n'^{(t)}|, \exp(-\gamma^{(t)}) \} \cdot \| \wb^{(t)} - \hat\wb^*_t \|_2 \Bigg]^2\\
    &\leq \|\wb^{(t)}\|_2^2 + \eta^2  G\cdot \max\{\gamma^{(t)2}, \gamma^{(t)4}\}\cdot \max\{ |\ell_1'^{(t)}|^2, \ldots, |\ell_n'^{(t)}|^2, \exp(-2\gamma^{(t)}) \}\cdot  \| \wb^{(t)} - \hat\wb^*_t \|_2^2,
\end{align*}
where $G = 64 \lambda_{\min}^{-3}\cdot \max_i \|\xb_{i}\|_2^6 \cdot \| \wb^{(0)} \|_2^{-4}$. 
% $G = 32 \max\{  \lambda_{\min}^{-3/2}\cdot \max_i \|\xb_{i}\|_2^3,  \lambda_{\min}^{-1}\cdot \max_i \|\xb_{i}\|_2^2\} \cdot \| \wb^{(0)} \|_2^{-2} $. Finally, note that by definition, 
% \begin{align*}
%     \lambda_{\min} \leq \tr(\bSigma) = \tr\Bigg(\frac{1}{n} \sum_{i=1}^n \xb_i \xb_i^\top\Bigg) = \frac{1}{n} \sum_{i=1}^n \| \xb_i\|_2^2 \leq \max_i \|\xb_{i}\|_2^2.
% \end{align*}
% Therefore
% \begin{align*}
%     G
%     &= 32 \max\{  \lambda_{\min}^{-3/2}\cdot \max_i \|\xb_{i}\|_2^3,  \lambda_{\min}^{-1}\cdot \max_i \|\xb_{i}\|_2^2\} \cdot \| \wb^{(0)} \|_2^{-2}\\
%     &= 32 \lambda_{\min}^{-3/2}\cdot \max_i \|\xb_{i}\|_2^2\cdot \max\{  \max_i \|\xb_{i}\|_2,   \lambda_{\min}^{1/2}\} \cdot \| \wb^{(0)} \|_2^{-2}\\
%      &= 32 \lambda_{\min}^{-3/2}\cdot \max_i \|\xb_{i}\|_2^3 \cdot \| \wb^{(0)} \|_2^{-2}.
% \end{align*}
This finishes the proof.
\end{proof}

\subsection{Proof of Lemma~\ref{lemma:firststage}}
In preparation of the proof of Lemma~\ref{lemma:firststage}, we first present the following lemma.
\begin{lemma}\label{lemma:gamma_comparison}
Suppose that a sequence $a^{(t)}$, $t\geq 0$ follows the iterative formula
\begin{align*}
    a^{(t+1 )} = a^{(t)} + c \cdot \exp(-a^{(t)})
\end{align*}
for some $c>0$. 
Then it holds that
\begin{align*}
    \log(c\cdot t + \exp(a^{(0)})) \leq a^{(t)} \leq c \exp(-a^{(0)}) + \log( c\cdot t +\exp(a^{(0)}) )
\end{align*}
for all $t\geq 0$.
\end{lemma}
The proof of Lemma~\ref{lemma:firststage} is given as follows. 

\begin{proof}[Proof of Lemma~\ref{lemma:firststage}]
Set $T_0 = 1 + 400 \eta^{-1} \epsilon^{-2} \cdot \lambda_{\max}^{3/2} \cdot \lambda_{\min}^{-1} \cdot \| \wb^{(0)} \|_2^2 \cdot \| \wb^* \|_2 $. By gradient descent update rule, we have
\begin{align*}
    |\gamma^{(t+1)} - \gamma^{(t)}| &= \Bigg| \frac{\eta}{n}\sum_{i=1}^n \ell_l'^{(t)}\cdot\frac{\la \wb^{(t)}, \zb_i\ra}{\sqrt{n^{-1}\cdot \sum_{j=1}^n \la \wb^{(t)}, \zb_i\ra^2}} \Bigg|\\
    &\leq \eta\cdot \frac{  n^{-1}\cdot \sum_{i=1}^n | \la \wb^{(t)}, \zb_i\ra|}{\sqrt{n^{-1}\cdot \sum_{j=1}^n \la \wb^{(t)}, \zb_i\ra^2}} \\
    &\leq \eta,
\end{align*}
where the first inequality follows by the fact that $|\ell'(z)| < 1 $ for all $z\in \RR$, and the second inequality follows by Jensen's inequality. Therefore, we have
\begin{align*}
     | \gamma^{(t)} - 1 | &  = | \gamma^{(t)} - \gamma^{(0)} | \leq \eta + 400 \epsilon^{-2} \cdot \lambda_{\max}^{3/2} \cdot \lambda_{\min}^{-1} \cdot \| \wb^{(0)} \|_2^2 \cdot \| \wb^* \|_2.
    %  \\
    %  &\leq \eta + 400 \epsilon^{-2} \cdot \lambda_{\max}^{3/2} \cdot \lambda_{\min}^{-1} \cdot \| \wb^{(0)} \|_2^2 \cdot \| \wb^* \|_2
\end{align*}
Further plugging in the definition of $\epsilon$ gives
\begin{align}
     | \gamma^{(t)} - 1 | & \leq  \eta +  \frac{6400 \max_i \| \zb_i \|_2 \cdot  \| \wb^* \|_2^2 \cdot \lambda_{\max}^{3/2} \cdot \lambda_{\min}^{-1} \cdot \| \wb^{(0)} \|_2^2 }{ \min\big\{1/3,  \| \wb^{(0)} \|_2^{-1} \cdot \lambda_{\min}^{1/2} / (40 \lambda_{\max}^{3/4}) \big\} }\nonumber\\
     &= \eta +  19200 \max_i \| \zb_i \|_2 \cdot  \| \wb^* \|_2^2 \cdot \lambda_{\max}^{3/2} \cdot \lambda_{\min}^{-1} \cdot \| \wb^{(0)} \|_2^2 \nonumber\\
     &\leq 1/2\label{eq:firststage_firstresult_overlineT}
\end{align}
for all $t\in [T_0]$, 
where the first equality follows by the assumption that $\| \wb^{(0)} \|_2 \leq \lambda_{\min}^{1/2} / (20 \lambda_{\max}^{3/4})$, and the second inequality follows by the assumption that $\eta \leq 1/4$ and $ \| \wb^{(0)} \|_2 \leq (\max_i \| \zb_i \|_2)^{-1/2} \cdot  \| \wb^* \|_2^{-1} \cdot \lambda_{\max}^{-3/4} \cdot \lambda_{\min}^{1/2} / 280$. 

By Lemma~\ref{lemma:norm_upperbound}, we have
\begin{align*}
     \|\wb^{(t)}\|_2^2  \leq \|\wb^{(t+1)}\|_2^2 \leq \|\wb^{(t)}\|_2^2 + 4 \eta^2 \cdot  \frac{ \gamma^{(t)2} \cdot \max_i \| \xb_i \|_2^3 }{ \lambda_{\min}^2 \cdot \| \wb^{(t)} \|_2^2}.
\end{align*}
By the monotonicity of $\|\wb^{(t)}\|_2$ and the result in \eqref{eq:firststage_firstresult_overlineT} that $\gamma^{(t)} \leq 3/2$ for all $t\in [T_0]$, we have
\begin{align*}
     \|\wb^{(t+1)}\|_2^2 \leq \|\wb^{(t)}\|_2^2 + 9 \eta^2 \cdot  \frac{ \max_i \| \xb_i \|_2^3 }{ \lambda_{\min}^2 \cdot \| \wb^{(0)} \|_2^2}
\end{align*}
for all  $t\in [T_0]$. 
Taking a telescoping sum then gives
\begin{align*}
    \|\wb^{(t)}\|_2^2 \leq \|\wb^{(0)}\|_2^2 +    \frac{ 9 \eta^2 t \cdot \max_i \| \xb_i \|_2^3 }{ \lambda_{\min}^2 \cdot \| \wb^{(0)} \|_2^2} \cdot t \leq \|\wb^{(0)}\|_2^2 +  \frac{ 9 \eta^2 T_0\cdot \max_i \| \xb_i \|_2^3 }{ \lambda_{\min}^2 \cdot \| \wb^{(0)} \|_2^2}
\end{align*}
for all $t\in [T_0]$. 
Plugging in the definition of $T_0$ gives
\begin{align*}
    \|\wb^{(t)}\|_2^2 &\leq \|\wb^{(0)}\|_2^2 +  \frac{ 9 \eta^2 \cdot \max_i \| \xb_i \|_2^3 }{ \lambda_{\min}^2 \cdot \| \wb^{(0)} \|_2^2}\cdot ( 1 + 400 \eta^{-1} \epsilon^{-2} \cdot \lambda_{\max}^{3/2} \cdot \lambda_{\min}^{-1} \cdot \| \wb^{(0)} \|_2^2 \cdot \| \wb^* \|_2 )\\
    &= \|\wb^{(0)}\|_2^2 +  \frac{ 9 \eta^2 \cdot \max_i \| \xb_i \|_2^3 }{ \lambda_{\min}^2 \cdot \| \wb^{(0)} \|_2^2} + \frac{ 3600 \eta \cdot \max_i \| \xb_i \|_2^3 }{ \lambda_{\min}^3 } \cdot \epsilon^{-2} \cdot \lambda_{\max}^{3/2} \cdot \| \wb^* \|_2 )\\
    &\leq (1 + \epsilon^2 / 4)\cdot \|\wb^{(0)}\|_2^2
\end{align*}
for all $t\in [T_0]$, 
where the last inequality follows by the assumption that $\eta \leq \min\{ \epsilon\cdot \| \wb^{(0)} \|_2^2 \cdot \lambda_{\min} \cdot (\max_i \| \xb_i \|_2)^{-3/2} / 9, \epsilon^4\cdot \lambda_{\min}^3 \cdot \lambda_{\max}^{-3/2}\cdot  \| \wb^* \|_2^{-1} \cdot (\max_i \| \xb_i \|_2)^{-3} / 28800\}$.
Therefore 
\begin{align}\label{eq:firststage_secondresult_overlineT}
    \|\wb^{(t)}\|_2 \leq \sqrt{1 + \epsilon^2 / 4}\cdot \|\wb^{(0)}\|_2 \leq (1 + \epsilon / 2)\cdot  \|\wb^{(0)}\|_2
\end{align}
for all $t\in [T_0]$.
By Lemma~\ref{lemma:innerproduct_lowerbound}, we have 
\begin{align}\label{eq:innerproduct_lowerbound_proof_stage1}
    \la \wb^{(t+1)}, \wb^* \ra &\geq \la \wb^{(t)}, \wb^* \ra +\frac{ \gamma^{(t)} \eta }{16 \| \wb^{(t)} \|_{\bSigma}^3}\cdot \exp(-\gamma^{(t)} ) \cdot  \|  \wb - \la \wb^*, \wb\ra_{\bSigma} \cdot \wb^* \|_{\bSigma}^2.
\end{align}
By the result that $\gamma^{(t)} \leq 3/2$ for all $t\in T_0$, we have $\exp(-\gamma^{(t)}) \geq \exp(-\gamma^{(3/2)}) \geq 1/5$
% Note that $|\ell'(z)| = [ 1 + \exp(z) ]^{-1}$ is a decreasing function, and we have 
% \begin{align*}
%     \max_i|\ell_{i}'^{(t)}| &= \max_i \Bigg| \ell'\Bigg( \frac{ \gamma^{(t)}\cdot \la  \wb^{(t)} , \zb_i \ra}{  \sqrt{n^{-1}\cdot \sum_j \la  \wb^{(t)} , \zb_j \ra^2 } }\Bigg) \Bigg|
%     = \Bigg| \ell'\Bigg( \frac{  \gamma^{(t)}\cdot   \min_i \la  \wb^{(t)} , \zb_i \ra}{ \sqrt{n^{-1}\cdot \sum_j \la  \wb^{(t)} , \zb_j \ra^2 } } \Bigg) \Bigg|\\
%     &\geq | \ell'( \gamma^{(t)} ) |
%     \geq \frac{1}{1 + \exp( 1.5 )}
%     \geq \frac{1}{6}
% \end{align*}
for all $t\in[T_0]$. %, . Therefore
% where the second inequality follows by the proved result that $ 1/2 \leq \gamma^{(t)} \leq 3/2$ for all $t \in [T_0]$.
Therefore by \eqref{eq:innerproduct_lowerbound_proof_stage1}, we have
\begin{align*}
    \la \wb^{(t+1)}, \wb^* \ra &\geq \la \wb^{(t)}, \wb^* \ra +\frac{ \gamma^{(t)} \eta }{80 \| \wb^{(t)} \|_{\bSigma}^3} \cdot  \|  \wb - \la \wb^*, \wb\ra_{\bSigma} \cdot \wb^* \|_{\bSigma}^2\\
    &\geq \la \wb^{(t)}, \wb^* \ra +\frac{ \gamma^{(t)} \eta }{80 \| \wb^{(0)} \|_{\bSigma}^3} \cdot  \|  \wb - \la \wb^*, \wb\ra_{\bSigma} \cdot \wb^* \|_{\bSigma}^2\\
    &\geq \la \wb^{(t)}, \wb^* \ra +\frac{ \eta }{160 \| \wb^{(0)} \|_{\bSigma}^3} \cdot  \|  \wb - \la \wb^*, \wb\ra_{\bSigma} \cdot \wb^* \|_{\bSigma}^2
\end{align*}
for all $t\in[T_0]$, 
where the second inequality follows by Lemma~\ref{lemma:norm_upperbound}, and the third inequality follows by the proved result that $\gamma^{(t)} \geq 1/2$ for all $t\in[T_0]$. Telescoping over $t=0,\ldots,T_0 - 1$ then gives
\begin{align*}
    \min_{t\in [T_0 - 1]}  \|  \wb - \la \wb^*, \wb\ra_{\bSigma} \cdot \wb^* \|_{\bSigma}^2 &\leq \frac{1}{T_0 - 1} \sum_{t=0}^{T_0 - 1} \|  \wb - \la \wb^*, \wb\ra_{\bSigma} \cdot \wb^* \|_{\bSigma}^2 \\
    &\leq \frac{160 \| \wb^{(0)} \|_{\bSigma}^3}{  (T_0 - 1) \eta } \cdot (\la \wb^{(T_0)}, \wb^* \ra - \la \wb^{(0)}, \wb^* \ra)\\
    &\leq \frac{160 \| \wb^{(0)} \|_{\bSigma}^3}{  (T_0 - 1) \eta } \cdot (\|  \wb^{(T_0)}\|_2  + \| \wb^{(0)} \|_2 )\cdot \| \wb^* \|_2.
    % \\
    % &\leq \frac{144 n \| \wb^{(0)} \|_{\bSigma}^3}{  (T_0 - 1) \eta } \cdot  \| \wb^{(0)} \|_2 \cdot \| \wb^* \|_2\\
    % &\leq 
\end{align*}
By the proved result that $\|  \wb^{(T_0)}\|_2 \leq (1 + \epsilon / 2)\cdot \|  \wb^{(0)}\|_2  \leq 1.5 \|  \wb^{(0)}\|_2 $ (by definition, $\epsilon \leq \max_i \la \wb^*, \zb_i \ra^{-1}/48 = 1/48$), we then obtain
\begin{align*}
    \min_{t\in [T_0 - 1]}  \|  \wb - \la \wb^*, \wb\ra_{\bSigma} \cdot \wb^* \|_{\bSigma}^2 
    &\leq \frac{400 \| \wb^{(0)} \|_{\bSigma}^3}{  (T_0 - 1) \eta } \cdot  \| \wb^{(0)} \|_2 \cdot \| \wb^* \|_2\\
    &\leq \frac{400 \lambda_{\max}^{3/2} }{  (T_0 - 1) \eta } \cdot  \| \wb^{(0)} \|_2^4 \cdot \| \wb^* \|_2\\
    &\leq \epsilon^2 \cdot  \lambda_{\min} \cdot  \| \wb^{(0)} \|_2^2,
\end{align*}
where the last inequality follows by the definition that $T_0 = 1 + 400 \eta^{-1} \epsilon^{-2} \cdot \lambda_{\max}^{3/2} \cdot \lambda_{\min}^{-1} \cdot \| \wb^{(0)} \|_2^2 \cdot \| \wb^* \|_2$. Therefore, we see that there exists $t_0 \in [T_0 - 1]$ such that 
\begin{align*}
    \big\| \wb^{(t_0)} - \la \wb^*, \wb^{(t_0 )} \ra \cdot \| \wb^* \|_2^{-2}  \cdot \wb^* \big\|_2 &\leq \|  \wb - \la \wb^*, \wb\ra_{\bSigma} \cdot \wb^* \|_2\\
    &\leq \lambda_{\min}^{-1/2} \cdot
    \|  \wb - \la \wb^*, \wb\ra_{\bSigma} \cdot \wb^* \|_{\bSigma}\\
    &\leq \epsilon \cdot  \| \wb^{(0)} \|_2,
\end{align*}
where the first inequality follows by the fact that $\la \wb^*, \wb^{(t_0 )} \ra \cdot \| \wb^* \|_2^{-2}  \cdot \wb^*$ is the $\ell_2$-projection of $\wb^{(t_0)}$ on $\mathrm{span}\{ \wb^*\}$. 
Together with \eqref{eq:firststage_firstresult_overlineT} and \eqref{eq:firststage_secondresult_overlineT}, we conclude that there exists $t_0 \in [T_0 - 1]$ such that all three results of Lemma~\ref{lemma:firststage} hold.
% \begin{align*}
%     \la \wb^{(t+1)}, \wb^* \ra &\geq \la \wb^{(t)}, \wb^* \ra + \frac{ \eta \gamma^{(t)} }{2 \| \wb^{(t)} \|_{\bSigma}^3} \cdot \min_i |\ell_i'^{(t)}| \cdot  \|  \wb^{(t)} -  \la \wb^*, \wb^{(t)}\ra_{\bSigma} \cdot \wb^* \|_{\bSigma}^2\\
%     &\geq \la \wb^{(t)}, \wb^* \ra + \frac{ \eta \exp(-3n/2) }{8 \| \wb^{(t)} \|_{\bSigma}^3}  \cdot  \|  \wb^{(t)} -  \la \wb^*, \wb^{(t)}\ra_{\bSigma} \cdot \wb^* \|_{\bSigma}^2\\
%     &\geq 
% \end{align*}
% for all $t\geq 0$.
\end{proof}

\subsection{Proof of Lemma~\ref{lemma:linear_asymp}}

\begin{proof}[Proof of Lemma~\ref{lemma:linear_asymp}] We prove the first five results together by induction, and then prove the sixth result. Clearly, all the results hold at $t = t_0$ by the assumptions. Now suppose that there exists $t_1 \geq t_0$ such that the results hold for $t = t_0,\ldots, t_1$, i.e., for $t = t_0,\ldots, t_1$ it holds that
\begin{enumerate}[label=(\roman*)]
    % \item\label{induction_1} $ \big\| \wb^{(\tau + t_0)} - \la \wb^*, \wb^{(\tau + t_0 )} \ra \cdot \| \wb^* \|_2^{-1}  \cdot \wb^* \big\|_2 \leq 2\epsilon$.
    % \item\label{induction_1} $\big\| \wb^{(t_0)} - \la \wb^*, \wb^{(t_0 )} \ra \cdot \| \wb^* \|_2^{-2}  \cdot \wb^* \big\|_2,\ldots, \big\| \wb^{(t)} - \la \wb^*, \wb^{(t )} \ra \cdot \| \wb^* \|_2^{-2}  \cdot \wb^* \big\|_2$ is a decreasing sequence.
    % % \begin{align*}
    % %     \alpha^{-1}\cdot \log((\eta\alpha / \beta) \tau + A)\leq \gamma^{t_0 + \tau } \leq \alpha^{-1} \cdot \log(\eta\alpha\beta (\tau+1) + A),
    % % \end{align*}
    % % where $A = \exp(\alpha \gamma^{t_0})$.
    % \item\label{induction_2} $ \|\wb^{(0)}\|_2\leq \| \wb^{(t)}\|_2 \leq (1 + \epsilon)\cdot \|\wb^{(0)}\|_2$.
    %  \item\label{induction_3} $ \log((\eta / \beta) (t - t_0) + \exp( \gamma^{t_0}))\leq \gamma^{(t)} \leq  \log(\eta\beta (t - t_0+1) + \exp( \gamma^{t_0}))$. %,  where $A = \exp( \gamma^{t_0})$.
    % \item\label{induction_4} $n^{-2}\sum_{i,i'=1}^n (\la \wb^{(t)} , \zb_{i'}\ra - \la \wb^{(t)} , \zb_{i}\ra)^2 \leq \zeta / \log^2(t - t_0+2)$.
    % \item\label{induction_5} $\max_{i} |\la \wb^{(t)} , \zb_{i}\ra - \alpha | \cdot \gamma^{(t)} \leq \beta$.
        \item\label{induction_1} $\big\| \wb^{(t_0)} - \la \wb^*, \wb^{(t_0 )} \ra \cdot \| \wb^* \|_2^{-2}  \cdot \wb^* \big\|_2,\ldots, \big\| \wb^{(t)} - \la \wb^*, \wb^{(t )} \ra \cdot \| \wb^* \|_2^{-2}  \cdot \wb^* \big\|_2$ is a decreasing sequence.
    % \begin{align*}
    %     \alpha^{-1}\cdot \log((\eta\alpha / \beta) \tau + A)\leq \gamma^{t_0 + \tau } \leq \alpha^{-1} \cdot \log(\eta\alpha\beta (\tau+1) + A),
    % \end{align*}
    % where $A = \exp(\alpha \gamma^{t_0})$.
    \item\label{induction_2} $ \|\wb^{(0)}\|_2\leq \| \wb^{(t)}\|_2 \leq (1 + \epsilon )\cdot \|\wb^{(0)}\|_2$.
    \item\label{induction_3} $\gamma^{(t)}$ has the following upper and lower bounds:
    \begin{align*}
        &\gamma^{(t)} \leq  \log[ 8\eta\cdot (t - t_0) +2\exp(\gamma^{(t_0)}) ],\\
        &\gamma^{(t)} \geq \log[ (\eta / 8) \cdot (t - t_0) +\exp(\gamma^{(t_0)}) ].
    \end{align*}
    % \begin{align*}
    %     &\gamma^{(t)} \leq  2\eta\cdot \exp(3\beta / \alpha) \cdot \exp(-\gamma^{(t_0)}) + \log[ 2\eta\cdot \exp(3\beta / \alpha) \cdot (t - t_0) +\exp(\gamma^{(t_0)}) ],\\
    %     &\gamma^{(t)} \geq \log[ (\eta / 4)\cdot \exp(-\beta / \alpha) \cdot (t - t_0) +\exp(\gamma^{(t_0)}) ].
    % \end{align*}
    % $ \log[ (\eta / 4)\cdot \exp(\beta / \alpha) \cdot (t - t_0) +\exp(\gamma^{(t_0)}) ]\leq \gamma^{(t)} \leq  2\eta\cdot \exp(3\beta / \alpha) \cdot \exp(-\gamma^{(t_0)}) + \log[ 2\eta\cdot \exp(3\beta / \alpha) \cdot (t - t_0) +\exp(\gamma^{(t_0)}) ]$. %,  where $A = \exp( \gamma^{t_0})$.
    \item\label{induction_4} It holds that 
    \begin{align*}
        &\big\| \wb^{(t)} - \la \wb^*, \wb^{(t )} \ra \cdot \| \wb^* \|_2^{-2}  \cdot \wb^* \big\|_2\\
        &\qquad \qquad  \leq \epsilon\cdot \|\wb^{(0)} \|_2 \cdot \exp\Bigg[ -\frac{ \lambda_{\min}  }{1024\lambda_{\max}^{3/2} \cdot \| \wb^{(0)} \|_2^2} \cdot \log^2( (8/9) \eta\cdot (t - t_0)+ 1)\Bigg],\\
        &n^{-2}\sum_{i,i'=1}^n (\la \wb^{(t)} , \zb_{i'}\ra - \la \wb^{(t)} , \zb_{i}\ra)^2 \\
        &\qquad \qquad  \leq \lambda_{\max}\cdot \epsilon^2\cdot \|\wb^{(0)} \|_2^2 \cdot \exp\Bigg[ -\frac{ \lambda_{\min}  }{512\lambda_{\max}^{3/2} \cdot \| \wb^{(0)} \|_2^2} \cdot \log^2( (8/9) \eta\cdot (t - t_0)+ 1)\Bigg].
    \end{align*}
    \item\label{induction_5} $\max_{i} |\la \wb^{(t)}/\|  \wb^{(t)}\|_2 , \zb_{i}\ra - \| \wb^* \|_2^{-1} | \cdot \gamma^{(t)} \leq \| \wb^* \|_2^{-1} / 4$.
    % \iten\label{induction_6} It holds that
    % \begin{align*}
    %     aaa
    % \end{align*}
\end{enumerate}
Then we aim to show that the above conclusions also hold at iteration $t_1 + 1$. 

\noindent\textbf{Preliminary results based on the induction hypotheses.} By definition, it is easy to see that 
\begin{align*}
    \epsilon \leq \| \wb^* \|_2^{-1}\cdot (\max_i \| \zb_i \|_2)^{-1} / 24 \leq ( \min_i |\la \wb^* , \zb_i \ra|)^{-1} / 24 = 1/24 < 1.
\end{align*}
By induction hypothesis \ref{induction_1}, we have
\begin{align}\label{eq:induction_preliminary1}
    \big\| \wb^{(t)} - \la \wb^*, \wb^{(t )} \ra \cdot \| \wb^* \|_2^{-2}  \cdot \wb^* \big\|_2 \leq \big\| \wb^{(t_0)} - \la \wb^*, \wb^{(t_0 )} \ra \cdot \| \wb^* \|_2^{-2}  \cdot \wb^* \big\|_2 \leq \epsilon \cdot \|\wb^{(0)} \|_2.
\end{align}
Taking the square of both sides and dividing by $\| \wb^{(t)} \|_2^{2}$ gives
\begin{align*}
   1 - \la \wb^*/ \| \wb^* \|_2 , \wb^{(t )}/  \|\wb^{(t)} \|_2 \ra^2 \leq \epsilon^2\cdot \|\wb^{(0)} \|_2^2 /  \|\wb^{(t)} \|_2^2 \leq\epsilon^2 ,
\end{align*}
where the last inequality follows by Lemma~\ref{lemma:norm_upperbound} on the monotonicity of $\|\wb^{(t)} \|_2$. 
Therefore, we have
\begin{align}\label{eq:induction_preliminary2}
    (1 - \epsilon) \leq \sqrt{1 - \epsilon^2 } \leq \la \wb^*/ \| \wb^* \|_2 , \wb^{(t )}  /  \|\wb^{(t)} \|_2 \ra \leq 1.
\end{align}
Moreover, by \eqref{eq:induction_preliminary1}, we have
\begin{align*}
    | \la \wb^{(t)} , \zb_i \ra - \la \wb^*, \wb^{(t )} \ra \cdot \| \wb^* \|_2^{-2} \cdot \la \wb^* , \zb_i \ra | \leq  \max_i \| \zb_i \|_2\cdot \|\wb^{(0)} \|_2 \cdot \epsilon
\end{align*}
for all $i\in[n]$. Dividing by $\| \wb^{(t)} \|_2$ on both sides above gives
\begin{align*}
    &| \la \wb^{(t)} / \| \wb^{(t)} \|_2 , \zb_i \ra - \la \wb^* / \| \wb^* \|_2 , \wb^{(t )}  / \| \wb^{(t)} \|_2 \ra \cdot \la \wb^* / \| \wb^* \|_2 , \zb_i \ra |\\
    &\qquad \qquad \qquad \qquad \qquad \qquad \leq  \max_i \| \zb_i \|_2\cdot \|\wb^{(0)} \|_2  / \| \wb^{(t)} \|_2 \cdot \epsilon \leq  \max_i \| \zb_i \|_2\cdot \epsilon.
\end{align*}
Recall that $\wb^*$ is chosen such that $\la \wb^*, \zb_i\ra = 1$ for all $i\in[n]$. Therefore, rearranging terms and applying \eqref{eq:induction_preliminary2} then gives
\begin{align*}
    \la \wb^{(t)} / \| \wb^{(t)} \|_2  , \zb_i \ra &\geq \la \wb^* / \| \wb^* \|_2 , \wb^{(t )}  / \| \wb^{(t)} \|_2 \ra \cdot \la \wb^* / \| \wb^* \|_2 , \zb_i \ra -  \max_i \| \zb_i \|_2\cdot \epsilon\\
    &\geq   \| \wb^* \|_2^{-1} \cdot \la \wb^* , \zb_i \ra -2 \max_i \| \zb_i \|_2 \cdot \epsilon,\\
    &= \| \wb^* \|_2^{-1} -2 \max_i \| \zb_i \|_2\cdot \|\wb^{(0)} \|_2 \cdot \epsilon\\
     \la \wb^{(t)} / \| \wb^{(t)} \|_2  , \zb_i \ra &\leq \la \wb^* / \| \wb^* \|_2 , \wb^{(t )}  / \| \wb^{(t)} \|_2 \ra \cdot \la \wb^* / \| \wb^* \|_2 , \zb_i \ra + \max_i \| \zb_i \|_2\cdot \epsilon\\
    &\leq  \| \wb^* \|_2^{-1} \cdot \la \wb^* , \zb_i \ra + \max_i \| \zb_i \|_2\cdot  \epsilon\\
    &= \| \wb^* \|_2^{-1} + \max_i \| \zb_i \|_2\cdot \epsilon.
\end{align*}
Therefore, we have
\begin{align}\label{eq:innerproduct_stable}
    | \la \wb^{(t)} / \| \wb^{(t)} \|_2  , \zb_i \ra - \| \wb^* \|_2^{-1} | \leq 2\max_i \| \zb_i \|_2\cdot \epsilon \leq  \| \wb^* \|_2^{-1} / 4
\end{align}
for all $i\in[n]$ and $t = t_0,\ldots, t_1$, where the second inequality follows by the definition of $\epsilon$. 
Now denote $\alpha = \| \wb^* \|_2^{-1} $ and $c^{(t)} = \max_{i} |\la \wb^{(t)}  / \| \wb^{(t)} \|_2 , \zb_{i}\ra - \alpha | $. Then we have
% \begin{align*}
%     \frac{\la \wb^{(\tau + t_0)}, \xb_i \ra}{\sqrt{n^{-2} \cdot \sum_{i'=1}^n \la \wb^{(\tau + t_0)}, \xb_{i'}\ra^2 }  } \leq 
% \end{align*}
\begin{align}\label{eq:yf_upperbound}
    y_i\cdot f(\wb^{(t)},\gamma^{(t)},\xb_i ) &= \frac{\gamma^{(t)}\cdot\la \wb^{(t)}, \zb_i \ra}{\sqrt{n^{-1} \cdot \sum_{i'=1}^n \la \wb^{(t)}, \zb_{i'}\ra^2 }  } = \frac{\gamma^{(t)}\cdot\la \wb^{(t)} / \| \wb^{(t)} \|_2, \zb_i \ra}{\sqrt{n^{-1} \cdot \sum_{i'=1}^n \la \wb^{(t)} / \| \wb^{(t)} \|_2, \zb_{i'}\ra^2 }  } \nonumber \\
    &\leq \gamma^{(t)}\cdot\frac{ \alpha + c^{(t)} }{\alpha - c^{(t)}  }
    \leq  \gamma^{(t)}\cdot (1+  3c^{(t)} / \alpha) \leq \gamma^{(t)} + 3/4, 
\end{align}
where the second inequality follows by \eqref{eq:innerproduct_stable} that $\max_i  |\la \wb^{(t)} / \| \wb^{(t)} \|_2 , \zb_{i}\ra - \alpha| \leq \alpha / 4$, and the last inequality follows by induction hypothesis \ref{induction_5}.  
Similarly, 
\begin{align}\label{eq:yf_lowerbound}
    y_i\cdot f(\wb^{(t)},\gamma^{(t)},\xb_i ) &= \frac{\gamma^{(t)}\cdot\la \wb^{(t)} / \| \wb^{(t)} \|_2, \zb_i \ra}{\sqrt{n^{-2} \cdot \sum_{i'=1}^n \la \wb^{(t)} / \| \wb^{(t)} \|_2, \zb_{i'}\ra^2 }  } \nonumber \\
    &\geq \gamma^{(t)}\cdot\frac{ \alpha - c^{(\tau)} }{\alpha + c^{(\tau)}  }\geq  \gamma^{(t)}\cdot (1 - c^{(\tau)} / \alpha) \geq \gamma^{(t)} - 1/4. 
\end{align}
Note that
\begin{align*}
    -\ell'( y_i\cdot f(\wb^{(t)},\gamma^{(t)},\xb_i ) ) = \frac{1}{1 + \exp[ y_i\cdot f(\wb^{(t)},\gamma^{(t)},\xb_i )   ]}
\end{align*}
for all $i\in[n]$ and all $t \in [t_0, t_1]$. 
Therefore by \eqref{eq:yf_upperbound} and \eqref{eq:yf_lowerbound}, we have
\begin{align}\label{eq:l_derivative_bounds}
   \exp(- \gamma^{(t)} - 3/4) / 2 \leq -\ell'( y_i\cdot f(\wb^{(t)},\gamma^{(t)},\xb_i ) )\leq \exp(- \gamma^{(t)} + 3/4)
\end{align}
for all $i\in [n]$ and all $t \in [t_0, t_1]$. 

\noindent\textbf{Proof of induction hypothesis \ref{induction_1} at iteration $t_1 + 1$.} By Lemma~\ref{lemma:innerproduct_lowerbound}, for any $t=t_0,\ldots,t_1$, we have
\begin{align*}
    &-\frac{\la \wb^*, \wb^{(t)}\ra}{\| \wb^* \|_2} \cdot \la \wb^{(t+1)}, \wb^* \ra \\
    &\leq -\frac{\la \wb^*, \wb^{(t)}\ra}{\| \wb^* \|_2} \cdot \la \wb^{(t)}, \wb^* \ra - \frac{ \eta\cdot \gamma^{(t)} \cdot \la \wb^{(t)}, \wb^* \ra}{2 \| \wb^{(t)} \|_{\bSigma}^3 \cdot \| \wb^* \|_2} \cdot \min_i |\ell_i'^{(t)}| \cdot \|  \wb^{(t)} -  \la \wb^*, \wb^{(t)}\ra_{\bSigma} \cdot \wb^* \|_{\bSigma}^2\\
    &\leq -\frac{\la \wb^*, \wb^{(t)}\ra}{\| \wb^* \|_2} \cdot \la \wb^{(t)}, \wb^* \ra - \frac{ \eta\cdot \gamma^{(t)} \cdot \lambda_{\min} \cdot \la \wb^{(t)}, \wb^* \ra}{2 \| \wb^{(t)} \|_{\bSigma}^3 \cdot \| \wb^* \|_2} \cdot \min_i |\ell_i'^{(t)}| \cdot \|  \wb^{(t)} -  \la \wb^*, \wb^{(t)}\ra_{\bSigma} \cdot \wb^* \|_2^2\\
    &\leq -\frac{\la \wb^*, \wb^{(t)}\ra}{\| \wb^* \|_2} \cdot \la \wb^{(t)}, \wb^* \ra - \frac{ \eta\cdot \gamma^{(t)} \cdot \lambda_{\min} \cdot \la \wb^{(t)}, \wb^* \ra}{2 \| \wb^{(t)} \|_{\bSigma}^3 \cdot \| \wb^* \|_2} \cdot \min_i |\ell_i'^{(t)}| \cdot  \bigg\|  \wb^{(t)} -  \frac{\la \wb^*, \wb^{(t)}\ra}{\| \wb^* \|_2^2} \cdot \wb^* \bigg\|_2^2\\
    &\leq -\frac{\la \wb^*, \wb^{(t)}\ra}{\| \wb^* \|_2} \cdot \la \wb^{(t)}, \wb^* \ra - \frac{ \eta\cdot \gamma^{(t)} \cdot \lambda_{\min} \cdot \la \wb^{(t)}, \wb^* \ra}{2 \lambda_{\max}^{3/2} \cdot \| \wb^{(t)} \|_2^3 \cdot \| \wb^* \|_2} \cdot \min_i |\ell_i'^{(t)}| \cdot  \bigg\|  \wb^{(t)} -  \frac{\la \wb^*, \wb^{(t)}\ra}{\| \wb^* \|_2^2} \cdot \wb^* \bigg\|_2^2\\
    % &\leq -\frac{\la \wb^*, \wb^{(t)}\ra}{\| \wb^* \|_2} \cdot \la \wb^{(t)}, \wb^* \ra - \frac{ \eta\cdot \gamma^{(t)} \cdot \lambda_{\min} \cdot (1 - \epsilon) \|\wb^{(0)} \|_2}{4\lambda_{\max}^{3/2} \cdot \| \wb^{(t)} \|_2^3} \cdot \exp(- \gamma^{(t)} - \beta / \alpha) \cdot  \bigg\|  \wb^{(t)} -  \frac{\la \wb^*, \wb^{(t)}\ra}{\| \wb^* \|_2^2} \cdot \wb^* \bigg\|_2^2\\
    &\leq -\frac{\la \wb^*, \wb^{(t)}\ra}{\| \wb^* \|_2} \cdot \la \wb^{(t)}, \wb^* \ra - \frac{ \eta\cdot \gamma^{(t)} \cdot \lambda_{\min} }{8\lambda_{\max}^{3/2} \cdot \| \wb^{(t)} \|_2^2} \cdot \exp(- \gamma^{(t)} ) \cdot  \bigg\|  \wb^{(t)} -  \frac{\la \wb^*, \wb^{(t)}\ra}{\| \wb^* \|_2^2} \cdot \wb^* \bigg\|_2^2\\
    &\leq -\frac{\la \wb^*, \wb^{(t)}\ra}{\| \wb^* \|_2} \cdot \la \wb^{(t)}, \wb^* \ra - \frac{ \eta\cdot \gamma^{(t)} \cdot \lambda_{\min} }{16\lambda_{\max}^{3/2} \cdot \| \wb^{(0)} \|_2^2} \cdot \exp(- \gamma^{(t)} ) \cdot  \bigg\|  \wb^{(t)} -  \frac{\la \wb^*, \wb^{(t)}\ra}{\| \wb^* \|_2^2} \cdot \wb^* \bigg\|_2^2,
    % \\
    % &= -\frac{\la \wb^*, \wb^{(t)}\ra}{\| \wb^* \|_2} \cdot \la \wb^{(t)}, \wb^* \ra - \frac{ \eta\cdot \gamma^{(t)} \cdot \lambda_{\min} }{16\lambda_{\max}^{3/2} \cdot \| \wb^{(0)} \|_2^2} \cdot \exp(- \gamma^{(t)} - \beta / \alpha) \cdot  \bigg\|  \wb^{(t)} -  \frac{\la \wb^*, \wb^{(t)}\ra}{\| \wb^* \|_2^2} \cdot \wb^* \bigg\|_2^2,
\end{align*}
where the third inequality follows by the fact that $ \la \wb^*, \wb^{(t)}\ra\cdot \| \wb^* \|_2^{-2} \cdot \wb^*$ is the projection of $\wb^{(t)}$ on $\mathrm{span}\{\wb^*\}$ and $\|  \wb^{(t)} - \la \wb^*, \wb^{(t)}\ra\cdot \| \wb^* \|_2^{-2} \cdot \wb^* \|_2^2 \leq \|  \wb^{(t)} - c \cdot \wb^* \|_2^2$ for all $c\in\RR$, the fifth inequality follows by \eqref{eq:induction_preliminary2} and \eqref{eq:l_derivative_bounds}, and the last inequality follows by \eqref{eq:induction_preliminary2} and $(1+\epsilon)\leq \sqrt{2}$.
% \begin{align*}
%     &-\frac{\la \wb^*, \wb^{(t)}\ra}{\| \wb^* \|_2} \cdot \la \wb^{(t+1)}, \wb^* \ra \\
%     &\leq -\frac{\la \wb^*, \wb^{(t)}\ra}{\| \wb^* \|_2} \cdot \la \wb^{(t)}, \wb^* \ra - \frac{ \eta\cdot \gamma^{(t)} \cdot \la \wb^{(t)}, \wb^* \ra}{2 \| \wb^{(t)} \|_{\bSigma}^3} \cdot \min_i |\ell_i'^{(t)}| \cdot \|  \wb^{(t)} -  \la \wb^*, \wb^{(t)}\ra_{\bSigma} \cdot \wb^* \|_{\bSigma}^2\\
%     &\leq -\frac{\la \wb^*, \wb^{(t)}\ra}{\| \wb^* \|_2} \cdot \la \wb^{(t)}, \wb^* \ra - \frac{\eta\cdot \gamma^{(t)} \cdot \la \wb^{(t)}, \wb^* \ra}{2 \| \wb^{(t)} \|_{\bSigma}^3} \cdot \min_i |\ell_i'^{(t)}| \cdot  \bigg\|  \wb^{(t)} - \frac{\la \wb^*, \wb^{(t)}\ra}{\| \wb^* \|_2} \cdot \wb^* \bigg\|_{\bSigma}^2\\
%     &\leq -\frac{\la \wb^*, \wb^{(t)}\ra}{\| \wb^* \|_2} \cdot \la \wb^{(t)}, \wb^* \ra+ \frac{ \eta\cdot \gamma^{(t)} \cdot \lambda_{\max} \cdot \la \wb^{(t)}, \wb^* \ra}{2 \| \wb^{(t)} \|_{\bSigma}^3} \cdot \min_i |\ell_i'^{(t)}| \cdot  \bigg\|  \wb^{(t)} -  \frac{\la \wb^*, \wb^{(t)}\ra}{\| \wb^* \|_2} \cdot \wb^* \bigg\|_2^2\\
%     &\leq -\frac{\la \wb^*, \wb^{(t)}\ra}{\| \wb^* \|_2} \cdot \la \wb^{(t)}, \wb^* \ra+ \frac{ \eta\cdot \gamma^{(t)} \cdot \lambda_{\max} \cdot \| \wb^* \|_2}{2\lambda_{\min}^{3/2} \cdot \| \wb^{(t)} \|_2^3} \cdot \exp(- \gamma^{(\tau + t_0)} + \beta / \alpha) \cdot  \bigg\|  \wb^{(t)} -  \frac{\la \wb^*, \wb^{(t)}\ra}{\| \wb^* \|_2} \cdot \wb^* \bigg\|_2^2
% \end{align*}
Adding $\| \wb^{(t)} \|_2^2 + \la \wb^*, \wb^{(t)}\ra^{2} \cdot \| \wb^*\|_2^{-2} / 4$ to both sides above gives
\begin{align}
    &\bigg\|  \wb^{(t+1)} -  \frac{\la \wb^*, \wb^{(t)}\ra}{\| \wb^* \|_2^2} \cdot \wb^* \bigg\|_2^2 + \| \wb^{(t)} \|_2^2 - \| \wb^{(t+1)} \|_2^2\nonumber\\
    &\leq \bigg\|  \wb^{(t)} -  \frac{\la \wb^*, \wb^{(t)}\ra}{\| \wb^* \|_2^2} \cdot \wb^* \bigg\|_2^2 - \frac{ \eta\cdot \gamma^{(t)} \cdot \lambda_{\min}}{16\lambda_{\max}^{3/2} \cdot \| \wb^{(0)} \|_2^2} \cdot \exp(- \gamma^{(t)})  \cdot \bigg\|  \wb^{(t)} -  \frac{\la \wb^*, \wb^{(t)}\ra}{\| \wb^* \|_2^2} \cdot \wb^* \bigg\|_2^2.\label{eq:monotone_proof_eq1}
\end{align}
Now by Lemma~\ref{lemma:norm_upperbound}, we have
\begin{align*}
    &\|\wb^{(t+1)}\|_2^2
    - \| \wb^{(t)}\|_2^2 \\
    &\qquad \leq \eta^2  G\cdot \max\{\gamma^{(t)2}, \gamma^{(t)4}\}\cdot \max\{ |\ell_1'^{(t)}|^2, \ldots, |\ell_n'^{(t)}|^2, \exp(-2\gamma^{(t)}) \}\cdot  \bigg\|  \wb^{(t)} -  \frac{\la \wb^*, \wb^{(t)}\ra}{\| \wb^* \|_2^2} \cdot \wb^* \bigg\|_2^2,
\end{align*}
where $G =  64  \lambda_{\min}^{-3}\cdot \max_i \|\xb_{i}\|_2^6 \cdot \| \wb^{(0)} \|_2^{-4} $. Then according to \eqref{eq:l_derivative_bounds}, we have
\begin{align*}
    \|\wb^{(t+1)}\|_2^2
    &\leq \| \wb^{(t)}\|_2^2 + \eta^2  G\cdot \max\{\gamma^{(t)2}, \gamma^{(t)4}\}\cdot  \exp(-2\gamma^{(t)}) \cdot \exp(3/2 ) \cdot \bigg\|  \wb^{(t)} -  \frac{\la \wb^*, \wb^{(t)}\ra}{\| \wb^* \|_2^2} \cdot \wb^* \bigg\|_2^2,
\end{align*}
Note that we have $0 < \max\{z, z^3\}\cdot  \exp(-z) < 2$ for all $z >0$. Therefore we have \begin{align}
    \|\wb^{(t+1)}\|_2^2
    &\leq \| \wb^{(t)}\|_2^2 + 2 \eta^2  G\cdot  \exp(-\gamma^{(t)}) \cdot \exp(3/2 ) \cdot \bigg\|  \wb^{(t)} -  \frac{\la \wb^*, \wb^{(t)}\ra}{\| \wb^* \|_2^2} \cdot \wb^* \bigg\|_2^2 \nonumber\\
    &\leq \frac{ \eta\cdot \gamma^{(t)} \cdot \lambda_{\min}}{32\lambda_{\max}^{3/2} \cdot \| \wb^{(0)} \|_2^2} \cdot \exp(- \gamma^{(t )} )  \cdot  \bigg\|  \wb^{(t)} -  \frac{\la \wb^*, \wb^{(t)}\ra}{\| \wb^* \|_2^2} \cdot \wb^* \bigg\|_2^2,\label{eq:monotone_proof_eq2}
\end{align}
where the last inequality follows by the assumption that $\eta \leq G^{-1} \cdot \exp(-3/2 )\cdot \frac{ \lambda_{\min} }{64\lambda_{\max}^{3/2} \cdot \| \wb^{(0)} \|_2^2} $. 
% Then conclusion \ref{induction_1} holds at $\tau = \tau_0 + 1$ by the assumption on $\eta$.
% By the gradient descent update formula, we have
% \begin{align*}
%     \wb^{(t+1)} = \wb^{(t)} -  
% \end{align*}
Moreover, note again that $ \frac{\la \wb^*, \wb^{(t+1)}\ra}{\| \wb^* \|_2^2} \cdot \wb^*$ is the projection of $\wb^{(t+1)}$ onto the subspace $\mathrm{span}\{\wb^*\}$, which implies that
% \begin{align*}
%     \|\wb^{(t+1)}\|_2^2
%     &\leq \| \wb^{(t)}\|_2^2 + \eta^2  G\cdot\gamma^{(t)} \cdot \exp(-\gamma^{(t)}) \cdot \exp(6\beta / \alpha ) \cdot \bigg\|  \wb^{(t)} -  \frac{\la \wb^*, \wb^{(t)}\ra}{\| \wb^* \|_2^2} \cdot \wb^* \bigg\|_2^2\\
%     &= \| \wb^{(t)}\|_2^2 + \eta^2  G\cdot\gamma^{(t)} \cdot \exp(-\gamma^{(t)} - \beta / \alpha) \cdot \exp(7\beta / \alpha ) \cdot \bigg\|  \wb^{(t)} -  \frac{\la \wb^*, \wb^{(t)}\ra}{\| \wb^* \|_2^2} \cdot \wb^* \bigg\|_2^2
% \end{align*}
\begin{align}\label{eq:monotone_proof_eq3}
    \bigg\|  \wb^{(t+1)} -  \frac{\la \wb^*, \wb^{(t+1)}\ra}{\| \wb^* \|_2^2} \cdot \wb^* \bigg\|_2^2\leq \bigg\|  \wb^{(t+1)} - \frac{\la \wb^*, \wb^{(t)}\ra}{\| \wb^* \|_2^2} \cdot \wb^* \bigg\|_2^2.
\end{align}
Plugging \eqref{eq:monotone_proof_eq2} and \eqref{eq:monotone_proof_eq3} into \eqref{eq:monotone_proof_eq1} gives
\begin{align}
    &\bigg\|  \wb^{(t+1)} -  \frac{\la \wb^*, \wb^{(t+1)}\ra}{\| \wb^* \|_2^2} \cdot \wb^* \bigg\|_2^2\nonumber \\
    &\qquad\qquad\leq \Bigg[1 - \frac{ \eta\cdot \gamma^{(t)} \cdot \lambda_{\min}}{32\lambda_{\max}^{3/2} \cdot \| \wb^{(0)} \|_2^2} \cdot \exp(- \gamma^{(t)} ) \Bigg] \cdot  \bigg\|  \wb^{(t)} -  \frac{\la \wb^*, \wb^{(t)}\ra}{\| \wb^* \|_2^2} \cdot \wb^* \bigg\|_2^2 \label{eq:distance_superlinear_eq1} 
    % \\
    % &\qquad\qquad\leq \bigg\|  \wb^{(t)} -  \frac{\la \wb^*, \wb^{(t)}\ra}{\| \wb^* \|_2^2} \cdot \wb^* \bigg\|_2^2\label{eq:distance_monotonicity}
\end{align}
for all $t \in [t_0, t_1]$. This implies that 
\begin{align*}
    \bigg\|  \wb^{(t+1)} -  \frac{\la \wb^*, \wb^{(t+1)}\ra}{\| \wb^* \|_2^2} \cdot \wb^* \bigg\|_2^2 \leq \bigg\|  \wb^{(t)} -  \frac{\la \wb^*, \wb^{(t)}\ra}{\| \wb^* \|_2^2} \cdot \wb^* \bigg\|_2^2
\end{align*}
for all $t \in [t_0, t_1]$, which completes the proof of induction hypothesis \ref{induction_1} at iteration $t_1 + 1$.
% for all $t \in [t_0, t_1]$, which completes the proof of induction hypothesis \ref{induction_1} at iteration $t_1 + 1$.

\noindent\textbf{Proof of induction hypothesis \ref{induction_2} at iteration $t_1 + 1$.} By Lemma~\ref{lemma:norm_upperbound} and \eqref{eq:l_derivative_bounds}, for any $t=t_0,\ldots,t_1$, we have
\begin{align*}
    \|\wb^{(t+1)}\|_2^2
    &\leq \| \wb^{(t)}\|_2^2 + \eta^2  G\cdot \max\{\gamma^{(t)2}, \gamma^{(t)4}\}\cdot \exp(- 2\gamma^{(t)} + 3/2)\cdot  \bigg\|  \wb^{(t)} -  \frac{\la \wb^*, \wb^{(t)}\ra}{\| \wb^* \|_2^2} \cdot \wb^* \bigg\|_2^2\\
    &\leq \| \wb^{(t)}\|_2^2 + 400 \eta^2  G\cdot \exp(- 1.5\gamma^{(t)})\cdot \bigg\|  \wb^{(t)} -  \frac{\la \wb^*, \wb^{(t)}\ra}{\| \wb^* \|_2^2} \cdot \wb^* \bigg\|_2^2\\
    &\leq \| \wb^{(t)}\|_2^2 + 400 \eta^2  G\cdot \exp(- 1.5\gamma^{(t)})\cdot  \bigg\|  \wb^{(t_0)} -  \frac{\la \wb^*, \wb^{(t_0)}\ra}{\| \wb^* \|_2^2} \cdot \wb^* \bigg\|_2^2\\
    &= \| \wb^{(t)}\|_2^2 + \frac{ 400 \eta^2  G}{\exp( 1.5\gamma^{(t)})}\cdot  \bigg\| \frac{\wb^{(t_0)}}{\| \wb^{(t_0)} \|_2}  -  \bigg\la \frac{\wb^*}{\| \wb^* \|_2}, \frac{\wb^{(t_0)}}{\| \wb^{(t_0)} \|_2}\bigg\ra \cdot \frac{\wb^*}{\| \wb^* \|_2} \bigg\|_2^2 \cdot \| \wb^{(t_0)} \|_2^2\\
    &= \| \wb^{(t)}\|_2^2 + \frac{ 400 \eta^2  G}{\exp( 1.5\gamma^{(t)})}\cdot   \bigg( \bigg\| \frac{\wb^{(t_0)}}{\| \wb^{(t_0)} \|_2} \bigg\|_2^2 - \bigg\la \frac{\wb^*}{\| \wb^* \|_2}, \frac{\wb^{(t_0)}}{\| \wb^{(t_0)} \|_2}\bigg\ra^2 \bigg) \cdot \| \wb^{(t_0)} \|_2^2\\
    &\leq \| \wb^{(t)}\|_2^2 + \frac{ 400 \eta^2  G}{\exp( 1.5\gamma^{(t)})}\cdot \| \wb^{(t_0)} \|_2^2\\
    &\leq \| \wb^{(t)}\|_2^2 + \frac{ 800 \eta^2  G}{\exp( 1.5\gamma^{(t)})}\cdot \| \wb^{(0)} \|_2^2,
\end{align*}
where  $G =  64  \lambda_{\min}^{-3}\cdot \max_i \|\xb_{i}\|_2^6 \cdot \| \wb^{(0)} \|_2^{-4} $ as defined in Lemma~\ref{lemma:norm_upperbound}, and the second inequality follows by the fact that $0 < \max\{z^2, z^4\}\cdot  \exp(-0.5z) < 80$ for all $z >0$, and the third inequality follows by induction hypothesis \ref{induction_1}. 
Therefore, by induction hypothesis \ref{induction_3}, we have
\begin{align}
    \|\wb^{(t_1+1)}\|_2^2 &\leq  \| \wb^{(0)} \|_2^2 + 800 \eta^2  G\cdot \| \wb^{(0)} \|_2^2 \cdot \sum_{t= t_0}^{t_1} \frac{1}{\exp( 1.5\gamma^{(t)})} \nonumber\\
    &\leq  \| \wb^{(0)} \|_2^2 + 800 \eta^2  G\cdot \| \wb^{(0)} \|_2^2 \cdot \sum_{t= t_0}^{t_1} \frac{1}{\exp\{ 1.5\log[ (\eta / 8)  \cdot (t - t_0) +\exp(\gamma^{(t_0)}) ] \}} \nonumber\\
    &=  \| \wb^{(0)} \|_2^2 +800 \eta^2  G\cdot \| \wb^{(0)} \|_2^2 \cdot \sum_{t= t_0}^{t_1} \frac{1}{[(\eta / 8)  \cdot (t - t_0) +\exp(\gamma^{(t_0)}) ]^{1.5}}.\label{eq:norm_induction_proof_eq1}
\end{align}
Moreover, we have
\begin{align*}
    \sum_{t= t_0}^{t_1} \frac{1}{[ (\eta / 8)  \cdot (t - t_0) +\exp(\gamma^{(t_0)}) ]^{1.5}} &= \sum_{t= 0}^{t_1 - t_0} \frac{1}{[ (\eta / 8)  \cdot t +\exp(\gamma^{(t_0)}) ]^{1.5}} \\
    &\leq \sum_{t= 0}^{\infty} \frac{1}{[ (\eta / 8)  \cdot t +\exp(\gamma^{(t_0)}) ]^{1.5}}\\
    &= \frac{1}{\exp(1.5 \gamma^{(t_0)})} + \sum_{t= 1}^{\infty} \frac{1}{[ (\eta / 8)  \cdot t +\exp(\gamma^{(t_0)}) ]^{1.5}}\\
    &\leq \frac{1}{\exp(1.5 \gamma^{(t_0)})} + \int_1^{\infty} \frac{1}{[ (\eta / 8)  \cdot t +\exp(\gamma^{(t_0)}) ]^{1.5}} \mathrm{d} t\\
    &= \frac{1}{\exp(1.5 \gamma^{(t_0)})} + \frac{2}{  (\eta / 8) \cdot \exp( \gamma^{(t_0)} / 2)}\\
    &\leq 1 + 16\eta^{-1}.
\end{align*}
Plugging the bound above into \eqref{eq:norm_induction_proof_eq1} gives
\begin{align*}
     |\wb^{(t_1+1)}\|_2^2 &\leq \| \wb^{(0)} \|_2^2 + 800 \eta^2  G  \cdot \| \wb^{(0)} \|_2^2 \cdot (1 + 16 \eta^{-1})\\
     &\leq \| \wb^{(0)} \|_2^2 + 16000\eta  G\cdot  \| \wb^{(0)} \|_2^2\\
     &\leq ( 1 + \epsilon^2 )\cdot \| \wb^{(0)} \|_2^2,
\end{align*}
where the second inequality follows by the definition of $H_1$ and the assumption that $\eta \leq 1$, and the last inequality follows by the assumption that $\eta \leq G^{-1}  \cdot \epsilon^2 / 16000$. Therefore we have
\begin{align*}
    |\wb^{(t_1+1)}\|_2 \leq \sqrt{1 + \epsilon^2}\cdot \| \wb^{(0)} \|_2 \leq ( 1 + \epsilon)\cdot \| \wb^{(0)} \|_2,
\end{align*}
which finishes the proof of induction hypothesis \ref{induction_2} at iteration $t_1 + 1$. 

\noindent\textbf{Proof of induction hypothesis \ref{induction_3} at iteration $t_1 + 1$.} The gradient descent update rule for $\gamma^{(t)}$ gives
\begin{align*}
    \gamma^{(t+1)} &= \gamma^{(t)} - \eta\cdot \frac{1}{n}\sum_{i=1}^{n} \ell'( y_i\cdot f(\wb^{(t)},\gamma^{(t)},\xb_i ) ) \cdot \frac{\la \wb^{(t)}, \zb_i \ra}{\| \wb^{(t)} \|_{\bSigma} }\\
    &= \gamma^{(t)} - \eta\cdot \frac{1}{n}\sum_{i=1}^{n} \ell'( y_i\cdot f(\wb^{(t)},\gamma^{(t)},\xb_i ) ) \cdot \frac{\la \wb^{(t)}, \zb_i \ra}{\sqrt{\frac{1}{n}\sum_{j=1}^n \la \wb^{(t)}, \zb_j \ra^2} }.
\end{align*}
By \eqref{eq:innerproduct_stable} and \eqref{eq:l_derivative_bounds}, we then have
\begin{align}
    \gamma^{(t+1)} &\leq \gamma^{(t)} - \eta\cdot \frac{1}{n}\sum_{i=1}^{n} \ell'( y_i\cdot f(\wb^{(t)},\gamma^{(t)},\xb_i ) ) \cdot \frac{\alpha + \alpha / 4}{\sqrt{\frac{1}{n}\sum_{j=1}^n (\alpha - \alpha / 4)^2} }\nonumber\\
     &\leq \gamma^{(t)} + \eta\cdot  \exp(- \gamma^{(t)} + 3/4) \cdot \frac{\alpha + \alpha / 4}{\sqrt{\frac{1}{n}\sum_{j=1}^n (\alpha - \alpha / 4)^2} }\nonumber\\
     &\leq \gamma^{(t)} + 4\eta\cdot \exp(- \gamma^{(t)}),\label{eq:gamma_update_upperbound}\\
    \gamma^{(t+1)} &\geq \gamma^{(t)} - \eta\cdot \frac{1}{n}\sum_{i=1}^{n} \ell'( y_i\cdot f(\wb^{(t)},\gamma^{(t)},\xb_i ) ) \cdot \frac{\alpha - \alpha / 4}{\sqrt{\frac{1}{n}\sum_{j=1}^n (\alpha + \alpha / 4)^2} }\nonumber\\
    &\geq \gamma^{(t)} + \eta\cdot \exp(- \gamma^{(t)} - 3/4 ) / 2 \cdot \frac{\alpha - \alpha / 4}{\sqrt{\frac{1}{n}\sum_{j=1}^n (\alpha + \alpha / 4)^2} }\nonumber\\
    &\geq \gamma^{(t)} + \frac{\eta}{8}\cdot \exp(- \gamma^{(t)}) \label{eq:gamma_update_lowerbound}
\end{align}
for all $t=t_0,\ldots,t_1$. The comparison theorem for discrete dynamical systems then gives  $\gamma^{(t_1 + 1)} \leq \overline{\gamma}^{(t_1 + 1)}$, where $\overline{\gamma}^{(t)}$ is given by the iterative formula
\begin{align*}
    \overline\gamma^{(t+1)} =\overline\gamma^{(t)} + 4\eta\cdot \exp(- \overline\gamma^{(t)}),\quad \overline\gamma^{(0)} = \gamma^{(0)}.
\end{align*}
Applying Lemma~\ref{lemma:gamma_comparison} then gives 
\begin{align*}
    \gamma^{(t_1+1)} \leq \overline\gamma^{(t_1+1)} \leq  4\eta \cdot \exp(-\gamma^{(t_0)}) + \log[ 4\eta  \cdot (t_1 + 1 - t_0) +\exp(\gamma^{(t_0)}) ].
\end{align*}
By the assumption that $\eta \leq 1/8\leq \log(2) / 4$,
 we have
\begin{align*}
    \gamma^{(t_1+1)} \leq  \log[ 8\eta\cdot (t_1 + 1 - t_0) +2\exp(\gamma^{(t_0)}) ].
\end{align*}
Similarly, by Lemma~\ref{lemma:gamma_comparison} and \eqref{eq:gamma_update_lowerbound}, we also have
\begin{align*}
    \gamma^{(t_1+1)} \geq  \log[ (\eta / 8)\cdot (t_1 + 1 - t_0) +\exp(\gamma^{(t_0)}) ].
\end{align*}
This finishes the proof of induction hypothesis \ref{induction_3} at iteration $t_1 + 1$.

% Taking a telescoping sum and plugging the induction hypothesis \ref{induction_3} into \eqref{eq:gamma_update_upperbound} then gives
% \begin{align*}
%     \gamma^{(t_1 + 1)} \geq \sum_{t=t_0}^{t_1} \frac{\eta}{4}\cdot \exp(- \beta / \alpha) \cdot
% \end{align*}

\noindent\textbf{Proof of induction hypothesis \ref{induction_4} at iteration $t_1 + 1$.} By \eqref{eq:distance_superlinear_eq1}, we have
\begin{align*}
    &\bigg\|  \wb^{(t_1+1)} -  \frac{\la \wb^*, \wb^{(t_1+1)}\ra}{\| \wb^* \|_2^2} \cdot \wb^* \bigg\|_2^2 \\
    &\qquad\leq \Bigg[1 - \frac{ \eta\cdot \gamma^{(t_1)} \cdot \lambda_{\min}}{32\lambda_{\max}^{3/2} \cdot \| \wb^{(0)} \|_2^2} \cdot \exp(- \gamma^{(t_1)} ) \Bigg] \cdot  \bigg\|  \wb^{(t_1)} -  \frac{\la \wb^*, \wb^{(t_1)}\ra}{\| \wb^* \|_2^2} \cdot \wb^* \bigg\|_2^2 \\
    &\qquad\leq \prod_{\tau = t_0}^{t_1}\Bigg[1 - \frac{ \eta\cdot \gamma^{(\tau)} \cdot \lambda_{\min}}{32\lambda_{\max}^{3/2} \cdot \| \wb^{(0)} \|_2^2} \cdot \exp(- \gamma^{(\tau)} ) \Bigg] \cdot  \bigg\|  \wb^{(t_0)} -  \frac{\la \wb^*, \wb^{(t_0)}\ra}{\| \wb^* \|_2^2} \cdot \wb^* \bigg\|_2^2\\
    &\qquad\leq \prod_{\tau = t_0}^{t_1}\Bigg[1 - \frac{ \eta\cdot \gamma^{(\tau)} \cdot \lambda_{\min}}{32\lambda_{\max}^{3/2} \cdot \| \wb^{(0)} \|_2^2} \cdot \exp(- \gamma^{(\tau)} ) \Bigg] \cdot  \epsilon^2\cdot \|\wb^{(0)} \|_2^2
\end{align*}
% Denote
% \begin{align*}
%     A^{(t)} = \bigg\|  \wb^{(t)} -  \frac{\la \wb^*, \wb^{(t)}\ra}{\| \wb^* \|_2^2} \cdot \wb^* \bigg\|_2
% \end{align*}
% for all $t\geq 0$. 
Then by the fact that $1 - x \leq \exp(-x)$ for all $x \in \RR$, we have
\begin{align}
    \bigg\|  \wb^{(t_1+1)} -  \frac{\la \wb^*, \wb^{(t_1+1)}\ra}{\| \wb^* \|_2^2} \cdot \wb^* \bigg\|_2^2 
    &\leq \prod_{\tau = t_0}^{t_1}\exp\Bigg[ - \frac{ \eta\cdot \gamma^{(\tau)} \cdot \lambda_{\min}}{64\lambda_{\max}^{3/2} \cdot \| \wb^{(0)} \|_2^2} \cdot \exp(- \gamma^{(\tau)} ) \Bigg] \cdot \epsilon \|\wb^{(0)} \|_2 \nonumber\\
    & = \exp\Bigg[ - \sum_{\tau = t_0}^{t_1} \frac{ \eta\cdot \gamma^{(\tau)} \cdot \lambda_{\min}}{64\lambda_{\max}^{3/2} \cdot \| \wb^{(0)} \|_2^2} \cdot \exp(- \gamma^{(\tau)}) \Bigg] \cdot \epsilon \|\wb^{(0)} \|_2.\label{eq:l2distance_convergence_proof_intermediate}
\end{align}
We then study the term $\gamma^{(\tau)}\cdot \exp(-\gamma^{(\tau)})$. 
Note that by induction hypothesis \ref{induction_3}, for all $\tau = t_0,\ldots,t_1$, we have
\begin{align*}
    &\gamma^{(\tau)} \geq \log[ (\eta / 8)  \cdot (\tau - t_0) +\exp(\gamma^{(t_0)}) ] \geq \gamma^{(t_0)} \geq 1/2, \\
    & \gamma^{(\tau)} \leq \log[ 8\eta\cdot (t - t_0) +2\exp(\gamma^{(t_0)}) ] \leq \log[ 8\eta\cdot (t - t_0) +2\exp(1.5) ] \leq \log[ 8 \eta\cdot (\tau - t_0) + 9 ].
\end{align*}
Over the interval $[1/2, \log[ H_0 \eta\cdot (\tau - t_0) + 9 ] ]$, the function $g(z) = z\cdot \exp(-z)$ is strictly increasing for $1/2\leq z \leq 1$ and strictly  decreasing for $ 1 \leq z \leq \log[ H_0 \eta\cdot (\tau - t_0) + 9 ]$. Note that $\log[ H_0 \eta\cdot (\tau - t_0) + 9 ] > \log(9) > 2$, and $g\{ \log[ H_0 \eta\cdot (\tau - t_0) + 9 ] \} < g(2) < g(1/2)$. Therefore we have
\begin{align*}
    \gamma^{(\tau)}\cdot \exp(-\gamma^{(\tau)})&\geq \min_{1/2 \leq z \leq \log[8 \eta\cdot (\tau - t_0) + 9 ]  } g(z)\\ 
    &\geq \log[8 \eta\cdot (\tau - t_0) + 9 ]\cdot \exp\{- \log[ 8 \eta\cdot (\tau - t_0) + 9 ]\}\\
    &= \frac{\log[ 8 \eta\cdot (\tau - t_0) + 9 ] }{8 \eta\cdot (\tau - t_0) + 9}.
\end{align*}
Plugging the bound above into \eqref{eq:l2distance_convergence_proof_intermediate} gives
\begin{align*}
    &\bigg\|  \wb^{(t_1+1)} -  \frac{\la \wb^*, \wb^{(t_1+1)}\ra}{\| \wb^* \|_2^2} \cdot \wb^* \bigg\|_2^2 \\
    &\quad \leq \exp\Bigg[ - \frac{ \eta\cdot  \lambda_{\min} }{128\lambda_{\max}^{3/2} \cdot \| \wb^{(0)} \|_2^2} \cdot  \sum_{\tau = t_0}^{t_1} \frac{\log[ 8 \eta\cdot (\tau - t_0) + 9 ] }{8 \eta\cdot (\tau - t_0) + 9} \Bigg] \cdot \epsilon \cdot \|\wb^{(0)} \|_2\\
    %     &\quad = \exp\Bigg[ - \sum_{\tau = t_0}^{t_1} \frac{ \eta\cdot \gamma^{(\tau)} \cdot \lambda_{\min}}{64\lambda_{\max}^{3/2} \cdot \| \wb^{(0)} \|_2^2} \cdot \exp(- \gamma^{(\tau)} - \beta / \alpha) \Bigg] \cdot \epsilon\cdot \|\wb^{(0)} \|_2\\
    % &\quad \leq \exp\Bigg[ - \sum_{\tau = t_0}^{t_1} \frac{ \eta\cdot \log[ (\eta / 4)\cdot \exp(-\beta / \alpha) \cdot (\tau - t_0) +\exp(\gamma^{(t_0)}) ] \cdot \lambda_{\min} \cdot \exp(-\beta / \alpha)}{64\lambda_{\max}^{3/2} \cdot \| \wb^{(0)} \|_2^2\cdot [ (\eta / 4)\cdot \exp(-\beta / \alpha) \cdot (\tau - t_0) +\exp(\gamma^{(t_0)})]} \Bigg]\cdot \epsilon\cdot \|\wb^{(0)} \|_2\\
    % % \exp(- \log[ (\eta / 4)\cdot \exp(\beta / \alpha) \cdot (t - t_0) +\exp(\gamma^{(\tau)}) ] - \beta / \alpha) \Bigg] \cdot  \bigg\|  \wb^{(t_0)} -  \frac{\la \wb^*, \wb^{(t_0)}\ra}{\| \wb^* \|_2^2} \cdot \wb^* \bigg\|_2^2\\
    % &\quad = \exp\Bigg[ -\frac{\eta\cdot \lambda_{\min} \cdot \exp(-\beta / \alpha)}{64\lambda_{\max}^{3/2} \cdot \| \wb^{(0)} \|_2^2} \cdot \sum_{\tau = t_0}^{t_1} \frac{ \log[ (\eta / 4)\cdot \exp(-\beta / \alpha) \cdot (\tau - t_0) +\exp(\gamma^{(t_0)}) ] }{ (\eta / 4)\cdot \exp(-\beta / \alpha) \cdot (\tau - t_0) +\exp(\gamma^{(t_0)})} \Bigg]\cdot \epsilon\cdot \|\wb^{(0)} \|_2\\
    &\quad \leq \exp\Bigg[ -\frac{\eta\cdot \lambda_{\min} }{128\lambda_{\max}^{3/2} \cdot \| \wb^{(0)} \|_2^2} \cdot \int_{0}^{t_1+1 - t_0}  \frac{\log( 8 \eta\cdot \tilde\tau + 9 ) }{ 8 \eta\cdot \tilde\tau + 9} \mathrm{d} \tilde\tau\Bigg]\cdot \epsilon\cdot \|\wb^{(0)} \|_2,
    % \\
    % &\quad = \exp\Bigg[ -\frac{\eta\cdot \lambda_{\min}}{128\lambda_{\max}^{3/2} \cdot \| \wb^{(0)} \|_2^2} \cdot \int_{0}^{t_1+1 - t_0}  \frac{\log( H_0 \eta\cdot \tilde\tau + 9 ) }{H_0 \eta\cdot \tilde\tau + 9} \mathrm{d} \tilde\tau\Bigg]\cdot \epsilon\cdot \|\wb^{(0)} \|_2
    % &\quad = \exp\Bigg[ -\frac{\eta\cdot \lambda_{\min} \cdot \exp(-\beta / \alpha)}{64\lambda_{\max}^{3/2} \cdot \| \wb^{(0)} \|_2^2} \cdot \int_{0}^{t_1+1 - t_0}  \frac{\log(  4\exp(3\beta / \alpha) \eta\cdot \tilde\tau + 9 ) }{ 4\exp(3\beta / \alpha) \eta\cdot \tilde\tau + 9} \mathrm{d} \tilde\tau\Bigg]\cdot \epsilon\cdot \|\wb^{(0)} \|_2
    % &\quad \leq \exp\Bigg[ -\frac{\eta\cdot \lambda_{\min} \cdot \exp(-\beta / \alpha)}{64\lambda_{\max}^{3/2}} \cdot \int_{0}^{t_1+1 - t_0} \frac{ \log[ (\eta / 4)\cdot \exp(-\beta / \alpha) \cdot \tilde\tau  +\exp(\gamma^{(t_0)}) ] }{ (\eta / 4)\cdot \exp(-\beta / \alpha) \cdot \tilde\tau  +\exp(\gamma^{(t_0)})} \mathrm{d} \tilde\tau\Bigg]\cdot \epsilon\cdot \|\wb^{(0)} \|_2,
\end{align*}
% , and the last inequality follows by the assumption that $\| \wb^{(0)} \|_2 \leq 1$
where the second inequality follows by $8 \eta\cdot \tilde\tau + 9 > e$, 
and $\log(z)/z$ is monotonically decreasing for $z > e$. Further calculating the integral, we obtain
\begin{align}
     &\bigg\|  \wb^{(t_1+1)} -  \frac{\la \wb^*, \wb^{(t_1+1)}\ra}{\| \wb^* \|_2^2} \cdot \wb^* \bigg\|_2\nonumber \\
     &\qquad \leq \exp\Bigg[ -\frac{ \lambda_{\min}  }{1024\lambda_{\max}^{3/2} \cdot \| \wb^{(0)} \|_2^2} \cdot \log^2(Q)\Big|_{Q = 9 }^{ 8 \eta\cdot (t_1 + 1 - t_0)+ 9 }\Bigg]\cdot\epsilon\cdot \|\wb^{(0)} \|_2 \nonumber\\
     &\qquad = \exp\Bigg[ -\frac{ \lambda_{\min}  }{1024\lambda_{\max}^{3/2} \cdot \| \wb^{(0)} \|_2^2} \cdot [\log^2(8 \eta\cdot (t_1 + 1 - t_0)+ 9) - \log^2(9)]\Bigg]\cdot\epsilon\cdot \|\wb^{(0)} \|_2 \nonumber\\
     &\qquad \leq \exp\Bigg[ -\frac{ \lambda_{\min}  }{1024\lambda_{\max}^{3/2} \cdot \| \wb^{(0)} \|_2^2} \cdot [\log(8 \eta\cdot (t_1 + 1 - t_0)+ 9) - \log(9)]^2\Bigg]\cdot\epsilon\cdot \|\wb^{(0)} \|_2 \nonumber\\
     &\qquad = \exp\Bigg[ -\frac{ \lambda_{\min}  }{1024\lambda_{\max}^{3/2} \cdot \| \wb^{(0)} \|_2^2} \cdot \log^2( (8/9) \eta\cdot (t_1 + 1 - t_0)+ 1)\Bigg]\cdot\epsilon\cdot \|\wb^{(0)} \|_2, \label{eq:induction_l2distance} 
    %  \nonumber\\
    %  &\qquad \leq \exp\Bigg[ -\frac{H_2 }{ \| \wb^{(0)} \|_2^2} \cdot \log^2(  (\eta/2)\cdot (t_1 + 1 - t_0)+ 1)\Bigg]\cdot\epsilon\cdot \|\wb^{(0)} \|_2, \label{eq:induction_l2distance}
    %  \nonumber\\
    %  &\qquad\leq H_3\cdot \|\wb^{(0)} \|_2 \cdot  \exp\{-H_2\cdot  \log^2[H_0\eta \cdot (t_1+1-t_0) + 9] \},\label{eq:induction_l2distance}
    %  \exp\Bigg[\frac{ \lambda_{\min} \cdot \exp(-2\beta / \alpha) \cdot \gamma^{(t_0)2}}{8\lambda_{\max}^{3/2} \cdot \| \wb^{(0)} \|_2^2}\Bigg] \cdot \exp \Bigg[ -\frac{ \lambda_{\min} \cdot \exp(-2\beta / \alpha)}{8\lambda_{\max}^{3/2} \cdot \| \wb^{(0)} \|_2^2} \cdot \log^2[(\eta/4)\cdot \exp(\beta / \alpha)\cdot (t_1+1-t_0) + \exp(\gamma^{(t_0)})]\Bigg]\cdot A^{(t_0)}
\end{align}
where the second inequality follows by the fact that $a^2 - b^2 \geq (a-b)^2$ for $a>b>0$. This finishes the proof of induction hypothesis \ref{induction_4} at iteration $t_1 + 1$. 
% , and 
% \begin{align*}
%     H = \frac{ \lambda_{\min}  }{1024\lambda_{\max}^{3/2} \cdot \| \wb^{(0)} \|_2^2}
% \end{align*}
% \begin{align*}
%     &H_0 = 4\exp(3\beta / \alpha),\quad H_2 = \frac{ \lambda_{\min} \cdot \exp(-4\beta / \alpha)}{256\lambda_{\max}^{3/2}}, \quad H_3 = \exp\Bigg(\frac{ \lambda_{\min} }{32\lambda_{\max}^{3/2} }\Bigg) \cdot \epsilon,
% \end{align*}
% % and we utilize the assumption that $\gamma^{(t_0)} \leq 2$ to derive the second inequality. 

\noindent\textbf{Proof of induction hypothesis \ref{induction_5} at iteration $t_1 + 1$.} By \eqref{eq:induction_l2distance}, we have
\begin{align*}
    \bigg\|  \wb^{(t_1+1)} -  \frac{\la \wb^*, \wb^{(t_1+1)}\ra}{\| \wb^* \|_2^2} \cdot \wb^* \bigg\|_2\leq \cE^{(t_1+1)} \cdot \|\wb^{(0)} \|_2,
\end{align*}
where 
\begin{align}
    \cE^{(t_1+1)}&:= \exp\Bigg[ -\frac{ \lambda_{\min}  }{1024\lambda_{\max}^{3/2} \cdot \| \wb^{(0)} \|_2^2} \cdot \log^2( (8/9) \eta\cdot (t_1 + 1 - t_0)+ 1)\Bigg]\cdot\epsilon\label{eq:induction_productbound_proof_eq0}
\end{align}
% \begin{align}
%      \cE^{(t_1+1)}&:= \exp\Bigg(\frac{ \lambda_{\min} }{32\lambda_{\max}^{3/2} }\Bigg) \cdot \epsilon \cdot  \exp\Bigg\{-\frac{ \lambda_{\min} \cdot \exp(-4\beta / \alpha)}{256\lambda_{\max}^{3/2}}\cdot \log^2[H_0\eta\cdot (t_1+1-t_0) +9] \Bigg\}.\label{eq:induction_productbound_proof_eq0}
%     %  \nonumber\\
%     %  & \leq \exp\Bigg[\frac{ \lambda_{\min} \cdot  \gamma^{(t_0)2}}{32\lambda_{\max}^{3/2} }\Bigg] \cdot \epsilon\cdot \exp\Bigg\{-\frac{ \lambda_{\min}}{32\lambda_{\max}^{3/2} }\cdot \log^2[H_3\cdot (t_1+1-t_0) +\exp(\gamma^{(t_0)})] \Bigg\} \nonumber\\
%     %  &= \exp\Bigg[\frac{ \lambda_{\min} \cdot  \gamma^{(t_0)2}}{32\lambda_{\max}^{3/2} }\Bigg] \cdot \epsilon\cdot \exp\Bigg\{-\frac{ \lambda_{\min}}{32\lambda_{\max}^{3/2} }\cdot \log^2[H_3\cdot (t_1+1-t_0) +\exp(\gamma^{(t_0)})] \Bigg\}
% \end{align}
% \exp\Bigg[\frac{ \lambda_{\min} \cdot  \gamma^{(t_0)2}}{32\lambda_{\max}^{3/2} }\Bigg] \cdot \epsilon\cdot \exp\{-H_2\cdot \log^2[H_1 \eta\cdot  (t_1+1-t_0) +\exp(\gamma^{(t_0)})] \} \nonumber \\
%      & = 
% The inequality above follows by the fact that $\log^2[H_3\cdot (t_1+1-t_0) +\exp(\gamma^{(t_0)})] \geq \gamma^{(t_0)2}$ and the assumption $\| \wb^{(0)} \|_2 \leq 1$. 
Taking the square of both sides and dividing by $\| \wb^{(t_1+1)} \|_2^{2}$ gives
\begin{align*}
   1 - \la \wb^*/ \| \wb^* \|_2 , \wb^{(t_1+1 )}/  \|\wb^{(t_1+1)} \|_2 \ra^2 \leq (\cE^{(t_1+1)})^2\cdot \|\wb^{(0)} \|_2^2 /  \|\wb^{(t_1+1)} \|_2^2 \leq (\cE^{(t_1+1)})^2 ,
\end{align*}
where the last inequality follows by Lemma~\ref{lemma:norm_upperbound} on the monotonicity of $\|\wb^{(t)} \|_2$. 
Therefore, we have
\begin{align}\label{eq:induction_productbound_proof_eq1}
    (1 - \cE^{(t_1+1)}) \leq \sqrt{1 - (\cE^{(t_1+1)})^2 } \leq \la \wb^*/ \| \wb^* \|_2 , \wb^{(t_1+1 )}  /  \|\wb^{(t_1+1)} \|_2 \ra \leq 1.
\end{align}
Moreover, by \eqref{eq:induction_l2distance}, we have
\begin{align*}
    \bigg| \la \wb^{(t_1+1)}, \zb_i\ra -  \frac{\la \wb^*, \wb^{(t_1+1)}\ra}{\| \wb^* \|_2^2} \cdot \la \wb^* , \zb_i\ra \bigg| \leq \max_i \| \zb_i \|_2\cdot \cE^{(t_1+1)} \cdot \|\wb^{(0)} \|_2
    % | \la \wb^{(t)} , \zb_i \ra - \la \wb^*, \wb^{(t )} \ra \cdot \| \wb^* \|_2^{-2} \cdot \la \wb^* , \zb_i \ra | \leq  \max_i \| \zb_i \|_2\cdot \|\wb^{(0)} \|_2 \cdot \epsilon
\end{align*}
for all $i\in[n]$. Dividing by $\| \wb^{(t_1+1)} \|_2$ on both sides above gives
\begin{align*}
    \bigg| \bigg\la \frac{\wb^{(t_1+1)}}{\| \wb^{(t_1+1)} \|_2} , \zb_i\bigg\ra -  \bigg\la \frac{\wb^*}{\| \wb^* \|_2}, \frac{\wb^{(t_1+1)}}{\| \wb^{(t_1+1)} \|_2}\bigg\ra \cdot \bigg\la \frac{\wb^*}{\| \wb^* \|_2} , \zb_i\bigg\ra \bigg| \leq \max_i \| \zb_i \|_2\cdot \cE^{(t_1+1)}.
    % | \la \wb^{(t)} , \zb_i \ra - \la \wb^*, \wb^{(t )} \ra \cdot \| \wb^* \|_2^{-2} \cdot \la \wb^* , \zb_i \ra | \leq  \max_i \| \zb_i \|_2\cdot \|\wb^{(0)} \|_2 \cdot \epsilon
\end{align*}
% \begin{align*}
%     &| \la \wb^{(t)} / \| \wb^{(t)} \|_2 , \zb_i \ra - \la \wb^* / \| \wb^* \|_2 , \wb^{(t )}  / \| \wb^{(t)} \|_2 \ra \cdot \la \wb^* / \| \wb^* \|_2 , \zb_i \ra |\\
%     &\qquad \qquad \qquad \qquad \qquad \qquad \leq  \max_i \| \zb_i \|_2\cdot \|\wb^{(0)} \|_2  / \| \wb^{(t)} \|_2 \cdot \epsilon \leq  \max_i \| \zb_i \|_2\cdot \epsilon.
% \end{align*}
Recall again that $\wb^*$ is chosen such that $\la \wb^*, \zb_i\ra = 1$ for all $i\in[n]$. Therefore, rearranging terms and applying \eqref{eq:induction_productbound_proof_eq1} then gives
\begin{align*}
     \bigg\la \frac{\wb^{(t_1+1)}}{\| \wb^{(t_1+1)} \|_2} , \zb_i\bigg\ra  &\geq \bigg\la \frac{\wb^*}{\| \wb^* \|_2}, \frac{\wb^{(t_1+1)}}{\| \wb^{(t_1+1)} \|_2}\bigg\ra \cdot \bigg\la \frac{\wb^*}{\| \wb^* \|_2} , \zb_i\bigg\ra - \max_i \| \zb_i \|_2\cdot \cE^{(t_1+1)}\\
    % \la \wb^{(t)} / \| \wb^{(t)} \|_2  , \zb_i \ra &\geq \la \wb^* / \| \wb^* \|_2 , \wb^{(t )}  / \| \wb^{(t)} \|_2 \ra \cdot \la \wb^* / \| \wb^* \|_2 , \zb_i \ra -  \max_i \| \zb_i \|_2\cdot \epsilon\\
    &\geq \bigg\la \frac{\wb^*}{\| \wb^* \|_2} , \zb_i\bigg\ra - 2\max_i \| \zb_i \|_2\cdot \cE^{(t_1+1)}\\
    &= \| \wb^* \|_2^{-1} -2\max_i \| \zb_i \|_2\cdot \cE^{(t_1+1)}, \\
    \bigg\la \frac{\wb^{(t_1+1)}}{\| \wb^{(t_1+1)} \|_2} , \zb_i\bigg\ra  &\leq \bigg\la \frac{\wb^*}{\| \wb^* \|_2}, \frac{\wb^{(t_1+1)}}{\| \wb^{(t_1+1)} \|_2}\bigg\ra \cdot \bigg\la \frac{\wb^*}{\| \wb^* \|_2} , \zb_i\bigg\ra + \max_i \| \zb_i \|_2\cdot \cE^{(t_1+1)}\\
    % \la \wb^{(t)} / \| \wb^{(t)} \|_2  , \zb_i \ra &\geq \la \wb^* / \| \wb^* \|_2 , \wb^{(t )}  / \| \wb^{(t)} \|_2 \ra \cdot \la \wb^* / \| \wb^* \|_2 , \zb_i \ra -  \max_i \| \zb_i \|_2\cdot \epsilon\\
    &\leq \bigg\la \frac{\wb^*}{\| \wb^* \|_2} , \zb_i\bigg\ra + \max_i \| \zb_i \|_2\cdot \cE^{(t_1+1)}\\
    &= \| \wb^* \|_2^{-1} + \max_i \| \zb_i \|_2\cdot \cE^{(t_1+1)}.
\end{align*}
Therefore, we have
\begin{align}\label{eq:induction_productbound_proof_eq2}
    \bigg| \bigg\la \frac{\wb^{(t_1+1)}}{\| \wb^{(t_1+1)} \|_2} , \zb_i\bigg\ra - \alpha \bigg| =\bigg| \bigg\la \frac{\wb^{(t_1+1)}}{\| \wb^{(t_1+1)} \|_2} , \zb_i\bigg\ra - \| \wb^* \|_2^{-1} \bigg| \leq 2\max_i \| \zb_i \|_2\cdot  \cE^{(t_1+1)}.
\end{align}
Moreover, by the induction hypothesis \ref{induction_3} at iteration $t_1 + 1$ (which has been proved), we have
\begin{align}
    \gamma^{(t_1+1)} &\leq  \log[ 8\eta\cdot (t_1 + 1 - t_0) +2\exp(\gamma^{(t_0)}) ] \nonumber\\
    &= \log(9) + \log[ (8/9) \eta\cdot (t_1 + 1 - t_0) +(2/9)\cdot\exp(\gamma^{(t_0)}) ]\nonumber\\
    &\leq \log(9) + \log[ (8/9) \eta\cdot (t_1 + 1 - t_0) +(2/9)\cdot\exp(1.5) ]\nonumber\\
    &\leq \log(9) + \log[ (8/9) \eta\cdot (t_1 + 1 - t_0) + 1 ] .\label{eq:induction_productbound_proof_eq3}
    % &=  \log[ 4 \exp(3\beta/\alpha )\cdot \eta \cdot (t_1+1 - t_0) +2\exp(\gamma^{(t_0)}) ] \nonumber \\
    % &\leq \log[ 4 \exp(3\beta/\alpha )\cdot \eta \cdot (t_1+1 - t_0) + 9 ]
\end{align}
% \begin{align}
%     \gamma^{(t_1+1)} &\leq  \log[ 2\eta\cdot \exp(3\beta / \alpha) \cdot (t_1 + 1 - t_0) +\exp(\gamma^{(t_0)}) ] +  2\eta\cdot \exp(3\beta / \alpha) \cdot \exp(-\gamma^{(t_0)}) \nonumber \\
%     & =  \log[ (\eta / 4)\cdot \exp(-\beta / \alpha) \cdot (t_1 + 1 - t_0) + \exp(-4\beta / \alpha) / 8 \cdot \exp(\gamma^{(t_0)}) ] \nonumber  \\
%     &\quad + \log[8\exp(4\beta / \alpha) ] + 2\eta\cdot \exp(3\beta / \alpha) \cdot \exp(-\gamma^{(t_0)}) \nonumber \\
%     & \leq  \log[ (\eta / 4)\cdot \exp(-\beta / \alpha) \cdot (t_1 + 1 - t_0) + \exp(\gamma^{(t_0)}) ]  + \log(8) + 4\beta / \alpha \nonumber \\
%     &\quad + 2\eta\cdot \exp(3\beta / \alpha) \cdot \exp(-\gamma^{(t_0)}) \nonumber \\
%     & = \log[H_1\eta \cdot (t_1 + 1 - t_0) + \exp(\gamma^{(t_0)}) ]  + \log(8) + 4\beta / \alpha + 2\eta\cdot \exp(3\beta / \alpha) \cdot \exp(-\gamma^{(t_0)})\nonumber \\
%     & \leq \log[H_1\eta \cdot (t_1 + 1 - t_0) + \exp(\gamma^{(t_0)}) ] + 4\beta / \alpha + 3,
%     \label{eq:induction_productbound_proof_eq3}
% \end{align}
% where the last inequality follows by the assumption at $\eta \leq \exp(-3\beta / \alpha) / 4$. 
Therefore by \eqref{eq:induction_productbound_proof_eq0}, \eqref{eq:induction_productbound_proof_eq2} and \eqref{eq:induction_productbound_proof_eq3}, we have
\begin{align*}
    &\bigg| \bigg\la \frac{\wb^{(t_1+1)}}{\| \wb^{(t_1+1)} \|_2} , \zb_i\bigg\ra - \alpha \bigg|\cdot \gamma^{(t_1+1)}\\
    &\qquad \qquad \leq \frac{ 2\max_i \| \zb_i \|_2 \cdot\epsilon  \cdot \{ \log(9) + \log[ (8/9) \eta\cdot (t_1 + 1 - t_0) + 1 ] \} }{ \exp\Big[\frac{ \lambda_{\min}  }{1024\lambda_{\max}^{3/2} \cdot \| \wb^{(0)} \|_2^2} \cdot \log^2( (8/9) \eta\cdot (t_1 + 1 - t_0)+ 1)\Big]}  \\
    &\qquad \qquad \leq 2\log(9) \cdot \max_i \| \zb_i \|_2 \cdot \epsilon + \frac{32\sqrt{2} \cdot \exp(1/2) \cdot \lambda_{\max}^{3/4}}{ \lambda_{\min}^{1/2}} \cdot \max_i \| \zb_i \|_2 \cdot \epsilon \cdot \| \wb^{(0)} \|_2\\
    &\qquad \qquad \leq 6 \max_i \| \zb_i \|_2  \cdot \epsilon + \frac{80 \lambda_{\max}^{3/4}}{ \lambda_{\min}^{1/2}}\cdot \max_i \| \zb_i \|_2  \cdot \epsilon \cdot \| \wb^{(0)} \|_2\\
    &\qquad \qquad \leq \| \wb^* \|_2^{-1} / 4.
\end{align*}
where the second inequality follows by the fact  that $\exp(-Az^2)\cdot z \leq \exp(-1/2)\cdot (2A)^{-1/2} $ for all $A,z > 0$, and the last inequality follows by the definition of $\epsilon$ which ensures that $\epsilon \leq (16\max_i \| \zb_i \|_2 )^{-1}\cdot \| \wb^* \|_2^{-1}\cdot \min\big\{1/3,  \| \wb^{(0)} \|_2^{-1} \cdot \lambda_{\min}^{1/2} / (40 \lambda_{\max}^{3/4}) \big\}$. This finishes the proof of induction hypothesis \ref{induction_5} at iteration $t_1 + 1$, and thus the first five results in Lemma~\ref{lemma:linear_asymp} hold for all $t\geq t_0$. 

\noindent\textbf{Proof of the last result in Lemma~\ref{lemma:linear_asymp}.} As we have shown by induction, the first five results in Lemma~\ref{lemma:linear_asymp} hold for all $t\geq t_0$. Therefore \eqref{eq:yf_upperbound} and \eqref{eq:yf_upperbound} also hold for all $t\geq t_0$. Therefore by the monotonicity of $\ell(\cdot)$ and the third result in Lemma~\ref{lemma:linear_asymp}, we have
\begin{align*}
    \ell( y_i\cdot f(\wb^{(t)},\gamma^{(t)},\xb_i ) ) &\leq \log( 1 +\exp(-\gamma^{(t)} + 1/4))\\
    &\leq \exp\{ - \log[ (\eta / 8) \cdot (t - t_0) +\exp(\gamma^{(t_0)}) ] + 1/4 \}\\
    &= \frac{\exp(1/4)}{(\eta / 8) \cdot (t - t_0) +\exp(\gamma^{(t_0)})} \\
    &\leq \frac{12}{\eta\cdot (t - t_0) + 1}
\end{align*}
for all t $\geq t_0$, where the second inequality follows by $\log(1+ z) \geq z$ for all $z\in\RR$. 
Similarly, we also have
\begin{align*}
    \ell( y_i\cdot f(\wb^{(t)},\gamma^{(t)},\xb_i ) ) &\geq \log(1 + \exp(- \gamma^{(t)} - 3/4) )\\
    &\geq \exp\{- \log[ 8\eta\cdot (t - t_0) +2\exp(\gamma^{(t_0)}) ] - 3/4\} / 2\\
    &= \frac{\exp(-3/4)/2}{8\eta\cdot (t - t_0) +2\exp(\gamma^{(t_0)})}\\
    &\geq \frac{1}{40}\cdot \frac{1}{\eta\cdot (t - t_0 ) + 1}
\end{align*}
for all t $\geq t_0$, where the second inequality follows by the fact that $\log(1 + z) \geq z/2$ for all $z\in [0,1]$. This proves the first part of the result.

As for the second part of the result, by Lemma~\ref{lemma:key_identity}, for all $t\geq t_0$ we have
\begin{align}
\frac{1}{n^2}\sum_{i,i'=1}^n   (\la \wb^{(t)} , \zb_{i'}\ra - \la \wb^{(t_1+1)} , \zb_{i}\ra)^2
&= 2 \|  \wb^{(t)} - \la \wb^*, \wb^{(t)}\ra_{\bSigma} \cdot \wb^* \|_{\bSigma}^2\nonumber \\
& \leq 2\lambda_{\max} \cdot \bigg\|  \wb^{(t)} -  \frac{\la \wb^*, \wb^{(t)}\ra}{\| \wb^* \|_2^2} \cdot \wb^* \bigg\|_2^2.\label{eq:uniform_margin_induction_proof_eq2}
% \\
%     \bigg\|  \wb^{(t_1+1)} -  \frac{\la \wb^*, \wb^{(t_1+1)}\ra}{\| \wb^* \|_2^2} \cdot \wb^* \bigg\|_{\bSigma}^2 \leq \lambda_{\max}\cdot \bigg\|  \wb^{(t_1+1)} -  \frac{\la \wb^*, \wb^{(t_1+1)}\ra}{\| \wb^* \|_2^2} \cdot \wb^* \bigg\|_2^2
\end{align}
% It is clear that
% \begin{align*}
%     \la \wb^*, \wb^{(t_1+1)}\ra_{\bSigma} \cdot \Zb^\top\wb^* = \frac{1}{n}\la \Zb^\top\wb^*, \Zb^\top\wb^{(t_1+1)}\ra \cdot \Zb^\top\wb^* = \frac{1}{n}\la \mathbf{1}, \Zb^\top\wb^{(t_1+1)}\ra \cdot \mathbf{1} 
% \end{align*}
% is the projection of $\Zb^\top\wb$ onto $\mathrm{span}\{ \Zb^\top\wb^* \} = \mathrm{span}\{ \mathbf{1} \}$. Therefore
% \begin{align*}
%     \frac{1}{n} \| \Zb^\top \wb^{(t_1+1)} - \la \wb^*, \wb^{(t_1+1)}\ra_{\bSigma} \cdot \Zb^\top \wb^*\|_2^2 \leq \frac{1}{n} \| \Zb^\top \wb^{(t_1+1)} -c \cdot \Zb^\top \wb^*\|_2^2
% \end{align*}
% for all $c\in \RR$. Hence 
% \begin{align}
%     \frac{1}{n} \| \Zb^\top \wb^{(t_1+1)} - \la \wb^*, \wb\ra_{\bSigma} \cdot \Zb^\top \wb^*\|_2^2 &\leq \frac{1}{n}\bigg\|  \Zb^\top \wb^{(t_1+1)} -  \frac{\la \wb^*, \wb^{(t_1+1)}\ra}{\| \wb^* \|_2^2} \cdot  \Zb^\top\wb^* \bigg\|_2^2\nonumber\\
%     &= \bigg\|  \wb^{(t_1+1)} -  \frac{\la \wb^*, \wb^{(t_1+1)}\ra}{\| \wb^* \|_2^2} \cdot \wb^* \bigg\|_{\bSigma}^2\nonumber\\
%     &\leq \lambda_{\max} \cdot \bigg\|  \wb^{(t_1+1)} -  \frac{\la \wb^*, \wb^{(t_1+1)}\ra}{\| \wb^* \|_2^2} \cdot \wb^* \bigg\|_2^2.\label{eq:uniform_margin_induction_proof_eq2}
% \end{align}
Now that the fourth result in Lemma~\ref{lemma:linear_asymp}, which has been proved to hold for all $t\geq t_0$, we have
\begin{align}\label{eq:uniform_margin_induction_proof_eq1}
    \bigg\|  \wb^{(t)} -  \frac{\la \wb^*, \wb^{(t)}\ra}{\| \wb^* \|_2^2} \cdot \wb^* \bigg\|_2^2 \leq  \epsilon \cdot \|\wb^{(0)} \|_2 \cdot \exp\Bigg[ -\frac{ \lambda_{\min} \cdot \log^2( (8/9) \eta\cdot (t - t_0)+ 1) }{1024\lambda_{\max}^{3/2} \cdot \| \wb^{(0)} \|_2^2} \Bigg]
\end{align}
for all $t\geq t_0$. 
Combining \eqref{eq:uniform_margin_induction_proof_eq1} and \eqref{eq:uniform_margin_induction_proof_eq2} then gives
\begin{align*}
    &\frac{1}{n^2}\sum_{i,i'=1}^n   (\la \wb^{(t)} , \zb_{i'}\ra - \la \wb^{(t)} , \zb_{i}\ra)^2 \\
    &\qquad \leq \lambda_{\max}\cdot \epsilon^2\cdot \|\wb^{(0)} \|_2^2 \cdot \exp\Bigg[ -\frac{ \lambda_{\min}  }{512\lambda_{\max}^{3/2} \cdot \| \wb^{(0)} \|_2^2} \cdot \log^2( (8/9) \eta\cdot (t - t_0)+ 1)\Bigg].
    % \cdot H_3^2\cdot \|\wb^{(0)} \|_2^2 \cdot  \exp\{-2H_2\cdot  \log^2[H_0\eta \cdot (t_1+1-t_0) + 9] \}.
\end{align*}
Then by the definitions of $D(\wb)$ and $\epsilon$, we have 
\begin{align*}
    &D(\wb^{(t)}) \\
    &\leq \frac{1}{\| \wb^{(t)} \|_{\bSigma}^2} \cdot \frac{1}{n^2}\sum_{i,i'=1}^n (\la \wb^{(t)} , \zb_{i'}\ra - \la \wb^{(t)} , \zb_{i}\ra)^2 \\
    &\leq \frac{1}{\| \wb^{(t)} \|_{\bSigma}^2} \cdot \frac{\lambda_{\max}\cdot  \|\wb^{(0)} \|_2^2}{2304\max_i \| \xb_i \|_2^2 \cdot \| \wb^* \|_2^2} \cdot \exp\Bigg[ -\frac{ \lambda_{\min}  }{512\lambda_{\max}^{3/2} \cdot \| \wb^{(0)} \|_2^2} \cdot \log^2( (8/9) \eta\cdot (t - t_0)+ 1)\Bigg].
\end{align*}
By the definition of $\wb^*$, clearly we have $\max_i \| \xb_i \|_2^2 \cdot \| \wb^* \|_2^2 \geq \min_i \la y_i\cdot \xb_i , \wb^* \ra^2 = 1$. Therefore
\begin{align*}
     D(\wb^{(t)}) &\leq \frac{\lambda_{\max}\cdot  \|\wb^{(0)} \|_2^2}{2304 \| \wb^{(t)} \|_{\bSigma}^2} \cdot \exp\Bigg[ -\frac{ \lambda_{\min}  }{512\lambda_{\max}^{3/2} \cdot \| \wb^{(0)} \|_2^2} \cdot \log^2( (8/9) \eta\cdot (t - t_0)+ 1)\Bigg]\\
     &\leq \frac{\lambda_{\max}\cdot  \|\wb^{(0)} \|_2^2}{2304 \lambda_{\min} \| \wb^{(t)} \|_{2}^2} \cdot \exp\Bigg[ -\frac{ \lambda_{\min}  }{512\lambda_{\max}^{3/2} \cdot \| \wb^{(0)} \|_2^2} \cdot \log^2( (8/9) \eta\cdot (t - t_0)+ 1)\Bigg]\\
    &\leq \frac{\lambda_{\max}}{2304 \lambda_{\min} } \cdot \exp\Bigg[ -\frac{ \lambda_{\min}  }{512\lambda_{\max}^{3/2} \cdot \| \wb^{(0)} \|_2^2} \cdot \log^2( (8/9) \eta\cdot (t - t_0)+ 1)\Bigg]
\end{align*}
for all $t\geq t_0$, where the last inequality follows by Lemma~\ref{lemma:norm_upperbound}. Therefore the last result in Lemma~\ref{lemma:linear_asymp} holds, and the proof of Lemma~\ref{lemma:linear_asymp} is thus complete.
% $\gamma > 0$ and inner products $> 0$
\end{proof}

% \section{Proofs of Theorems in Section \ref{subsection:data_examples}.}
\section{Proofs for Batch Normalization in Two-Layer Linear CNNs}

\subsection{Proof of Theorem~\ref{thm:CNNBN}}
As most part of the proof of Theorem~\ref{thm:CNNBN} are the same as the proof of Theorem~\ref{thm:linearBN}, here we only highlight the differences between these two proofs. Specifically, we give the proofs of the counterparts of Lemmas~\ref{lemma:gradient_inner_linear_model}, \ref{lemma:innerproduct_lowerbound}, \ref{lemma:norm_upperbound}. The rest of the proofs are essentially the same based on the multi-patch versions of these three results.

By Assumption~\ref{assump:uniformly_separable_CNN}, we can define $\wb^*$ as the minimum norm solution of the system:
\begin{align}\label{eq:def_w*_CNN}
    \wb^* := \argmin_{\wb} \| \wb \|_2^2, ~~\text{subject to }  \la \wb, y_i\cdot \xb_i^{(p)} \ra = 1, ~i \in [n],~p\in[P].
\end{align}

The following lemma is the counterpart of Lemma~\ref{lemma:gradient_inner_linear_model}.

\begin{lemma}\label{lemma:gradient_inner_CNN}
Under Assumption~\ref{assump:uniformly_separable}, for any $\wb\in \RR^d$, it holds that 
\begin{align*}
    \la -\nabla_{\wb} L(\wb,\gamma), \wb^* \ra &= \frac{\gamma}{2 n^2 P \| \wb \|_{\bSigma}^3}\sum_{i,i'=1}^n |\ell'_i|\cdot |\ell'_{i'}| \cdot (\la \wb , \zb_{i'}\ra - \la \wb , \zb_{i}\ra)\cdot ( |\ell'_{i'}|^{-1} \cdot \la \wb , \zb_{i'}\ra - |\ell'_{i}|^{-1} \cdot \la \wb , \zb_{i}\ra )\\
    &\quad + \frac{\gamma}{2 n^2 P \| \wb \|_{\bSigma}^3} \cdot \Bigg(\sum_{i=1}^n |\ell'_i| \Bigg) \cdot \sum_{i'=1}^n \sum_{p,p'=1}^P   \big(\la \wb , \zb_{i'}^{(p)}\ra - \la \wb , \zb_{i'}^{(p')}\ra \big)^2,
\end{align*}
% \begin{align*}
%     &I_1 = \| \wb \|_{\bSigma}^{-3}\cdot \frac{\gamma}{2 n^2 P }\sum_{i,i'=1}^n |\ell'_i|\cdot |\ell'_{i'}|\cdot (u_{i'} - u_{i})\cdot \big( |\ell'_{i'}|^{-1} \cdot u_{i'} - |\ell'_{i}|^{-1} \cdot u_{i} \big). \\
%     &I_2= \| \wb \|_{\bSigma}^{-3}\cdot \frac{\gamma}{2n^2 P }\sum_{i,i'=1}^n \sum_{p,p'=1}^P  |\ell'_i| \cdot \big(u_{i'}^{(p)} - u_{i'}^{(p')} \big)^2
% \end{align*}
where $ \ell'_i = \ell'[y_i\cdot f(\wb,\gamma,\xb_i)]$, $\zb_i^{(p)} = y_i\cdot \xb_i^{(p)}$, and $\zb_i = \sum_{p=1}^P \zb_i^{(p)}$ for $i\in[n]$, $p\in[P]$. 
\end{lemma}
\begin{proof}[Proof of Lemma~\ref{lemma:gradient_inner_CNN}] By definition, we have 
\begin{align*}
    \nabla_{\wb} L(\wb,\gamma) = \frac{1}{n}\sum_{i=1}^n \ell'[y_i\cdot f(\wb,\gamma,\xb_i)] \cdot y_i\cdot \nabla_{\wb} f(\wb,\gamma,\xb_i).
\end{align*}
Then by the definition of $f(\wb,\gamma,\xb)$, we have the following calculation using chain rule: 
\begin{align*}
    \nabla_{\wb} L(\wb,\gamma) &=  \| \wb \|_{\bSigma}^{-1}\cdot \frac{1}{n}\sum_{i=1}^n \ell'[y_i\cdot f(\wb,\gamma,\xb_i)] \cdot y_i \cdot\sum_{p=1}^P \gamma \cdot \Big( \Ib - \| \wb \|_{\bSigma}^{-2}\cdot \bSigma \wb \wb^\top \Big)\xb_i^{(p)} \\
    & = \| \wb \|_{\bSigma}^{-3}\cdot \frac{\gamma}{n}\sum_{i=1}^n \sum_{p=1}^P \ell'_i \cdot y_i \cdot \Big( \| \wb \|_{\bSigma}^{2} -  \bSigma \wb \wb^\top \Big)\xb_i^{(p)}\nonumber \\
     & = \| \wb \|_{\bSigma}^{-3}\cdot \frac{\gamma}{n}\sum_{i=1}^n \sum_{p=1}^P  \ell'_i \cdot \Bigg( \frac{1}{nP } \sum_{i'=1}^n\sum_{p'=1}^{P}  \la \wb, \zb_{i'}^{(p')}\ra^2 - \frac{1}{n P}\sum_{i'=1}^n\sum_{p'=1}^{P} \la \wb , \zb_{i'}^{(p')}\ra \cdot \zb_{i'}^{(p')} \wb^\top \Bigg)\zb_i^{(p)}\\
     & = \| \wb \|_{\bSigma}^{-3}\cdot \frac{\gamma}{n^2 P }\sum_{i,i'=1}^n \sum_{p,p'=1}^P \Big[ \ell'_i \cdot \la \wb, \zb_{i'}^{(p')}\ra^2\cdot \zb_i^{(p)} - \ell'_i \cdot \la \wb , \zb_{i'}^{(p')}\ra \cdot  \la\wb,\zb_i^{(p)}\ra \cdot \zb_{i'}^{(p')} \Big],
\end{align*}
where we denote $\zb_i^{(p)} = y_i\cdot \xb_i^{(p)}$, $i\in[n]$, $p\in[P]$.  
By Assumption~\ref{assump:uniformly_separable_CNN}, taking inner product with $-\wb^*$ on both sides above then gives
\begin{align*}
    -\la \nabla_{\wb} L(\wb,\gamma), \wb^* \ra & = - \| \wb \|_{\bSigma}^{-3}\cdot \frac{\gamma}{n^2 P }\sum_{i,i'=1}^n \sum_{p,p'=1}^P \Big[ \ell'_i \cdot \la \wb, \zb_{i'}^{(p')}\ra^2 - \ell'_i \cdot \la \wb , \zb_{i'}^{(p')}\ra \cdot  \la\wb,\zb_i^{(p)}\ra \Big]\\
    & = \| \wb \|_{\bSigma}^{-3}\cdot \frac{\gamma}{n^2 P }\sum_{i,i'=1}^n \sum_{p,p'=1}^P \Big[ |\ell'_i| \cdot \la \wb, \zb_{i'}^{(p')}\ra^2 - |\ell'_i| \cdot \la \wb , \zb_{i'}^{(p')}\ra \cdot  \la\wb,\zb_i^{(p)}\ra \Big],
\end{align*}
where the second equality follows by the fact that $\ell'_i < 0$, $i\in[n]$. % \begin{align*}
%     \la \nabla_{\wb} L(\wb,\gamma), \wb^* \ra = \| \wb \|_{\bSigma}^{-3}\cdot \frac{\gamma}{n^2}\sum_{i=1}^n \sum_{i'=1}^n  \ell'_i \cdot \la \wb, \zb_{i'}\ra^2 - \| \wb \|_{\bSigma}^{-3}\cdot \frac{\gamma}{n^2}\sum_{i=1}^n \sum_{i'=1}^n \ell'_i \cdot \la \wb , \zb_{i'}\ra \cdot  \la\wb,\zb_i\ra. 
% \end{align*}
Further denote $u_i^{(p)} = \la \wb, \zb_{i}\ra$ for $i\in[n]$, $p\in [P]$. Then we have 
\begin{align*}
    -\la \nabla_{\wb} L(\wb,\gamma), \wb^* \ra
    & = \| \wb \|_{\bSigma}^{-3}\cdot \frac{\gamma}{n^2 P }\sum_{i,i'=1}^n \sum_{p,p'=1}^P \big[ |\ell'_i| \cdot u_{i'}^{(p')2} - |\ell'_i| \cdot u_{i'}^{(p')}u_{i}^{(p)} \big]\\
    & = \| \wb \|_{\bSigma}^{-3}\cdot \frac{\gamma}{n^2 P }\sum_{i,i'=1}^n \Bigg[ |\ell'_i| \cdot \Bigg(  \sum_{p,p'=1}^P u_{i'}^{(p')2} -  \sum_{p,p'=1}^P u_{i'}^{(p')}u_{i}^{(p)} \Bigg)\Bigg]\\
    & = \| \wb \|_{\bSigma}^{-3}\cdot \frac{\gamma}{n^2 P }\sum_{i,i'=1}^n \Bigg\{ |\ell'_i| \cdot \Bigg[  P\cdot \sum_{p'=1}^P (u_{i'}^{(p')})^2 -  \Bigg(\sum_{p=1}^P u_{i}^{(p)} \Bigg)\cdot \Bigg(\sum_{p'=1}^P u_{i'}^{(p')} \Bigg)\Bigg]\Bigg\}.
    % \\
    % & = \| \wb \|_{\bSigma}^{-3}\cdot \frac{\gamma}{n^2 P }\sum_{i,i'=1}^n \Bigg\{ |\ell'_i| \cdot \Bigg[ \Bigg(\sum_{p'=1}^P u_{i'}^{(p')} \Bigg)^2 -  \Bigg(\sum_{p=1}^P u_{i}^{(p)} \Bigg)\cdot \Bigg(\sum_{p'=1}^P u_{i'}^{(p')} \Bigg)\Bigg]\Bigg\}\\
    % &\quad + \| \wb \|_{\bSigma}^{-3}\cdot \frac{\gamma}{n^2 P }\sum_{i,i'=1}^n \Bigg\{ |\ell'_i| \cdot \Bigg[  P\cdot \sum_{p'=1}^P u_{i'}^{(p')2} -  \Bigg(\sum_{p'=1}^P u_{i'}^{(p')} \Bigg)^2\Bigg]\Bigg\}
\end{align*}
Adding and subtracting a term 
$$\| \wb \|_{\bSigma}^{-3}\cdot \frac{\gamma}{n^2 P }\sum_{i,i'=1}^n  \Bigg[ |\ell'_i| \cdot \Bigg(\sum_{p'=1}^P u_{i'}^{(p')} \Bigg)^2  \Bigg]
$$
then gives 
\begin{align}
    -\la \nabla_{\wb} L(\wb,\gamma), \wb^* \ra
    & = \underbrace{\| \wb \|_{\bSigma}^{-3}\cdot \frac{\gamma}{n^2 P }\sum_{i,i'=1}^n \Bigg\{ |\ell'_i| \cdot \Bigg[ \Bigg(\sum_{p'=1}^P u_{i'}^{(p')} \Bigg)^2 - \Bigg(\sum_{p=1}^P u_{i}^{(p)} \Bigg)\cdot \Bigg(\sum_{p'=1}^P u_{i'}^{(p')} \Bigg)\Bigg]\Bigg\} }_{I_1}\nonumber\\
    &\quad +\underbrace{ \| \wb \|_{\bSigma}^{-3}\cdot \frac{\gamma}{n^2 P }\sum_{i,i'=1}^n \Bigg\{ |\ell'_i| \cdot \Bigg[  P\cdot \sum_{p'=1}^P (u_{i'}^{(p')})^2 -  \Bigg(\sum_{p'=1}^P u_{i'}^{(p')} \Bigg)^2\Bigg]\Bigg\}}_{I_2}.\label{eq:proof_gradient_inner_CNN_I1I2}
\end{align}
We then calculate the two terms $I_1$ and $I_2$ in \eqref{eq:proof_gradient_inner_CNN_I1I2} separately. The calculation for $I_1$ is the same as the derivation in the proof of Lemma~\ref{lemma:gradient_inner_linear_model}. To see this, let $u_i = \sum_{p=1}^P u_i^{(p)}$ for $i\in[n]$, $p\in [P]$. Then we have
\begin{align}\label{eq:proof_gradient_inner_CNN_eq1}
    I_1 = \| \wb \|_{\bSigma}^{-3}\cdot \frac{\gamma}{n^2 P }\sum_{i,i'=1}^n  \big( |\ell'_i| \cdot u_{i'}^2 - |\ell'_i| \cdot u_{i'}u_{i} \big).
\end{align}
Switching the index notations $i,i'$ in the above equation also gives
\begin{align}\label{eq:proof_gradient_inner_CNN_eq2}
    I_1 = \| \wb \|_{\bSigma}^{-3}\cdot \frac{\gamma}{n^2 P }\sum_{i,i'=1}^n  \big( |\ell'_{i'}| \cdot u_{i}^2 - |\ell'_{i'}| \cdot u_{i'}u_{i} \big).
\end{align}
We can add \eqref{eq:proof_gradient_inner_CNN_eq1} and \eqref{eq:proof_gradient_inner_CNN_eq2} together to obtain
\begin{align*}
    2I_1 &= \| \wb \|_{\bSigma}^{-3}\cdot \frac{\gamma}{n^2 P }\sum_{i,i'=1}^n  \big( |\ell'_i| \cdot u_{i'}^2 - |\ell'_i| \cdot u_{i'}u_{i} + |\ell'_{i'}| \cdot u_{i}^2 - |\ell'_{i'}| \cdot u_{i'}u_{i} \big) \\
    &= \| \wb \|_{\bSigma}^{-3}\cdot \frac{\gamma}{n^2 P }\sum_{i,i'=1}^n  (u_{i'} - u_{i})\cdot \big( |\ell'_i| \cdot u_{i'} - |\ell'_{i'}| \cdot u_{i} \big)\\ 
    &= \| \wb \|_{\bSigma}^{-3}\cdot \frac{\gamma}{n^2 P }\sum_{i,i'=1}^n |\ell'_i|\cdot |\ell'_{i'}|\cdot (u_{i'} - u_{i})\cdot \big( |\ell'_{i'}|^{-1} \cdot u_{i'} - |\ell'_{i}|^{-1} \cdot u_{i} \big).
\end{align*}
Therefore we have
\begin{align}\label{eq:proof_gradient_inner_CNN_I1calc}
    I_1 = \| \wb \|_{\bSigma}^{-3}\cdot \frac{\gamma}{2 n^2 P }\sum_{i,i'=1}^n |\ell'_i|\cdot |\ell'_{i'}|\cdot (u_{i'} - u_{i})\cdot \big( |\ell'_{i'}|^{-1} \cdot u_{i'} - |\ell'_{i}|^{-1} \cdot u_{i} \big).
\end{align}
This completes the calculation for $I_1$. We then proceed to calculate $I_2$. for any $i'\in[n]$, we can directly check that the following identity holds: 
\begin{align*}
 \frac{1}{2} \cdot \sum_{p,p'=1}^P \big(u_{i'}^{(p)} - u_{i'}^{(p')} \big)^2 = \frac{1}{2} \cdot \sum_{p,p'=1}^P \big[(u_{i'}^{(p)})^2 + (u_{i'}^{(p')})^2 - 2 u_{i'}^{(p)}u_{i'}^{(p')} \big] = 
    P\cdot \sum_{p'=1}^P (u_{i'}^{(p')})^2 -  \Bigg(\sum_{p'=1}^P u_{i'}^{(p')} \Bigg)^2.
\end{align*}
Note that the right hand side above appears in $I_2$. Plugging the above calculation into the definition of $I_2$ gives
\begin{align}
    I_2 &=  \| \wb \|_{\bSigma}^{-3}\cdot \frac{\gamma}{n^2 P }\sum_{i,i'=1}^n \Bigg\{ |\ell'_i| \cdot \Bigg[ \frac{1}{2} \cdot \sum_{p,p'=1}^P \big(u_{i'}^{(p)} - u_{i'}^{(p')} \big)^2 \Bigg]\Bigg\}\nonumber \\
    & = \| \wb \|_{\bSigma}^{-3}\cdot \frac{\gamma}{2n^2 P }\sum_{i,i'=1}^n \sum_{p,p'=1}^P  |\ell'_i| \cdot \big(u_{i'}^{(p)} - u_{i'}^{(p')} \big)^2\nonumber\\
    & = \| \wb \|_{\bSigma}^{-3}\cdot \frac{\gamma}{2n^2 P } \cdot \Bigg( \sum_{i=1}^n |\ell'_i| \Bigg) \cdot\sum_{i'=1}^n \sum_{p,p'=1}^P \big(u_{i'}^{(p)} - u_{i'}^{(p')} \big)^2.\label{eq:proof_gradient_inner_CNN_I2calc}
\end{align}
Finally, plugging \eqref{eq:proof_gradient_inner_CNN_I1calc} and \eqref{eq:proof_gradient_inner_CNN_I2calc} into \eqref{eq:proof_gradient_inner_CNN_I1I2} completes the proof. 
% Note that by definition we have $\ell'_i< 0$. Therefore, 
% \begin{align*}
%     - \la \nabla_{\wb} L(\wb,\gamma) = \| \wb \|_{\bSigma}^{-3}\cdot \frac{\gamma}{2 n^2}\sum_{i=1}^n \sum_{i'=1}^n |\ell'_i|\cdot |\ell'_{i'}| \cdot (u_{i'} - u_i)( |\ell'_{i'}|^{-1} \cdot u_{i'} - |\ell'_{i}|^{-1} \cdot u_{i} ).
% \end{align*}
% This completes the proof.
\end{proof}

The following lemma follows by exactly the same proof as Lemma~\ref{lemma:key_identity}. 
\begin{lemma}\label{lemma:key_identity_CNN}
For any $\wb \in \RR^d$, it holds that
\begin{align*}
    \frac{1}{n^2P^2}\sum_{i,i'=1}^n \sum_{p,p'=1}^P   (\la \wb , \zb_{i'}\ra - \la \wb , \zb_{i}\ra)^2 = \|  \wb - \la \wb^*, \wb\ra_{\bSigma} \cdot \wb^* \|_{\bSigma}^2.
\end{align*}
\end{lemma}

The following lemma is the counterpart of Lemma~\ref{lemma:innerproduct_lowerbound}.
\begin{lemma}\label{lemma:innerproduct_lowerbound_CNN}
For all $t \geq 0$, it holds that
\begin{align*}
 \la \wb^{(t+1)}, \wb^* \ra \geq \la \wb^{(t)}, \wb^* \ra + \frac{\eta P \gamma^{(t)}\cdot \exp(-\gamma^{(t)}) }{16\cdot \| \wb^{(t)} \|_{\bSigma}^3}\cdot  \|  \wb - \la \wb^*, \wb\ra_{\bSigma} \cdot \wb^* \|_{\bSigma}^2.
\end{align*}
% \begin{align*}
%  \la \wb^{(t+1)}, \wb^* \ra \geq \la \wb^{(t)}, \wb^* \ra + \frac{\eta \gamma^{(t)}\cdot \exp(-\gamma^{(t)}) }{16n^2P\cdot \| \wb^{(t)} \|_{\bSigma}^3}\cdot \sum_{i,i'=1}^n \sum_{p,p'=1}^P (\la \wb^{(t)} , \zb_{i}^{(p)}\ra - \la \wb^{(t)} , \zb_{i'}^{(p')}\ra )^2.
% \end{align*}
\end{lemma}
\begin{proof}[Proof of Lemma~\ref{lemma:innerproduct_lowerbound_CNN}]

By Lemma~\ref{lemma:gradient_inner_CNN} and the gradient descent update rule, we have
\begin{align}
    &\la \wb^{(t+1)}, \wb^* \ra -\la \wb^{(t)}, \wb^* \ra = \eta\cdot \la -\nabla_{\wb} L(\wb^{(t)},\gamma^{(t)}), \wb^* \ra \nonumber\\
    & = \frac{\eta \gamma^{(t)}}{2 n^2 P \| \wb^{(t)} \|_{\bSigma}^3}\sum_{i,i'=1}^n |\ell_i'^{(t)}|\cdot |\ell_{i'}'^{(t)}| \cdot (\la \wb^{(t)} , \zb_{i'}\ra - \la \wb^{(t)} , \zb_{i}\ra)\cdot ( |\ell_{i'}'^{(t)}|^{-1} \cdot \la \wb^{(t)} , \zb_{i'}\ra - |\ell_{i}'^{(t)}|^{-1} \cdot \la \wb^{(t)} , \zb_{i}\ra )\nonumber\\
    &\quad + \frac{\eta \gamma^{(t)}}{2 n^2 P \| \wb^{(t)} \|_{\bSigma}^3} \cdot \Bigg(\sum_{i=1}^n |\ell_i'^{(t)}| \Bigg) \cdot \sum_{i'=1}^n \sum_{p,p'=1}^P   \big(\la \wb^{(t)} , \zb_{i'}^{(p)}\ra - \la \wb , \zb_{i'}^{(p')}\ra \big)^2\label{eq:nnerproduct_lowerbound_CNN_proof_eq1}
\end{align}
for all $t\geq 0$, $\zb_i^{(p)} = y_i\cdot \xb_i^{(p)}$, and $\zb_i = \sum_{p=1}^P \zb_i^{(p)}$ for $i\in[n]$, $p\in[P]$. Note that 
\begin{align*}
    y_i \cdot f(\wb^{(t)},\gamma^{(t)},\xb_i) = y_i \cdot \sum_{p=1}^P \gamma^{(t)}\cdot \frac{\la \wb^{(t)}, \xb_i^{(p)} \ra}{\| \wb^{(t)} \|_{\bSigma} } = \gamma^{(t)}\cdot  \frac{\la \wb^{(t)}, \zb_i^{(p)} \ra}{\| \wb^{(t)} \|_{\bSigma} },
\end{align*}
and 
$$|\ell_i'^{(t)}|^{-1} = -\{ \ell'[y_i\cdot f(\wb^{(t)},\gamma^{(t)},\xb_i)] \}^{-1} = 1 + \exp( \gamma^{(t)}\cdot \la \wb^{(t)} , \zb_{i}\ra /  \| \wb^{(t)} \|_{\bSigma}).
$$
With the exact same derivation as \eqref{eq:sharper_lower_boudnd_first} in the proof of  Lemma~\ref{lemma:gradient_inner_linear_model}, we have
% \begin{align*}
%     \frac{|\ell_{i'}'^{(t)}|^{-1} \cdot \la \wb^{(t)} , \zb_{i'}\ra - |\ell_{i}'^{(t)}|^{-1} \cdot \la \wb^{(t)} , \zb_{i}\ra }{ \la \wb^{(t)} , \zb_{i'}\ra - \la \wb^{(t)} , \zb_{i}\ra } \geq \max\{ |\ell_{i}'^{(t)}|^{-1}, |\ell_{i'}'^{(t)}|^{-1} \},
% \end{align*}
% and therefore
% \begin{align*}
%     &\la \wb^{(t+1)}, \wb^* \ra -\la \wb^{(t)}, \wb^* \ra \\
%     & \geq \frac{\eta \gamma^{(t)}}{2 n^2 P \| \wb^{(t)} \|_{\bSigma}^3}\sum_{i,i'=1}^n |\ell_i'^{(t)}|\cdot |\ell_{i'}'^{(t)}| \cdot (\la \wb^{(t)} , \zb_{i'}\ra - \la \wb^{(t)} , \zb_{i}\ra)\cdot ( |\ell_{i'}'^{(t)}|^{-1} \cdot \la \wb^{(t)} , \zb_{i'}\ra - |\ell_{i}'^{(t)}|^{-1} \cdot \la \wb^{(t)} , \zb_{i}\ra )\\
%     &\quad + \frac{\eta \gamma^{(t)}}{2 n^2 P \| \wb \|_{\bSigma}^3} \cdot \Bigg(\sum_{i=1}^n |\ell_i'^{(t)}| \Bigg) \cdot \sum_{i'=1}^n \sum_{p,p'=1}^P   \big(\la \wb^{(t)} , \zb_{i'}^{(p)}\ra - \la \wb , \zb_{i'}^{(p')}\ra \big)^2
% \end{align*}
\begin{align}
    &\sum_{i,i'=1}^n |\ell_i'^{(t)}|\cdot |\ell_{i'}'^{(t)}| \cdot (\la \wb^{(t)} , \zb_{i'}\ra - \la \wb^{(t)} , \zb_{i}\ra)\cdot ( |\ell_{i'}'^{(t)}|^{-1} \cdot \la \wb^{(t)} , \zb_{i'}\ra - |\ell_{i}'^{(t)}|^{-1} \cdot \la \wb^{(t)} , \zb_{i}\ra )\nonumber\\ 
    &\qquad\qquad \geq \frac{1}{8} \exp(-\gamma^{(t)}) \sum_{i,i'=1}^n (\la \wb^{(t)} , \zb_{i'}\ra - \la \wb^{(t)} , \zb_{i}\ra)^2.\label{eq:nnerproduct_lowerbound_CNN_proof_eq2}
\end{align}
Moreover, we also have
\begin{align*}
    \frac{1}{n} \sum_{i=1}^n |\ell_i'^{(t)}| & = \frac{1}{n} \sum_{i=1}^n \Bigg[1 + \exp\Bigg( \frac{\gamma^{(t)} \cdot \la \wb^{(t)}, \zb_j \ra }{\sqrt{\frac{1}{n} \sum_{j=1}^n \la \wb^{(t)}, \zb_j \ra^2 }} \Bigg) \Bigg]^{-1}\\
     & \geq  \frac{1}{n} \sum_{i=1}^n \Bigg[1 + \exp\Bigg( \frac{\gamma^{(t)} \cdot |\la \wb^{(t)}, \zb_j \ra| }{\sqrt{\frac{1}{n} \sum_{j=1}^n \la \wb^{(t)}, \zb_j \ra^2 }} \Bigg) \Bigg]^{-1},
\end{align*}
where the inequality follows by the fact that $[1 + \exp(z)]^{-1}$ is a decreasing function. Further note that $[1 + \exp(z)]^{-1}$ is convex over $z \in[0,+\infty)$. Therefore by Jensen's inequality, we have
\begin{align}
    \frac{1}{n} \sum_{i=1}^n |\ell_i'^{(t)}|
     & \geq  \frac{1}{n} \sum_{i=1}^n \Bigg[1 + \exp\Bigg( \frac{\gamma^{(t)} \cdot |\la \wb^{(t)}, \zb_j \ra| }{\sqrt{\frac{1}{n} \sum_{j=1}^n \la \wb^{(t)}, \zb_j \ra^2 }} \Bigg) \Bigg]^{-1} \nonumber\\
     & \geq \Bigg[1 + \exp\Bigg( \frac{1}{n} \sum_{i=1}^n \frac{\gamma^{(t)} \cdot |\la \wb^{(t)}, \zb_j \ra| }{\sqrt{\frac{1}{n} \sum_{j=1}^n \la \wb^{(t)}, \zb_j \ra^2 }} \Bigg) \Bigg]^{-1} \nonumber\\
     & \geq [1 + \exp(\gamma^{(t)} )]^{-1} \nonumber\\
     & \geq \exp(-\gamma^{(t)} ) / 2.\label{eq:apply_jensen_CNN}
\end{align}
% \begin{align*}
%     \frac{1}{n}\sum_{i=1}^n |\ell_i'^{(t)}| = 
% \end{align*}
Plugging \eqref{eq:nnerproduct_lowerbound_CNN_proof_eq2} and \eqref{eq:apply_jensen_CNN} into \eqref{eq:nnerproduct_lowerbound_CNN_proof_eq1}, we obtain
\begin{align*}
&\la \wb^{(t+1)}, \wb^* \ra -\la \wb^{(t)}, \wb^* \ra\\ &\geq 
\frac{\eta \gamma^{(t)}}{16 n^2 P \| \wb^{(t)} \|_{\bSigma}^3}\cdot \exp(-\gamma^{(t)}) \sum_{i,i'=1}^n (\la \wb^{(t)} , \zb_{i'}\ra - \la \wb^{(t)} , \zb_{i}\ra)^2 \\
&\quad + \frac{\eta \gamma^{(t)}}{4 n^2 P \| \wb^{(t)} \|_{\bSigma}^3} \cdot \exp(-\gamma^{(t)}) \cdot \sum_{i=1}^n \sum_{p,p'=1}^P   \big(\la \wb^{(t)} , \zb_{i}^{(p)}\ra - \la \wb , \zb_{i}^{(p')}\ra \big)^2\\
&\geq 
\frac{\eta \gamma^{(t)}\cdot \exp(-\gamma^{(t)}) }{16 P \| \wb^{(t)} \|_{\bSigma}^3}\cdot \Bigg[ \frac{1}{n^2}\sum_{i,i'=1}^n (\la \wb^{(t)} , \zb_{i'}\ra - \la \wb^{(t)} , \zb_{i}\ra)^2 + \frac{1}{n} \sum_{i=1}^n \sum_{p,p'=1}^P   \big(\la \wb^{(t)}
, \zb_{i}^{(p)}\ra - \la \wb , \zb_{i}^{(p')}\ra \big)^2 \Bigg].
% \nonumber\\
% &\quad + \frac{\eta \gamma^{(t)}}{4 n P \| \wb^{(t)} \|_{\bSigma}^3} \cdot \exp(-\gamma^{(t)}) \cdot \sum_{i'=1}^n \sum_{p,p'=1}^P   \big(\la \wb^{(t)} , \zb_{i'}^{(p)}\ra - \la \wb , \zb_{i'}^{(p')}\ra \big)^2
\end{align*}
Recall that $\zb_i = \sum_{p=1}^P \zb_i^{(p)}$. Denoting 
$$
u_{i,p} = \la \wb^{(t)} , \zb_{i}^{(p)}\ra - \frac{1}{nP} \sum_{i'=1}^n\sum_{p'=1}^P \la \wb^{(t)} , \zb_{i'}^{(p')}\ra
$$ 
for $i\in [n]$ and $p\in [P]$, we have
\begin{align}\label{eq:nnerproduct_lowerbound_CNN_proof_eq3}
    \la \wb^{(t+1)}, \wb^* \ra -\la \wb^{(t)}, \wb^* \ra \geq \frac{\eta \gamma^{(t)}\cdot \exp(-\gamma^{(t)}) }{16 P \| \wb^{(t)} \|_{\bSigma}^3}\cdot (I_1 + I_2),
\end{align}
where 
\begin{align*}
    I_1 = \frac{1}{n^2}\sum_{i,i'=1}^n \Bigg( \sum_{p=1}^p u_{i,p} - \sum_{p=1}^p u_{i',p}\Bigg)^2,\qquad
    I_2 = \frac{1}{n} \sum_{i=1}^n \sum_{p,p'=1}^P  ( u_{i,p} - u_{i,p'} )^2.
\end{align*}
By direct calculation, we have
\begin{align}
    I_1 &=  \frac{1}{n^2}\sum_{i,i'=1}^n \Bigg[ \Bigg( \sum_{p=1}^p u_{i,p}\Bigg)^2 - 2 \cdot \Bigg( \sum_{p=1}^p u_{i,p}\Bigg)\cdot \Bigg(\sum_{p=1}^p u_{i',p}\Bigg) + \Bigg(\sum_{p=1}^p u_{i',p}\Bigg)^2 \Bigg] \nonumber \\
    &= \frac{2}{n}\sum_{i=1}^n\Bigg( \sum_{p=1}^p u_{i,p}\Bigg)^2 - \frac{2}{n^2} \cdot \Bigg( \sum_{i=1}^n\sum_{p=1}^p u_{i,p}\Bigg)^2 \nonumber \\
    &= \frac{2}{n}\sum_{i=1}^n\Bigg( \sum_{p=1}^p u_{i,p}\Bigg)^2,\label{eq:nnerproduct_lowerbound_CNN_proof_eq4}
\end{align}
where the last equality follows by the definition of $u_{i,p}$. Moreover, we have
\begin{align}\label{eq:nnerproduct_lowerbound_CNN_proof_eq5}
    I_2 & = \frac{1}{n} \sum_{i=1}^n \sum_{p,p'=1}^P  ( u_{i,p}^2 -2  u_{i,p}  u_{i,p'} + u_{i,p'}^2 )= \frac{2P}{n} \sum_{i=1}^n \sum_{p=1}^P u_{i,p}^2 - \frac{2}{n} \sum_{i=1}^n \Bigg( \sum_{p=1}^P u_{i,p} \Bigg)^2.
\end{align}
Plugging \eqref{eq:nnerproduct_lowerbound_CNN_proof_eq4}, \eqref{eq:nnerproduct_lowerbound_CNN_proof_eq5} and the definition of $u_{i,p}$ into \eqref{eq:nnerproduct_lowerbound_CNN_proof_eq3} gives
\begin{align*}
    \la \wb^{(t+1)}, \wb^* \ra -\la \wb^{(t)}, \wb^* \ra &\geq \frac{\eta \gamma^{(t)}\cdot \exp(-\gamma^{(t)}) }{8n\cdot \| \wb^{(t)} \|_{\bSigma}^3}\cdot \sum_{i=1}^n \sum_{p=1}^P \Bigg(  \la \wb^{(t)} , \zb_{i}^{(p)}\ra - \frac{1}{nP} \sum_{i'=1}^n\sum_{p'=1}^P \la \wb^{(t)} , \zb_{i'}^{(p')}\ra \Bigg)^2.
    % \\
    % &= \frac{\eta \gamma^{(t)}\cdot \exp(-\gamma^{(t)}) }{8n\cdot \| \wb^{(t)} \|_{\bSigma}^3}\cdot \Bigg[ \sum_{i=1}^n \sum_{p=1}^P \la \wb^{(t)} , \zb_{i}^{(p)}\ra ^2   \Bigg]
    % \sum_{i=1}^n \sum_{p=1}^P \Bigg(  \la \wb^{(t)} , \zb_{i}^{(p)}\ra - \frac{1}{nP} \sum_{i'=1}^n\sum_{p'=1}^P \la \wb^{(t)} , \zb_{i'}^{(p')}\ra \Bigg)^2
\end{align*}
We continue the calculation as follows: 
\begin{align*}
    &\frac{1}{nP}\sum_{i=1}^n \sum_{p=1}^P \Bigg(  \la \wb^{(t)} , \zb_{i}^{(p)}\ra - \frac{1}{nP} \sum_{i'=1}^n\sum_{p'=1}^P \la \wb^{(t)} , \zb_{i'}^{(p')}\ra \Bigg)^2\\
    &\qquad = \frac{1}{nP}\sum_{i=1}^n \sum_{p=1}^P \la \wb^{(t)} , \zb_{i}^{(p)}\ra^2 - \Bigg( \frac{1}{nP} \sum_{i=1}^n\sum_{p=1}^P \la \wb^{(t)} , \zb_{i}^{(p)}\ra \Bigg)^2\\
    &\qquad = \frac{1}{2}\cdot \Bigg[ \frac{2}{nP}\sum_{i=1}^n \sum_{p=1}^P \la \wb^{(t)} , \zb_{i}^{(p)}\ra^2 - 2\cdot\Bigg( \frac{1}{nP} \sum_{i=1}^n\sum_{p=1}^P \la \wb^{(t)} , \zb_{i}^{(p)}\ra \Bigg)^2 \Bigg]\\
    &\qquad = \frac{1}{2}\cdot \Bigg[ \frac{1}{nP}\sum_{i=1}^n \sum_{p=1}^P \la \wb^{(t)} , \zb_{i}^{(p)}\ra^2 - 2\cdot\Bigg( \frac{1}{nP} \sum_{i=1}^n\sum_{p=1}^P \la \wb^{(t)} , \zb_{i}^{(p)}\ra \Bigg)^2 + \frac{1}{nP}\sum_{i'=1}^n \sum_{p'=1}^P \la \wb^{(t)} , \zb_{i'}^{(p')}\ra^2 \Bigg]\\
    &\qquad = \frac{1}{2n^2P^2}\cdot \sum_{i,i'=1}^n \sum_{p,p'=1}^P (\la \wb^{(t)} , \zb_{i}^{(p)}\ra^2 - 2\cdot \la \wb^{(t)} , \zb_{i}^{(p)}\ra \la \wb^{(t)} , \zb_{i'}^{(p')}\ra + \la \wb^{(t)} , \zb_{i'}^{(p')}\ra^2 )\\
    &\qquad = \frac{1}{2n^2P^2}\cdot \sum_{i,i'=1}^n \sum_{p,p'=1}^P (\la \wb^{(t)} , \zb_{i}^{(p)}\ra - \la \wb^{(t)} , \zb_{i'}^{(p')}\ra )^2.
\end{align*}
Therefore we have
\begin{align*}
    \la \wb^{(t+1)}, \wb^* \ra -\la \wb^{(t)}, \wb^* \ra &\geq \frac{\eta \gamma^{(t)}\cdot \exp(-\gamma^{(t)}) }{16n^2P\cdot \| \wb^{(t)} \|_{\bSigma}^3}\cdot \sum_{i,i'=1}^n \sum_{p,p'=1}^P (\la \wb^{(t)} , \zb_{i}^{(p)}\ra - \la \wb^{(t)} , \zb_{i'}^{(p')}\ra )^2.
\end{align*}
Applying Lemma~\ref{lemma:key_identity_CNN} finishes the proof.
\end{proof}

The following lemma is the counterpart of Lemma~\ref{lemma:innerproduct_lowerbound}.
\begin{lemma}\label{lemma:norm_upperbound_CNN}
For all $t \geq 0$, it holds that
\begin{align*}
    \| \wb^{(t)}\|_2^2 \leq \|\wb^{(t+1)}\|_2^2 \leq \|\wb^{(t)}\|_2^2 + 4 \eta^2 \cdot  \frac{ P^2\gamma^{(t)2} \cdot \max_i \| \xb_i \|_2^3 }{ \lambda_{\min}^2 \cdot \| \wb^{(t)} \|_2^2}.
\end{align*}
Moreover, if $ \|  \wb^{(t)} -  \la \wb^{(t)}, \wb^* \ra \cdot \|\wb^*\|_2^{-1} \cdot \wb^* \|_2 \leq \| \wb^{(0)} \|_2 / 2 $, then
\begin{align*}
    \|\wb^{(t+1)}\|_2^2
    &\leq \| \wb^{(t)}\|_2^2 + \eta^2  G\cdot \max\{\gamma^{(t)2}, \gamma^{(t)4}\}\cdot \max\{ |\ell_1'^{(t)}|^2, \ldots, |\ell_n'^{(t)}|^2, \exp(-2\gamma^{(t)}) \}\\
    &\quad\cdot  \bigg\|  \wb^{(t)} -  \frac{\la \wb^*, \wb^{(t)}\ra}{\| \wb^* \|_2^2} \cdot \wb^* \bigg\|_2^2,
\end{align*}
where $G = 64 P^3 \lambda_{\min}^{-3}\cdot \max_i \|\xb_{i}\|_2^6 \cdot \| \wb^{(0)} \|_2^{-4}$. 
\end{lemma}
\begin{proof}[Proof of Lemma~\ref{lemma:norm_upperbound_CNN}]
Note that
\begin{align*}
    \nabla_{\wb} L(\wb^{(t)},\gamma^{(t)}) &= \frac{1}{n \cdot \| \wb \|_{\bSigma}}\sum_{i=1}^n  \ell'[y_i\cdot f(\wb^{(t)},\gamma^{(t)},\xb_i)] \cdot y_i \cdot \sum_{p=1}^P \gamma^{(t)} \cdot \bigg( \Ib - \frac{\bSigma \wb^{(t)} \wb^{(t)\top}}{\| \wb^{(t)} \|_{\bSigma}^2} \bigg)\xb_i^{(p)}\\
    &= \frac{1}{n \cdot \| \wb \|_{\bSigma}}\sum_{i=1}^n \sum_{p=1}^P \ell'[y_i\cdot f(\wb^{(t)},\gamma^{(t)},\xb_i)] \cdot y_i \cdot  \gamma^{(t)} \cdot \bigg( \Ib - \frac{\bSigma \wb^{(t)} \wb^{(t)\top}}{\| \wb^{(t)} \|_{\bSigma}^2} \bigg)\xb_i^{(p)}
\end{align*}
One can essentially treat $\xb_i^{(p)}$, $i\in[n]$ and $p\in[P]$ as different data points and use the same proof as Lemma~\ref{lemma:norm_upperbound} to prove the first inequality. As for the second inequality, for any $\wb \in 
\cB^{(t)}:= \{\wb \in \RR^d : \| \wb - \wb^{(t)} \|_2 \leq \| \wb^{(t)} \|_2/2 \}$, we have
\begin{align*}
    \| \nabla_{\wb} f(\wb, \gamma^{(t)}, \xb_i)\|_2 &=  \Bigg\|\sum_{p=1}^P \frac{ \gamma^{(t)}}{ \| \wb\|_{\bSigma}} \cdot \bigg( \Ib - \frac{\bSigma \wb \wb^{\top}}{\| \wb \|_{\bSigma}^2} \bigg)\xb_i^{(p)} \Bigg\|_2 \\
    % & \leq \frac{ \gamma^{(t)}}{ \| \wb \|_{\bSigma}} \cdot  \Bigg\| \bigg( \Ib - \frac{\bSigma \wb \wb^{\top}}{\| \wb \|_{\bSigma}^2} \bigg) \Bigg\|_2 \cdot \| \xb_i \|_2\\
    &\leq \frac{ P \gamma^{(t)}}{ \| \wb \|_{\bSigma}} \cdot \max_{i,p} \| \xb_i^{(p)} \|_2 \cdot \bigg( 1 +  \bigg\| \frac{\bSigma \wb \wb^{\top}}{\| \wb \|_{\bSigma}^2}  \bigg\|_2 \bigg) \\
    &= \frac{ P\gamma^{(t)}}{ \| \wb \|_{\bSigma}} \cdot \max_{i,p} \| \xb_i^{(p)} \|_2 \cdot \bigg( 1 +   \frac{\| \bSigma \wb\|_2\cdot \| \wb\|_2}{\| \wb \|_{\bSigma}^2} \bigg) \\
    &\leq  \frac{2P \gamma^{(t)}}{ \| \wb \|_{\bSigma}} \cdot \max_{i,p} \| \xb_i^{(p)} \|_2 \cdot \frac{\| \bSigma \wb\|_2\cdot \| \wb\|_2}{\| \wb \|_{\bSigma}^2} ,
\end{align*}
where the last inequality follows by the fact that $\| \wb \|_{\bSigma}^2 = \la \wb, \bSigma \wb \ra \leq \| \bSigma \wb\|_2\cdot \| \wb\|_2$. Further plugging in the definition of $\bSigma$ gives
\begin{align*}
\| \nabla_{\wb} f(\wb, \gamma^{(t)}, \xb_i)\|_2
    &\leq \frac{2P \gamma^{(t)}}{ \| \wb \|_{\bSigma}^3} \cdot \max_{i,p} \| \xb_i^{(p)} \|_2 \cdot \| \wb\|_2 \cdot \Bigg\| \frac{1}{nP}\sum_{i'=1}^n \sum_{p'=1}^P \xb_{i'}^{(p)} \cdot \la \wb, \xb_{i'}^{(p)} \ra \Bigg\|_2  \\
    &\leq \frac{2 P\gamma^{(t)}}{ \| \wb \|_{\bSigma}^3} \cdot \max_{i,p} \| \xb_i^{(p)} \|_2^2 \cdot \| \wb\|_2 \cdot \frac{1}{nP}\sum_{i'=1}^n \sum_{p'=1}^P |\la \wb, \xb_{i'}^{(p)} \ra|  \\
    &\leq \frac{2 P\gamma^{(t)}}{ \| \wb \|_{\bSigma}^3} \cdot \max_{i,p} \| \xb_i^{(p)} \|_2^2 \cdot \sqrt{\frac{1}{nP}\sum_{i'=1}^n \sum_{p'=1}^P |\la \wb, \xb_{i'}^{(p)} \ra|^2 } \\
    &= \frac{2 P\gamma^{(t)}}{ \| \wb \|_{\bSigma}^2} \cdot \max_{i,p} \| \xb_i^{(p)} \|_2^2 \cdot \| \wb\|_2 \\
    &\leq \frac{2 P\gamma^{(t)}}{ \lambda_{\min}\cdot \| \wb \|_2^2} \cdot \max_{i,p} \| \xb_i^{(p)} \|_2^2 \cdot \| \wb\|_2 \\
    &\leq \frac{4 P\gamma^{(t)}}{ \lambda_{\min}\cdot \| \wb^{(t)} \|_2} \cdot \max_{i,p} \| \xb_i^{(p)} \|_2^2,
    % \\
    % &\leq \frac{ \gamma^{(t)}}{ \lambda_{\min}^{1/2}\cdot \| \wb \|_2 } \cdot \| \xb_i \|_2\\
    % &\leq \frac{ 2\gamma^{(t)}}{ \lambda_{\min}^{1/2}\cdot \| \wb^{(t)} \|_2 } \cdot \| \xb_i \|_2,
\end{align*}
where the third inequality follows by Jensen's inequality, and the last inequality follows by $\|\wb\|_2 \geq \| \wb^{(t)} \|_2 / 2$ for all $\wb \in \cB^{(t)}$. Therefore $ f(\wb, \gamma^{(t)}, \xb_i)$ is $(4 \gamma^{(t)} \lambda_{\min}^{-1}\cdot \| \wb^{(t)} \|_2^{-1} \cdot \max_{i,p} \| \xb_i^{(p)} \|_2^2)$-Lipschitz. The rest of the proof is the same as the proof of Lemma~\ref{lemma:norm_upperbound}. 
\end{proof}

\subsection{Proof of Theorem \ref{thm:data_example1}}
\begin{proof}[Proof of Theorem~\ref{thm:data_example1}] We first show the existence and uniqueness results for the maximum margin and patch-wise uniform margin classifiers. By Example~\ref{def:data_example1}, the training data patches are given as $\xb_i^{(p)} = y_i\cdot \ub + \bxi_i^{(p)}$ for $i\in [n]$, $p\in [P]$, where $\bxi_i^{(p)} \sim N(\mathbf{0}, \sigma^2 (\Ib - \ub \ub^\top / \|\ub \|_2^2) ))$ are orthogonal to $\ub$. Therefore it is clear that the linear classifier defined by $\ub$ can linearly separate all the data points $(\overline\xb_i, y_i)$, $i\in[n]$. Therefore the maximum margin solution exists. Its uniqueness then follows by the definition of the maximum margin problem \eqref{eq:max_margin_def}, which has a strongly convex objective function and linear constraints.

As for the  patch-wise uniform margin classifier, the existence also follows by the observation that $\ub$ gives such a  classifier with patch-wise uniform margin. Moreover, for a patch-wise uniform margin classifier $\wb$, by definition we have
\begin{align*}
    y_{i'}\cdot \la \wb , \xb_{i'}^{(p')}\ra = y_{i}\cdot \la \wb , \xb_{i}^{(p)}\ra
\end{align*}
for all $i,i'\in [n]$ and $p,p'\in [P]$. Plugging in the data model $\xb_i^{(p)} = y_i\cdot \ub + \bxi_i^{(p)}$ then gives
\begin{align}\label{eq:example1_uniform_margin_equations}
    \la \wb , y_{i'}\cdot \bxi_{i'}^{(p')} - y_{i}\cdot\bxi_{i}^{(p)}\ra = 0 %-  \la \wb , y_{i}\cdot\bxi_{i}^{(p)}\ra = 0
\end{align}
for all $i,i'\in [n]$ and $p,p'\in [P]$. Note that $nP \geq 4n = 2d$. Therefore it is easy to see that with probability $1$, 
\begin{align*}
    \mathrm{span}\big\{ y_{i'}\cdot \bxi_{i'}^{(p')} - y_{i}\cdot\bxi_{i}^{(p)}: i,i'\in [n],~p,p'\in [P] \big \} = \mathrm{span}\{ \ub \}^{\perp}.
\end{align*}
Therefore by \eqref{eq:example1_uniform_margin_equations}, we conclude that $\wb$ is parallel to $\ub$, and thus the patch-wise uniform margin classifier is unique up to a scaling factor. Moreover, this immediately implies that 
\begin{align*}
    \PP_{(\xb_{\mathrm{test}},y_{\mathrm{test}})\sim \cD}( y_{\mathrm{test}}\cdot \la \wb_{\mathrm{uniform}}, \overline\xb_{\mathrm{test}} \ra < 0 ) &= \PP_{(\xb_{\mathrm{test}},y_{\mathrm{test}})\sim \cD}\Bigg[ y_{\mathrm{test}}\cdot \Bigg\la \ub, \sum_{p=1}^P( y_{\mathrm{test}}\cdot \ub + \bxi_{\mathrm{test}}^{(p)} ) \Bigg\ra < 0 \Bigg] \\
    &= \PP_{(\xb_{\mathrm{test}},y_{\mathrm{test}})\sim \cD}(\|\ub \|_2^2 < 0) = 0.
\end{align*}
This proves the first result.

As for the maximum margin classifier $\wb_{\max}$, we first denote $\overline{\bxi}_i = y_i\cdot \sum_{p=1}^P \bxi_i^{(p)} $. Then  $\overline{\bxi}_i$, $i\in [n]$ are independent Gaussian random vectors from $N(\mathbf{0}, \sigma^2P (\Ib - \ub \ub^\top / \|\ub \|_2^2) )$. Define $\bGamma = [ \overline{\bxi}_1, \overline{\bxi}_2,\ldots,\overline{\bxi}_n ] \in \RR^{d\times n}$. Then focusing on the subspace $\mathrm{span}\{\ub\}^{\perp}$, by Corollary~5.35 in \citet{vershynin2010introduction}, with probability at least $1 - \exp(- \Omega(\sigma^2Pd))$ we have
\begin{align}
    & \sigma_{\min}( \bGamma ) \geq (\sqrt{d-1} - \sqrt{n} - \sqrt{d} / 10)\cdot \sigma\sqrt{P} \geq \sigma \sqrt{Pd} / 10, \label{eq:Gamma_singular_lower_bound} \\
    & \sigma_{\max}( \bGamma ) \leq (\sqrt{d-1} + \sqrt{n} + \sqrt{d} / 10)\cdot \sigma\sqrt{P} \leq 2\sigma\sqrt{Pd},\label{eq:Gamma_singular_upper_bound}
\end{align}
where we use $\sigma_{\min}(\cdot)$ and $\sigma_{\max}(\cdot)$ to denote the smallest and largest non-zero singular values of a matrix. Note that by $d = 2n > n+1$, $\bGamma$ has full column rank. Let $\hat\wb = \bGamma(\bGamma^\top \bGamma)^{-1}\mathbf{1}$, then we have $\hat\wb \in \mathrm{span}\{\ub\}^{\perp}$ and $\bGamma^\top \hat\wb = \mathbf{1}$. Therefore 
\begin{align*}
    \la \hat\wb, \overline{\bxi}_i \ra = \la \hat\wb, P\cdot y_i\cdot \ub + \overline{\bxi}_i \ra = 0 + \la \hat\wb, \overline{\bxi}_i \ra = \eb_i^\top \bGamma^\top \hat\wb = 1
\end{align*}
for all $i\in [n]$. This implies that $\hat\wb$ is a feasible solution to the maximum margin problem \eqref{eq:max_margin_def}. Therefore by the optimality of $\wb_{\max}$, we have
\begin{align*}
    \la \wb_{\max}, \ub \ra \cdot \| \ub \|_2^{-1} \leq \| \wb_{\max} \|_2 \leq \| \hat\wb \|_2 = \| \bGamma(\bGamma^\top \bGamma)^{-1}\mathbf{1} \|_2 = \sqrt{ \mathbf{1}^\top (\bGamma^\top \bGamma)^{-1}\mathbf{1} } \leq n\cdot \sigma_{\min}^{-1}(\bGamma).
\end{align*}
Then by \eqref{eq:Gamma_singular_lower_bound}, we have
\begin{align*}
    \la \wb_{\max}, \ub \ra \leq \| \ub \|_2 \cdot n\cdot \sigma_{\min}^{-1}(\bGamma) \leq 10 \| \ub \|_2 \cdot n P^{-1/2} d^{-1/2} \sigma^{-1}.
\end{align*}
Now by the assumption that $\sigma \geq 20 \| \ub \|_2 \cdot n P^{1/2} d^{-1/2} $, we have
\begin{align}\label{eq:example1_max_margin_feature_upper_bound}
    \la \wb_{\max}, \ub \ra  \leq 1/(2P),
\end{align}
and therefore
% $\la \wb_{\max}, \ub \ra  \leq 1/(2P)$, and
\begin{align}\label{eq:example1_max_margin_noise_lower_bound_1}
    \la \wb_{\max} , \overline{\bxi}_i \ra \geq 1 - \la \wb_{\max}, P\cdot \ub \ra \geq 1/2. 
\end{align}
Further note that \eqref{eq:max_margin_def} indicates that $\wb_{\max} \in \mathrm{span}\{ \ub, \overline{\bxi}_1, \overline{\bxi}_2,\ldots,\overline{\bxi}_n \}$. Therefore we have
\begin{align}
    \| (\Ib - \ub\ub^\top / \|\ub \|_2^2)\wb_{\max} \|_2
    &= \|\bGamma(\bGamma^\top \bGamma)^{-1} \bGamma^\top \wb_{\max} \|_2 \nonumber\\
    &= \sqrt{(\bGamma^\top \wb_{\max})^\top (\bGamma^\top \bGamma)^{-1} (\bGamma^\top \wb_{\max})} \nonumber\\
    &\geq \| \bGamma^\top \wb_{\max} \|_2 \cdot \sigma_{\max}^{-1}(\bGamma) \nonumber\\
    &\geq \sqrt{n} \cdot (1/2) \cdot (2\sigma\sqrt{Pd})^{-1} \nonumber\\
    &\geq \sigma^{-1} P^{-1/2} /8,\label{eq:example1_max_margin_noise_lower_bound}
\end{align}
where the second inequality follows by \eqref{eq:example1_max_margin_noise_lower_bound_1} and \eqref{eq:Gamma_singular_upper_bound}. 
Now for a new test data point $(\xb_{\mathrm{test}},y_{\mathrm{test}})$, we have
$\overline\xb_{\mathrm{test}} = P\cdot y_{\mathrm{test}} \cdot \ub + \overline{\bxi}_{\mathrm{test}}$, where $\overline{\bxi}_{\mathrm{test}} = y_{\mathrm{test}}\cdot \sum_{p=1}^P \bxi_{\mathrm{test}}^{(p)} \sim N(\mathbf{0},  \sigma^2P (\Ib - \ub \ub^\top / \|\ub \|_2^2))$. Then by \eqref{eq:example1_max_margin_feature_upper_bound}, we have
\begin{align*}
    &\la \wb_{\max}, P\ub \ra \leq 1/2.
\end{align*}
Moreover, by \eqref{eq:example1_max_margin_noise_lower_bound} we also have
\begin{align*}
    &\la \wb_{\max}, \overline{\bxi}_{\mathrm{test}} \ra \sim N( 0, \overline\sigma_{\mathrm{test}}^2), ~\overline\sigma_{\mathrm{test}} \geq  \| (\Ib - \ub\ub^\top / \|\ub \|_2^2)\wb_{\max} \|_2\cdot \sigma\sqrt{P} \geq 1/8.
\end{align*}
Therefore 
\begin{align*}
    \PP_{(\xb_{\mathrm{test}},y_{\mathrm{test}})\sim \cD}( y_{\mathrm{test}}\cdot \la \wb_{\max}, \overline\xb_{\mathrm{test}} \ra < 0 ) = \PP_{(\xb_{\mathrm{test}},y_{\mathrm{test}})\sim \cD}( \la \wb_{\max}, P\ub \ra + \la \wb_{\max}, \overline{\bxi}_{\mathrm{test}} \ra  < 0 )= \Theta(1).
\end{align*}
This finishes the proof. 
\end{proof}

\subsection{Proof of Theorem \ref{thm:data_example2}}

\begin{proof}[Proof of Theorem \ref{thm:data_example2} ]
\noindent\textbf{Proof for uniform margin solution $\wb_{\mathrm{uniform}}$.}
Without loss of generality, we assume $\|\wb_{\mathrm{uniform}}\|_2=\|\wb_{\max}\|_2=1$. Then according to the definition of uniform margin, we can get that for all $i\in[n]$, it holds that for all strong signal data
\begin{align}\label{eq:uniform_strong}
|\la\wb_{\mathrm{uniform}},\ub\ra|=|\la\wb_{\mathrm{uniform}},\vb\ra| = |\la\wb_{\mathrm{uniform}},\bxi_i\ra|;
\end{align}
and for all weak signal data:
\begin{align}\label{eq:uniform_weak}
|\la\wb_{\mathrm{uniform}},\ub\ra|=|\la\wb_{\mathrm{uniform}},\vb\ra| =|\la\wb_{\mathrm{uniform}},\bxi_i+\alpha\zeta_i\ub\ra|.
\end{align}
Then note that $\wb_{\mathrm{uniform}}$ lies in the span of $\{\ub,\vb\}\cup\{\bxi_i\}_{i=1,\dots,n}$, we can get,
\begin{align*}
\bigg|\bigg\la \wb_{\mathrm{uniform}}, \frac{\ub}{\|\ub\|_2}\bigg\ra\bigg|^2 + \bigg|\bigg\la \wb_{\mathrm{uniform}}, \frac{\vb}{\|\vb\|_2}\bigg\ra\bigg|^2 + \sum_{i=1}^n \bigg|\bigg\la \wb_{\mathrm{uniform}}, \frac{\bxi_i+\alpha\zeta_i\ub}{\|\bxi_i+\alpha\zeta_i\ub\|_2}\bigg\ra\bigg|^2 \ge \|\wb_{\mathrm{uniform}}\|_2^2=1,
\end{align*}
where we slightly abuse the notation by setting $\zeta_i=0$ if $(\xb_i,y_i)$
is a strong signal data. Then use the fact that $\|\ub\|_2=1$, $\|\vb\|_2=\alpha^2$, and $\sigma=d^{-1/2}$, we can get that with probability at least $1-\exp(-\Omega(d))$ with respect to the randomness of training data,
\begin{align*}
|\la\wb_{\mathrm{uniform}}, \ub\ra|^2 + \alpha^{-4}|\la\wb_{\mathrm{uniform}}, \vb\ra|^2 + \sum_{i=1}^n |\la\wb_{\mathrm{uniform}}, \bxi_i+\alpha\zeta_i\ub\ra|^2 \ge c
\end{align*}
for some absolute positive constant $c$. Then by \eqref{eq:uniform_strong} and \eqref{eq:uniform_weak} and using the fact that $\alpha=n^{-1/4}$, we can immediately get that 
\begin{align*}
\la\wb_{\mathrm{uniform}},\ub\ra = \la\wb_{\mathrm{uniform}},\vb\ra = \Omega(n^{-1/2}).
\end{align*}

We will then move on to the test phase. Consider a new strong signal data $(\xb ,y)$ with $\xb = [\ub, \bxi]$ and $y=1$, we have
\begin{align*}
\la\wb_{\mathrm{uniform}}, \xb^{(1)}\ra + \la\wb_{\mathrm{uniform}}, \xb^{(2)}\ra = \la\wb_{\mathrm{uniform}}, \ub\ra + \la\wb_{\mathrm{uniform}}, \bxi\ra = \Omega(n^{-1/2})  + \xi,
\end{align*}
where $\xi$ is an independent Gaussian random variable with variance smaller than $\sigma^2$. Additionally, given a weak signal data $(\xb ,y)$ with $\xb = [\vb, \bxi]$ and $y=1$, we have
\begin{align*}
\la\wb_{\mathrm{uniform}}, \xb^{(1)}\ra + \la\wb_{\mathrm{uniform}}, \xb^{(2)}\ra &= \la\wb_{\mathrm{uniform}}, \vb\ra + \la\wb_{\mathrm{uniform}}, \bxi+\alpha\zeta\ub\ra \notag\\
&= \la\wb_{\mathrm{uniform}}, \vb\ra\cdot (1 + \alpha\zeta)  + \xi\notag\\
&= \Omega(n^{-1/2})+\xi,
\end{align*}
where the second equation is due to $\la\wb_{\mathrm{uniform}}, \ub\ra=\la\wb_{\mathrm{uniform}}, \vb\ra$ and $\xi$ is an independent Gaussian random variable with variance smaller than $\sigma^2$. Further note that $d=\omega(n\log(n))$ and $\sigma = d^{-1/2}$, we can immediately get that with probability at most $1-1/\poly(n)\ge 1-1/n^{10}$, the random variable $\xi$ will exceed $n^{-1/2}$. This completes the proof for the uniform margin solution.

\noindent\textbf{Proof for maximum margin solution $\wb_{\max}$.}
For the maximum margin solution, we consider 
\begin{align}\label{eq:def_maxmargin_proof}
\wb_{\max} = \arg\min_{\wb} \|\wb\|_2\quad \text{s.t. } y_i\wb^\top[\xb_i^{(0)}+ \xb_i^{(1)}]\ge 1.
\end{align}
We first prove an upper bound on the norm of $\wb_{\max}$ as follows. Based on the above definition, the upper bound can be obtained by simply finding a $\wb$ that satisfies the margin requirements. Therefore, let $\zb_i = y_i\cdot [\xb_i^{(0)}+\xb_i^{(1)}]$, we consider a candidate solution $\hat\wb$ that satisfies 
\begin{align*}
&\text{ for all strong signal data:}\quad\hat\wb^\top\ub = 1, y_i\hat\wb^\top\bxi_i=0 ;\\
&\text{ for all weak signal data:}\quad\hat\wb^\top\vb=0, \wb^\top y_i(\bxi_i+\alpha\zeta_i\ub)=1 
\end{align*}
Then, let $\Pb_{\cE^c}$ be the projection on the subspace that is orthogonal to all noise vectors of the strong signal data. Then the above condition for weak signal data requires
\begin{align*}
y_i\la\hat\wb,\Pb_{\cE^c}\bxi_i\ra = 1 - \alpha y_i\zeta_i,
\end{align*}
where we use the fact that $\la\hat\wb,\ub\ra=1$. 
Let $n_1$ and $n_2$ be the numbers of strong signal data and weak signal data respectively, 
 which clearly satisfy $n_1=\Theta(n)$ and $n_2=\Theta(\rho n)=\Theta(n^{1/4})$ with probability at least $1-\exp(-\Omega(n^{1/4}))$. Define $\rb \in\RR^{n_2\times 1}$ be the collection of $y_i(1 - \alpha y_i\zeta_i)$ for all weak signal data, let $\bomega_i = \Pb_{\cE^c}\bxi_i$ and set $\bOmega\in\RR^{n_2\times d}$ as the collection of $\bomega_i$'s. Then the above equation can be rewritten as $\bOmega\hat\wb = \rb$, which leads to a feasible solution 
\begin{align*}
\hat\wb = \ub +\bOmega^\top(\bOmega\bOmega^\top)^{-1}\rb.
\end{align*}
Then note that the noise vectors for all data points are independent, conditioning on $\Pb_{\cE^c}$, the random vector $\omega_i=\Pb_{\cE^c}\bxi_i$ can be still regarded as a Gaussian random vector in $d-n_1-2$ dimensional space. Then by standard random matrix theory, we can obtain that  we can get that with probability at least $1-\exp(-\Omega(d))$ with respect to the randomness of training data, $\lambda_{\min}(\bOmega\bOmega^\top),\lambda_{\min}(\bOmega\bOmega^\top)=\Theta\big(\sigma^2\cdot (d-n_1-2)\big)=\Theta(1)$. Further note that $\|\rb\|_2 = \Theta(n_2^{1/2})$, we can finally obtain that
\begin{align}\label{eq:upperbound_wmax}
\|\wb_{\max}\|_2^2\le \|\hat\wb\|_2^2 = 1 + \big\|\bOmega^\top(\bOmega\bOmega^\top)^{-1}\rb\big\|_2^2 = \Theta(n_2)=\Theta(n^{1/4}).
\end{align}
Therefore, we can further get
\begin{align*}
\la\wb_{\max},\vb\ra\le \|\wb_{\max}\|_2\cdot\|\vb\|_2=O\big(\alpha^2 n^{1/8}\big)=O(n^{-7/8}).
\end{align*}

Next, we will show that $\la\wb_{\max},\ub\ra>1/2$. In particular, under the margin condition in \eqref{eq:def_maxmargin_proof}, we have for all strong signal data
\begin{align*}
\la\wb_{\max},\ub\ra + \la\wb_{\max},y_i\bxi_i\ra\ge 1.
\end{align*}
Then if $\la\wb_{\max},\ub\ra\le 1/2$, we will then get $\la\wb_{\max},y_i\bxi_i\ra\ge 1/2$ for all strong signal data. Let $\cS_{\ub}$ be the collection of indices of strong signal data, we further have
\begin{align*}
\bigg\la\wb_{\max},\sum_{i\in\cS_{\ub}}y_i\bxi_i\bigg\ra\ge n_1.
\end{align*}
Let $\bxi' = \sum_{i\in\cS_{\ub}}y_i\bxi_i$, it can be seen that $\bxi'$ is also a Gaussian random vector with covariance matrix $n_1\sigma^2\big(\Ib-\ub\ub^\top/\|\ub\|_2^2-\vb\vb^\top/\|\vb\|_2^2\big)$. This implies that conditioning on $\cS_\ub$, with probability at least $1-\exp(-\Omega(d))$, it holds that $\|\bxi'\|_2=\Theta(n_1^{1/2})$. Then using the fact that $n_1=\Theta(n)$, it further suggests that
\begin{align*}
\|\wb_{\max}\|_2\ge \frac{n_1}{\|\bxi'\|_2}=\Theta(n^{1/2}),
\end{align*}
which contradicts the upper bound on the norm of maximum margin solution we have proved in \eqref{eq:upperbound_wmax}. Therefore we must have $\la\wb_{\max},\ub\ra>1/2$. 

Finally, we are ready to evaluate the test error of $\wb_{\max}$. In particular, since we are proving a lower bound on the test error, we will only consider the weak signal data while assuming all strong feature data can be correctly classified. Consider a weak signal data $(\xb, y)$ with with $\xb = [\vb, \bxi]$ and $y=1$, we have
\begin{align*}
\la\wb_{\max}, \xb^{(1)}\ra + \la\wb_{\max}, \xb^{(2)}\ra &= \la\wb_{\max}, \vb\ra + \la\wb_{\max}, \bxi+\alpha\zeta\ub\ra \notag\\
&= \la\wb_{\max}, \vb\ra  + \zeta\cdot\alpha\la\wb_{\max},\ub\ra+\|\wb_{\max}\|_2\cdot\xi,
\end{align*}
where $\xi$ is a random variable with variance smaller than $\sigma_0$.
Note that $\zeta=-1$ with half probability. In this case, using our previous results on $\la\wb_{\max},\vb\ra$ and $\la\wb_{\max},\ub\ra$, and the fact that $\xi$ is independent of $\vb$ and $\ub$, we can get with probability at least $1/2$,
\begin{align*}
\la\wb_{\max}, \xb^{(1)}\ra + \la\wb_{\max}, \xb^{(2)}\ra \le \Theta(n^{-7/8}) - \Theta(n^{-1/2}) <0
\end{align*}
which will lead to an incorrect prediction. This implies that for weak signal data, the population prediction error will be at least $1/4$. Noting that the weak signal data will appear with probability $\rho$, combining them will be able to complete the proof.
\end{proof}

\section{Proofs of Technical Lemmas}

\subsection{Proof of Lemma~\ref{lemma:auxiliary_inequality}}

\begin{proof}[Proof of Lemma~\ref{lemma:auxiliary_inequality}]
First, it is easy to see that
\begin{align}\label{eq:half_sum}
\sum_{i,i'=1}^n \max\{b_i,b_{i'}\}\cdot (a_i - a_{i'})^2  = 2\sum_{i=1}^n\sum_{i'>i}b_i\cdot (a_{i'} - a_i)^2.
\end{align}
Then, we can also observe that for any $i<j$, we have
\begin{align}\label{eq:increasing_ai}
\sum_{i'>i}(a_{i'}-a_i)^2 \ge \sum_{i'>j} (a_{i'}-a_i)^2 \ge \sum_{i'>j}(a_{i'}-a_j)^2,
\end{align}
where we use the fact that $a_i\le a_j$.
Then let $\bar b = \frac{\sum_{i=1}^n b_i}{n}$ and $k^*$ be the index satisfying $b_{k^*-1}\ge \bar b/2>b_{k^*}$, we can immediately get that
\begin{align}\label{eq:bound_sum_bis}
\sum_{i\ge k^*}b_i \le \frac{\sum_{i=1}^n b_i}{2n}\cdot n = \frac{\sum_{i=1}^n b_i}{2}, 
\end{align}
which further implies that $
\sum_{i< k^*}b_i \ge n\bar b/2$.
Then by \eqref{eq:increasing_ai}, we can get
% \begin{align*}
% \sum_{i\ge k^*}\sum_{i'>i}  b_i\cdot  (a_{i'}-a_i)^2 \le \sum_{i\ge k^*} \sum_{i'>k^*}  b_i\cdot (a_{i'}-a_{i})^2 \le \sum_{i\ge k^*} \sum_{i'>k^*}  b_i\cdot (a_{i'}-a_{k^*})^2  \le \frac{n\bar b}{2}\cdot \sum_{i'>k^*}(a_{i'}-a_{k^*})^2.
% \end{align*}
\begin{align*}
\frac{\bar b}{2}\cdot\sum_{i\ge k^*}\sum_{i'>i}  (a_{i'}-a_i)^2 \le \frac{\bar b}{2}\cdot\sum_{i\ge k^*} \sum_{i'>k^*} (a_{i'}-a_{k^*})^2 \le \frac{n\bar b}{2}\cdot \sum_{i'>k^*}(a_{i'}-a_{k^*})^2.
\end{align*}
Besides, we also have
\begin{align*}
\sum_{i< k^*}\sum_{i'>i} b_i\cdot (a_{i'}-a_i)^2 \ge\bigg(\sum_{i<k^*} b_i\bigg)\cdot\sum_{i'>k^*}  (a_{i'}-a_{k^*})^2 \ge \frac{n\bar b}{2}\cdot \sum_{i'>k^*}(a_{i'}-a_{k^*})^2,
\end{align*}
which immediately implies that
\begin{align*}
\sum_{i< k^*}\sum_{i'>i} b_i\cdot (a_{i'}-a_i)^2\ge\frac{\bar b}{2}\cdot\sum_{i\ge k^*}\sum_{i'>i}  (a_{i'}-a_i)^2.
\end{align*}
Therefore, we can get
\begin{align*}
\sum_{i=1}^n\sum_{i'>i}b_i\cdot (a_{i'} - a_i)^2&\ge \sum_{i< k^*}\sum_{i'>i} b_i\cdot (a_{i'}-a_i)^2\notag\\
&\ge \sum_{i< k^*}\sum_{i'>i} \frac{b_i}{2}\cdot (a_{i'}-a_i)^2 + \frac{\bar b}{4}\cdot \sum_{i\ge k^*}\sum_{i'>i}(a_{i'}-a_i)^2\notag\\
&\ge\frac{\bar b}{4}\sum_{i=1}^n\sum_{i'>i}(a_{i'}-a_i)^2,
\end{align*}
where we use the fact that $b_i\ge \bar b/2$ for all $i< k^*$.
Finally, putting the above inequality into \eqref{eq:half_sum} and applying the definition of $\bar b$, we are able to complete the proof.
\end{proof}

\subsection{Proof of Lemma~\ref{lemma:gamma_comparison}}

\begin{proof}[Proof of Lemma~\ref{lemma:gamma_comparison}]
We first show the lower bound of $ a^{(t)}$. Consider a continuous-time sequence $\underline{a}^{(t)} $, $t\geq 0$ defined by the integral equation
\begin{align}\label{eq:integral_equation}
    \underline{a}^{(t)} = \underline{a}^{(0)} + c\cdot  \int_{0}^{t} \exp(-\underline{a}^{(\tau)} ) \mathrm{d} \tau,\quad \underline{a}^{(0)} = a^{(0)}.   
\end{align}
Note that $ \underline{a}^{(t)}$ is obviously an increasing function of $t$. Therefore we have
\begin{align*}
    \underline{a}^{(t+1)} &= \underline{a}^{(t)} + c\cdot \int_{t}^{t+1} \exp(-\underline{a}^{(\tau)} ) \mathrm{d} \tau\\
    &\leq \underline{a}^{(t)} + c\cdot \int_{t}^{t+1} \exp(-\underline{a}^{(t)} ) \mathrm{d} \tau\\
    &= \underline{a}^{(t)} + c\cdot \exp(-\underline{a}^{(t)} )
\end{align*}
for all $t\in \NN$. 
Comparing the above inequality with the iterative formula $ a^{(t+1 )} = a^{(t)} + c \cdot \exp(-a^{(t)})$, we conclude by the comparison theorem that $a^{(t)} \geq \underline{a}^{(t)}$ for all $t\in \NN$. Note that \eqref{eq:integral_equation} has an exact solution
\begin{align*}
     \underline{a}^{(t)} = \log( c\cdot t + \exp(a^{(0)}) ).
\end{align*}
Therefore we have
\begin{align*}
    a^{(t)} \geq \log( c\cdot t + \exp(a^{(0)}) )
\end{align*}
for all $t\in \NN$, which completes the first part of the proof. Now for the upper bound of $a^{(t)}$, we have
\begin{align*}
    a^{(t)} &= a^{(0)} +c\cdot \sum_{\tau=0}^{t} \exp(- a^{(\tau)} ) \\
    &\leq a^{(0)} +c\cdot \sum_{\tau=0}^{t} \exp[- \log( c\cdot \tau + \exp(a^{(0)}) ) ]\\
    &= a^{(0)} +c\cdot \sum_{\tau=0}^{t} \frac{1}{  c\cdot \tau + \exp(a^{(0)})  }\\
    &= a^{(0)} + \frac{c}{ \exp(a^{(0)})  } +  c\cdot \sum_{\tau=1}^{t} \frac{1}{  c\cdot \tau + \exp(a^{(0)})  }\\
    &\leq a^{(0)} + \frac{c}{ \exp(a^{(0)})  } +  c\cdot\int_{0}^t \frac{1}{  c\cdot \tau + \exp(a^{(0)})  }\mathrm{d}\tau, 
\end{align*}
where the second inequality follows by the lower bound of $a^{(t)}$ as the first part of the result of this lemma. Therefore we have
\begin{align*}
    a^{(t)}&\leq a^{(0)} + \frac{c}{ \exp(a^{(0)})  } + \log( c\cdot t +\exp(a^{(0)}) ) - \log(\exp(a^{(0)}) )\\
    &= c \exp(-a^{(0)}) + \log( c\cdot t +\exp(a^{(0)}) ).
\end{align*}
This finishes the proof.
\end{proof}

\bibliography{deeplearningreference}
\bibliographystyle{ims}

\end{document}